\documentclass{report}

\usepackage{algorithm}
\usepackage[noend]{algorithmic}
\usepackage{amsbsy}
\usepackage{amsfonts}
\usepackage{amsmath}
\usepackage{amssymb}
\usepackage{amsthm}
\usepackage[toc]{appendix}
\usepackage{booktabs}
\usepackage{cite}
\usepackage{fullpage}
\usepackage{graphicx}
\usepackage{hyperref}
\usepackage{makecell,multirow}
\usepackage[square,numbers]{natbib}
\usepackage{tablefootnote}
\usepackage{xspace}

\newcommand{\dd}{\mathrm{d}}

\newcommand{\eqdef}{\triangleq}

\newcommand{\ignore}[1]{}
\newcommand{\inner}[2]{ \ensuremath { \left\langle #1, #2 \right\rangle  } }
\newcommand{\integers}{\mathbb{Z}}
\newcommand{\mat}[1]{\mathbf{#1}}

\newcommand{\paren}[1]{\left(#1\right)}

\newcommand{\reals}{\mathbb{R}}

\newcommand{\set}[1]{\left\{#1\right\}}
\newcommand{\shape}{\mathrm{shape}}

\newcommand{\tee}{\mathsf{\tiny T}}
\renewcommand{\vec}[1]{\mathbf{#1}}

\newtheorem{definition}{Definition}

\newtheorem{lemma}{Lemma}
\newtheorem{proposition}{Proposition}

\newtheorem{theorem}{Theorem}
\newtheorem*{lemma*}{Lemma}
\newtheorem*{theorem*}{Theorem}

\renewcommand{\a}{\vec{a}}
\renewcommand{\b}{\vec{b}}

\newcommand{\lep}[1]{\underline{#1}}

\newcommand{\iab}{I_{\mathrm{AutoBound}}}
\newcommand{\iia}{I_{\mathrm{IntervalArithmetic}}}
\newcommand{\intervals}{\mathbb{IR}}

\newcommand{\rep}[1]{\overline{#1}}

\newcommand{\width}{\mathrm{width}}

\newcommand{\zeros}{\mathbf{0}}

\newcommand{\eye}{\mathbb{I}}
\newcommand{\extable}{\Sigma}

\newcommand{\intervalpolynomials}{\mathbb{IP}}
\newcommand{\op}{\mathsf{op}}
\newcommand{\xop}{\mathsf{Op}}

\newcommand{\pg}{\mathcal{T}}
\newcommand{\primitives}{\mathbb{O}}
\newcommand{\varset}{\mathcal{V}}
\newcommand{\eqlist}{\mathcal{E}}

\newcommand{\mA}{\mat{A}}
\newcommand{\mB}{\mat{B}}
\newcommand{\mC}{\mat{C}}

\newcommand{\mU}{\mat{U}}
\newcommand{\mV}{\mat{V}}
\newcommand{\mW}{\mat{W}}
\newcommand{\mX}{\mat{X}}
\newcommand{\mY}{\mat{Y}}
\newcommand{\mZ}{\mat{Z}}
\newcommand{\sA}{\mathcal{A}}
\newcommand{\sB}{\mathcal{B}}

\newcommand{\sP}{\mathcal{P}}
\newcommand{\sQ}{\mathcal{Q}}
\newcommand{\sS}{\mathcal{S}}
\newcommand{\sU}{\mathcal{U}}
\newcommand{\sV}{\mathcal{V}}
\newcommand{\sW}{\mathcal{W}}
\newcommand{\sx}{X}
\newcommand{\sX}{\mathcal{X}}
\newcommand{\sY}{\mathcal{Y}}
\newcommand{\sZ}{\mathcal{Z}}

\newcommand{\va}{\vec{a}}
\newcommand{\vb}{\vec{b}}

\newcommand{\vx}{\vec{x}}
\newcommand{\vy}{\vec{y}}

\newenvironment{varalgorithm}[1]
  {\algorithm\renewcommand{\thealgorithm}{#1}}
  {\endalgorithm}

\newtheorem{innercustomlemma}{Lemma}
\newenvironment{customlemma}[1]
  {\renewcommand\theinnercustomlemma{#1}\innercustomlemma}
  {\endinnercustomlemma}

\newtheorem{innercustomproposition}{Proposition}
\newenvironment{customproposition}[1]
  {\renewcommand\theinnercustomproposition{#1}\innercustomproposition}
  {\endinnercustomproposition}

\newtheorem{innercustomthm}{Theorem}
\newenvironment{customthm}[1]
  {\renewcommand\theinnercustomthm{#1}\innercustomthm}
  {\endinnercustomthm}

\makeatletter
\newcommand{\algrule}[1][.2pt]{\par\vskip.2\baselineskip\hrule height #1\par\vskip.5\baselineskip}
\makeatother

\title{
Automatically Bounding the Taylor Remainder Series: \\
Tighter Bounds and New Applications}
\author{Matthew Streeter and Joshua V. Dillon\\ 
\vspace{-.2cm} \\
\small{{\tt \href{mailto:mstreeter@google.com}{mstreeter@google.com}, \href{mailto:jvdillon@google.com}{jvdillon@google.com}}}}
\date{}

\begin{document}

\maketitle

\begin{abstract}
Taylor polynomials play a central role in optimization and machine learning, in part because they can be easily computed using automatic differentiation.  In many places where Taylor polynomials are used, it would be advantageous to also have a bound on the remainder series, but techniques for generating such bounds automatically are not as mature as automatic differentiation, and their potential has been less fully explored.

In this work, we present a new algorithm for automatically bounding the Taylor remainder series.  In the special case of a scalar function $f: \reals \to \reals$, our algorithm takes as input a reference point $x_0$, trust region $[a, b]$, and integer $k \ge 1$, and returns an interval $I$ such that $f(x) - \sum_{i=0}^{k-1} \frac {1} {i!} f^{(i)}(x_0)  (x - x_0)^i \in I (x - x_0)^k$ for all $x \in [a, b]$.  As in automatic differentiation, the function $f$ is provided to the algorithm in symbolic form, and must be composed of known atomic functions.

At a high level, our algorithm has two steps:
\begin{enumerate}
  \item For a variety of commonly-used functions (e.g., $\exp$, $\log$, $\mathrm{relu}$, $\mathrm{softplus}$), we use recently-developed theory \cite{streeter2023sharp} to derive \emph{sharp} polynomial upper and lower bounds on the Taylor remainder series.
    \item We recursively combine the bounds for the atomic functions using an interval arithmetic variant of Taylor-mode automatic differentiation.   
\end{enumerate}
Our algorithm can make efficient use of machine learning hardware accelerators, %
and we provide an open source implementation in JAX.\footnote{\url{http://github.com/google/autobound}}

We then turn our attention to applications.  Most notably, in a companion paper \cite{streeter2023universal} we use our new machinery to create the first \emph{universal} majorization-minimization optimization algorithms: algorithms that iteratively minimize an arbitrary loss using a majorizer that is derived \emph{automatically}, rather than by hand.  We also show that our automatically-derived bounds can be used for verified global optimization and numerical integration, and to prove sharper versions of Jensen's inequality.
\end{abstract}

\tableofcontents

\chapter{Introduction} \label{sec:intro}

Taylor polynomials are among the most widely used tools in science and engineering.  They play a central role in numerical optimization, where first and second-order Taylor polynomials are used to iteratively minimize a scalar-valued function.  Because Taylor polynomials can be computed using automatic differentiation, numerical optimizers can be easily applied to very complex functions, such as the training losses used to fit modern machine learning models.

However, a Taylor polynomial provides only a local approximation of a function's behavior, with no guarantees on the accuracy of this approximation at points far from the point $x_0$ at which derivatives are computed.
In applications, this lack of error information is often compensated for by the addition of knobs, whose values must be tuned experimentally.  For example, in gradient-based optimization, trial and error is often required to find a value of the learning rate hyperparameter that is large enough to yield sufficiently fast progress but small enough to avoid divergence.  Though automated methods for bounding the Taylor remainder series exist, they are less mature than automatic differentiation, and have been less widely used.

\begin{figure}[h]
\begin{center}
\includegraphics[width=0.4\linewidth]{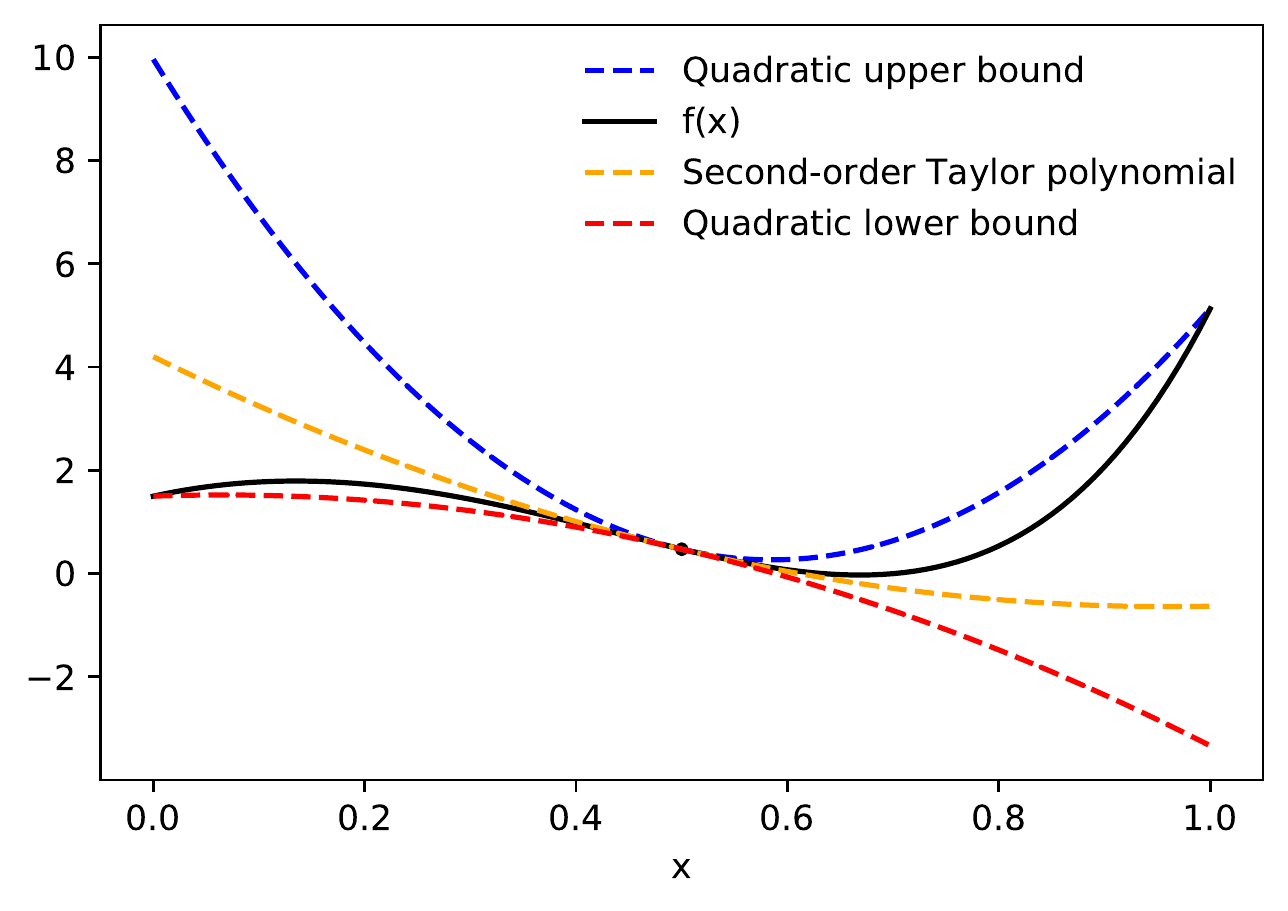}
\caption{Automatically-derived quadratic upper and lower bounds for the function $f(x) = \frac 3 2 \exp(3 x) - 25 x^2$, centered at $x_0 = \frac 1 2$, and valid over the interval $[0, 1]$.}
\label{fig:example_enclosure}
\end{center}
\end{figure}

In this work, we present algorithms for automatically bounding the Taylor remainder series.
Applied to a univariate function $f: \reals \to \reals$, our algorithms will take as input a trust region $[a, b]$, a reference point $x_0 \in [a, b]$, and a polynomial degree $k \ge 1$, and will output an interval $I$ such that
\begin{equation} \label{eq:our_bound}
  f(x) \in \underbrace{ \paren {\sum_{i=0}^{k-1}  \frac {1} {i!} f^{(i)}(x_0) (x - x_0)^i } }_{\text{Degree $k-1$ Taylor polynomial}} + \underbrace{I (x - x_0)^k}_{\text{Remainder bound}} \quad \forall x \in [a, b]
\end{equation}
where $I$ depends on $x_0$, $k$, and $[a, b]$, and is derived from the symbolic expression for $f$.
Note that the right hand side of \eqref{eq:our_bound} is an interval, defined according to the semantics of interval arithmetic,\footnote{The product of an interval $I = [\lep{I}, \rep{I}]$, and a scalar $\alpha$ is defined as $I \alpha \eqdef \set{z \alpha: z \in I} = [ \min \set { \lep{I} \alpha, \rep{I} \alpha} , \max \set { \lep{I} \alpha, \rep{I} \alpha} ]$.} and therefore gives both upper and lower bounds on $f(x)$.
Figure~\ref{fig:example_enclosure} depicts the upper and lower bounds our algorithm produces for the function $f(x) = \frac 3 2 \exp(3 x) - 25 x^2$ at $x_0 = \frac 1 2$, over the trust region $[0, 1]$, with $k = 2$ (which yields quadratic bounds).

We will also be able to compute bounds for vector-variate functions, and the form of these bounds will generalize \eqref{eq:our_bound} in a natural way.

Among other applications, bounds of this form will allow us to automatically derive majorization-minimization (MM) optimization algorithms.  In particular, if we set $k=2$, the upper bound in \eqref{eq:our_bound} is a quadratic majorizer, which can be minimized to obtain a point $x_1$ such that $f(x_1) \le f(x_0)$.  This process can be repeated to further reduce the loss, until a fixed point is reached.

We pay special attention to making the interval $I$ as tight as possible, as this leads to better performance in applications (e.g., faster reduction of the loss in MM optimization).  To encourage others to experiment with our bounds and to find new applications, we provide an open source implementation in JAX.

\section{Outline}

We first present our new approach to automatically bounding the Taylor remainder series.
\begin{itemize}
    \item In Chapter~\ref{chap:arithmetic}, we present the \ref{alg:autobound} algorithm, which recursively combines sharp Taylor polynomial enclosures for atomic univariate functions and elementary binary operations, in order to obtain Taylor polynomial enclosures for arbitrary compositions of these functions.
    \item  Chapter~\ref{chap:bprop} presents the \ref{alg:autoboundprop} algorithm, which generalizes the results of Chapter~\ref{chap:arithmetic} to multivariate functions.  Because the Taylor series for a multivariate function cannot be expressed in standard matrix calculus notation, we must define an appropriate alternative, and this chapter is more notation-heavy than the previous one.  We take care to ensure that \ref{alg:autoboundprop} can make efficient use of modern automatic differentiation frameworks, and discuss our open source implementation in JAX.
\end{itemize}

In Chapter~\ref{chap:applications}, we demonstrate four applications of \ref{alg:autoboundprop}: automatically deriving majorization-minimization optimizers, globally optimizing a loss function via branch and bound, computing upper and lower bounds on the value of an integral, and automatically proving sharper versions of Jensen's inequality.

Chapter~\ref{chap:related_work} discusses related work, and Chapter~\ref{chap:conclusions} summarizes our work and presents conclusions.

\section{An End-to-End Example}

Though presenting the \ref{alg:autoboundprop} algorithm in full generality will require a fair amount of notation, much of the intuition behind the algorithm can be conveyed through a simple example.

Toward this end, suppose we want to come up with quadratic upper and lower bounds on a univariate function $f: \reals \to \reals$, defined by
\begin{equation}
  f(x) = \exp(x^2).
\end{equation}
Further suppose that we require the bounds to be tight at the point $x_0 = 0.2$, and only require these bounds to be valid for $x \in [-0.5, 0.5]$ (the \emph{trust region}).

As in automatic differentiation, we begin by writing the value of $f(x)$ as a sequence of equations, where each equation gives the value of an intermediate variable as an elementary function of one or more intermediate variables:
\begin{align}
  & v_0 = x \nonumber \\
  & v_1 = v_0^2 \nonumber \\
  & v_2 = \exp(v_1).
\end{align}
We then compute the values of the intermediate variables at $x = x_0 = 0.2$:
\begin{align}
  & v_0^{(0)} = 0.2 \nonumber \\
  & v_1^{(0)}= 0.04 \nonumber \\
  & v_2^{(0)} = 1.0408. \label{eq:vals}
\end{align}
Here and throughout the example, numeric values are shown with 5 significant digits.

Next, given our knowledge that $x \in [-0.5, 0.5]$, we compute intervals that enclose the possible values of all the intermediate variables.  Using the rules of interval arithmetic, we obtain
\begin{align}
  & v_0 \in \left [-.5, .5 \right ] \nonumber \\
  & v_1  \in \left [0, .25 \right ] \nonumber \\
  & v_2 \in \left [1, 1.2841 \right ]. \label {eq:trust_regions}
\end{align}

We are now in a position to compute quadratic upper and lower bounds on each intermediate variable, as a function of $x - x_0$.  For $v_0$, we have the trivial equality
\begin{equation} \label {eq:v0}
  v_0 = x_0 + 1 \cdot (x - x_0) = .2 + 1 \cdot \paren { x - .2 }.
\end{equation}
Note that this equality provides (tight) upper and lower bounds on $v_0$, and that these bounds are affine (and hence trivially quadratic) in terms of $x - x_0$.

To obtain quadratic upper and lower bounds for $v_1 = v_0^2$, we square the polynomial on the right hand side of \eqref{eq:v0}, to obtain
\begin{equation} \label{eq:v1}
  v_1 = 0.04 + 0.4 \cdot \paren { x - .2 } + 1 \cdot \paren { x - .2 }^2.
\end{equation}
We now consider $v_2 = \exp(v_1)$.  To bound $v_2$ by quadratics, we require quadratic bounds on the $\exp$ function.  In a companion paper \cite{streeter2023sharp} we developed theory that lets us produce sharp polynomial bounds of arbitrary degree for $\exp$ and other functions.  Applying these theory allows us to show that for $v_1^{(0)} = 0.04$ (calculated in \eqref{eq:vals}) and $v_1 \in [0, .25]$ (calculated in \eqref{eq:trust_regions}),
\begin{equation}
  \exp(v_1)
   \in 1.0408 + 1.0408 \cdot \paren{v_1 - v_1^{(0)}} + \left [ 0.51353,  0.55883 \right ] \cdot \paren { v_1 - v_1^{(0)} }^2.
\end{equation}
Plugging in the expression for $v_1$ from \eqref{eq:v1}, plugging in $v_1^{(0)} = 0.04$, and expanding gives a quartic polynomial bound in terms of $x - x_0$:
\begin{align}
  \exp(v_1)
  & \in 1.0408 + 0.41632 \cdot (x - x_0) + \left [ 1.1230, 1.1302 \right ] (x - x_0)^2  \nonumber \\
  & \quad + \left [ 0.41083, 0.44706 \right ] (x - x_0)^3 + \left [ 0.51353, 0.55883 \right ] (x - x_0)^4.
\end{align}
To obtain a bound that is quadratic (rather than quartic) in terms of $x - x_0$, we again use knowledge of the fact that $x \in [-.5, .5]$ to infer that $x - x_0 \in [-.7, .3]$ (recall $x_0 = 0.2$), and therefore $(x - x_0)^3 \in [-.7, .3] (x - x_0)^2$, while $(x - x_0)^4 \in [0, .49] (x - x_0)^2$.  Making these substitutions, and collecting terms yields the quadratic bounds:
\begin{equation}
  \exp(v_1) \in 1.0408 + 0.41632 \cdot (x - x_0) + \left [0.81728 , 1.5382 \right ] (x - x_0)^2.
\end{equation}

Figure~\ref{fig:exp_x2_enclosure} plots the function $f(x) = \exp(x^2)$, together with the quadratic upper and lower bounds we have just derived.

\begin{figure}[h]
\begin{center}
\includegraphics[width=0.4\linewidth]{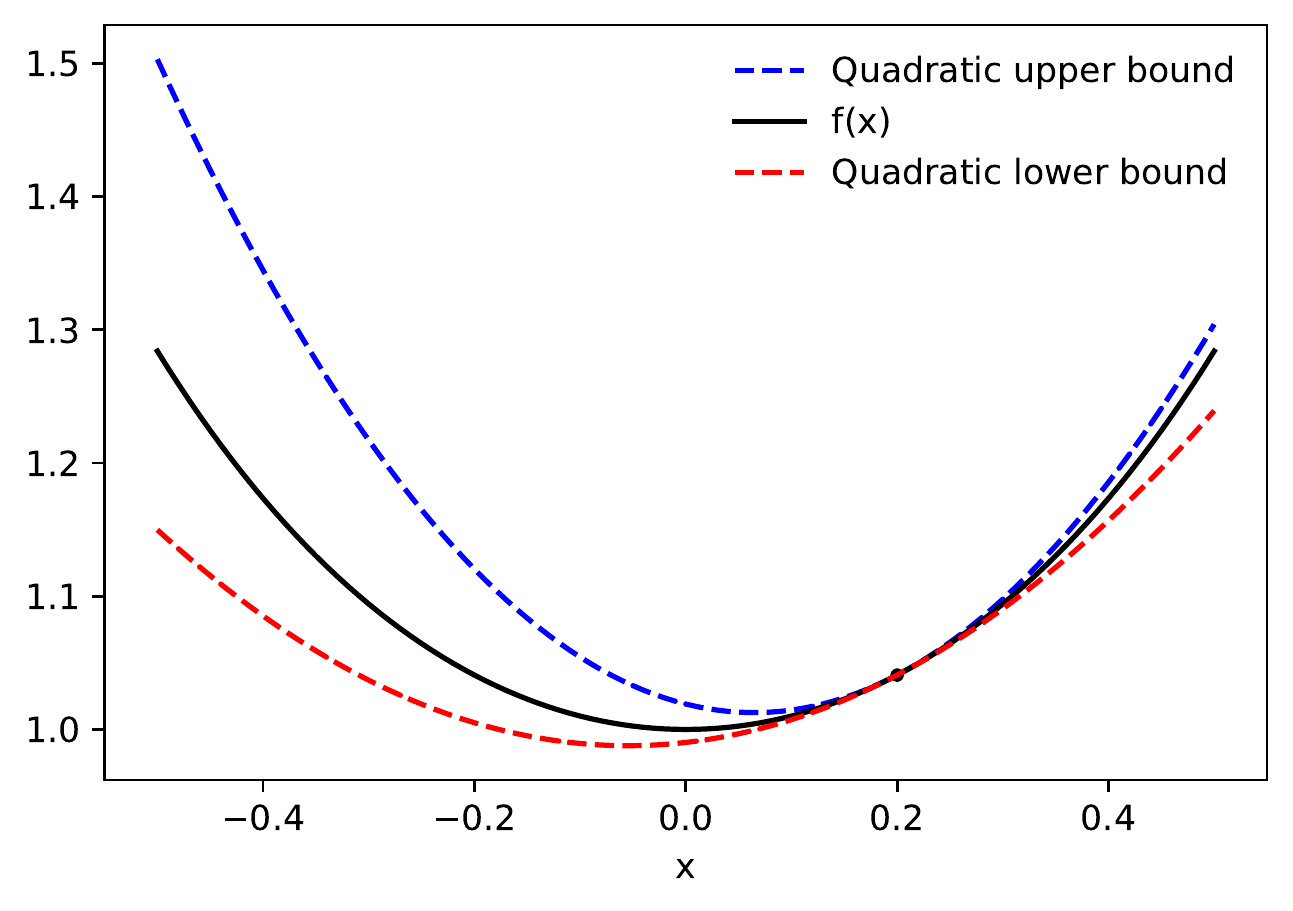}
\caption{Quadratic upper and lower bounds for the function $f(x) = \exp(x^2)$, centered at $x_0 = 0.2$, and valid over the interval $[-0.5, 0.5]$.}
\label{fig:exp_x2_enclosure}
\end{center}
\end{figure}

\chapter{Taylor Polynomial Enclosures for One-Dimensional Functions} \label{chap:arithmetic}

The goal of this work is to automatically derive polynomial upper and lower bounds of the form given in \eqref{eq:our_bound}, which we call \emph{Taylor polynomial enclosures}.  In this chapter, we will consider the problem of deriving Taylor polynomial enclosures for one-dimensional functions that can be written in terms of atomic univariate functions, plus the elementary binary operations of addition, subtraction, multiplication, division, and exponentiation.

To do so, we will make use of known \emph{sharp} Taylor polynomial enclosures for various elementary functions, which we derived in a companion paper \cite{streeter2023sharp}.
We would like to combine these known enclosures to obtain enclosures for more complex functions.  As an example, having already derived sharp Taylor polynomial enclosures for $\exp$, we would like to be able to derive a Taylor polynomial enclosure of the function
\begin{equation} \label{eq:example}
  f(x) = \frac { \exp(x) } {2 + x}.
\end{equation}
To do so, we will see that it suffices to come up with rules for performing various operations on \emph{interval polynomials}.  We develop such rules in \S\ref{sec:ab_algorithm}, after first reviewing relevant background knowledge.

\section{Background: Taylor Mode Automatic Differentiation} \label{sec:univariate_background}

The algorithm we will develop in this chapter can be thought of as a generalization of Taylor-mode automatic differentiation, and to understand our algorithm it may be helpful to first review this method.

Applied to a univariate function $f: \reals \to \reals$, Taylor mode automatic differentiation computes the coefficients of the degree $k$ Taylor polynomial of $f$ at some point $x_0 \in \reals$, which is equivalent to computing $f^{(i)}(x_0)$ for $i \in \set{0, 1, \ldots, k}$.

One can think of Taylor mode automatic differentiation as an evaluation of $f$ using overloaded operators whose input and output are Taylor polynomials, rather than scalars.  A key insight behind this approach is that the degree $k$ Taylor polynomial for a composite function $f(x) = g(h(x))$ at $x_0$ can be obtained by first computing the degree $k$ Taylor polynomial for $h$ at $x_0$, then plugging it into the degree $k$ Taylor polynomial for $g$ at $h(x_0)$, then expanding the result and dropping all terms of degree $> k$.  Thus, if $f = \sigma_1 \circ \sigma_2 \circ \ldots \circ \sigma_n$, where each $\sigma_i$ is an atomic function (whose Taylor polynomial can be computed in closed form), we can compute a Taylor polynomial for $f$ recursively, at each step plugging the Taylor polynomial for $\sigma_{i+1} \circ \sigma_{i+2} \circ \ldots \circ \sigma_n$ into the Taylor polynomial for $\sigma_i$.  Thus can be thought of as evaluating $f$ on the identity polynomial, using an extended version of each $\sigma_i$ that inputs and outputs Taylor polynomials rather than scalars.  This approach can be extended to functions composed of operations that take two or more arguments.

For a formal treatment of Taylor mode automatic differentiation and additional discussion, see \cite{bettencourt2019taylor,griewank2008evaluating}.

\section{Definitions and Notation} \label {sec:ab_definitions}

We denote the set of closed real intervals by $\intervals \eqdef \set { [a, b]: a, b \in \reals, a \le b }$.  Mathematical expressions containing intervals represent sets according to the usual rules of interval arithmetic \cite{moore1966interval}.  In particular, the product of an interval $I \in \intervals$ and scalar $\alpha \in \reals$ is defined as $I \alpha \eqdef  \set{z \alpha: z \in I}$, and $\alpha I$ is defined identically.

\begin{definition} [Interval polynomial]
A \emph{degree $k$ interval polynomial} is a function $F: \reals \to \intervals$, of the form:
\[
  F(z) = \sum_{i=0}^k F_{[i]} z^k
\]
where each coefficient $F_{[i]} \in \intervals \cup \reals$ is either an interval or a scalar.
\end{definition}

For any interval polynomial $F$, we use $F_{[i]}$ to denote its $i$th coefficient (starting at $i = 0$, so that $F(z) = \sum_{i=0}^k F_{[i]} z^i$, where $k$ is the degree of $F$).  We use $\intervalpolynomials(k)$ to denote the set of degree $k$ interval polynomials.

\begin{definition} [Interval polynomial enclosure] \label{def:ipif}
A degree $k$ interval polynomial $F \in \intervalpolynomials(k)$ is an \emph{enclosure} for a function $f: \reals \to \reals$ over an interval $Z \in \intervals$ if
\[
  f(z) \in F(z) \quad \forall z \in Z.
\]
\end{definition}
As an example, the interval polynomial $F(z) = z + [0, 1] z^2$ is a degree 2 interval polynomial enclosure of the function $f(z) = z + \sin(z) z^2$ over the interval $Z = [0, \pi]$. 

Our goal in this work is to derive a \emph{Taylor polynomial enclosure}, which can be thought of a particular kind of interval polynomial enclosure in which all but the last coefficient is a scalar.
\begin{definition} [Taylor polynomial enclosure] \label{def:tpe}
A degree $k$ \emph{Taylor polynomial enclosure} of a function $f: \reals \to \reals$ at a point $x_0 \in \reals$ over an interval $[a, b]$ is an interval polynomial of the form
\[
  F(x) = \underbrace{ \paren {\sum_{i=0}^{k-1}  \frac {1} {i!} f^{(i)}(x_0) (x - x_0)^i } }_{\text{Degree $k-1$ Taylor polynomial}} + \underbrace{I (x - x_0)^k}_{\text{Remainder bound}}
\]
where $I \in \intervals$ is an interval, and $f(x) \in F(x)$ for all $x \in [a, b]$.
\end{definition}
Observe that Taylor polynomial enclosures can be mapped to interval polynomial enclosures via the change of variable $z = x - x_0$.
Although we are mainly interested in deriving Taylor polynomial enclosures, we will find it convenient to work with interval polynomial enclosures when defining our recursive algorithm.

The key to our algorithm will be to define, for each atomic scalar-valued function of interest, an \emph{interval polynomial extension} that inputs and outputs interval polynomials rather than scalars.  The formal definition of an interval polynomial extension ends up being somewhat technical, but is immediately followed by a concrete example designed to make it more clear.

\begin{definition} [Interval polynomial extension] \label{def:ip_extension}
For integers $n > 0$ and $k \ge 0$, and an $n$-argument scalar-valued function $f: \reals^n \to \reals$, the function $F: \intervalpolynomials(k)^n \to \intervalpolynomials(k)$ is an \emph{interval polynomial extension} of $f$ over trust regions $Y_1, Y_2, \ldots, Y_n, Z \in \intervals$ iff.\ the following condition holds: For any $z \in Z$, any interval polynomials $P_1, P_2, \ldots, P_n \in \intervalpolynomials(k)$, and any $y_1, y_2, \ldots, y_n \in \reals$ with $y_i \in Y_i$ and $y_i \in P_i(z)$ for all $i$, we have
\[
  f(y_1, y_2, \ldots, y_n) \in F(P_1, P_2, \ldots, P_n)(z).
\]
\end{definition}

\noindent{\bf Example.}  Definition~\ref{def:ip_extension} is rather subtle, and in particular it may not be obvious why it should require trust regions for both the arguments of $f$ and the independent variable $z$.
To illustrate the definition, consider the function $f: \reals \to \reals$, defined by
\begin{equation}
  f(y) \eqdef \begin{cases}
  y^2 & \mbox {if } y \in [-1, 1] \\
  |y| & \mbox {otherwise.}
\end{cases}
\end{equation}
Let $F: \intervalpolynomials(2) \to \intervalpolynomials(2)$ be defined by
\begin{equation}
  F(A) = A_{[0]}^2 + 2 A_{[0]} A_{[1]} z +  \paren { A_{[1]}^2 + 2 A_{[1]} A_{[2]} [0, .5] + A_{[2]}^2 [0, .25] } z^2.
\end{equation}
Then, $F$ is a degree 2 interval polynomial extension of $f$ over the trust regions $Y = [-1, 1]$ and $Z = [0, .5]$.
This is true because, for any $z \in Z$, any degree 2 interval polynomial $A(z) \eqdef A_{[0]} + A_{[1]} z + A_{[2]} z^2$, and any $y \in Y$ with $y \in A(z)$,
\begin{align}
  f(y) & = y^2  & \mbox {because } y \in [-1, 1] \nonumber \\
  & \in ( A_{[0]} + A_{[1]} z + A_{[2]} z^2 )^2 & \mbox {because } y \in A(z) \nonumber \\
  & \subseteq A_{[0]}^2 + 2 A_{[0]} A_{[1]} z + z^2 \paren {A_{[1]}^2 + 2 A_{[1]} A_{[2]} z + A_{[2]}^2 z^2 } \nonumber \\
  & \subseteq  A_{[0]}^2 + 2 A_{[0]} A_{[1]} +  z^2 \paren {A_{[1]}^2 + 2 A_{[1]} A_{[2]} [0, .5] + A_{[2]}^2 [0, .25] } & \mbox {because } z \in [0, .5] \nonumber \\
  & = F(A).
\end{align}

Given an arbitrary function $f$ which is composed of atomic functions for which we have derived interval polynomial extensions, we will be able to recursively compute a Taylor polynomial enclosure of $f$ by evaluating $f$ on the identity polynomial using the extended functions.  The validity of this approach is captured in the following proposition, which follows immediately from the definitions given so far.

\begin{proposition} \label{prop:compose_enclosures}
Consider a function $f: \reals \to \reals$ of the form $f(z) = \op(a_1(z), a_2(z), \ldots, a_n(z))$, and an arbitrary trust region $Z \in \intervals$, and let the following conditions hold:
\begin{enumerate}
  \item For all $i$, we have $a_i(z) \in Y_i$ for all $z \in Z$.
  \item For all $i$, $P_i$ is an interval polynomial enclosure of $a_i$ over $Z$.
  \item $\xop$ is an interval polynomial extension of $\op$ over $Y_1, Y_2, \ldots, Y_n$ and $Z$.
\end{enumerate}
Then, the interval polynomial $F \eqdef \xop(P_1, P_2, \ldots P_n)$ is enclosure of $f$ over $Z$.
\end{proposition}
\begin{proof}
Consider some arbitrary $z \in Z$, and let $y_i = a_i(z)$.  By conditions 1 and 2, $y_i \in Y_i$ and $y_i \in P_i(z)$.  Thus, using condition 3 and Definition~\ref{def:ip_extension}:
\begin{equation}
  f(z) = \op(y_1, y_2, \ldots, y_n) \in \xop(P_1, P_2, \ldots P_n)(z) = F(z).
\end{equation}
Because $z \in Z$ was arbitrary, it follows that $F$ is an interval polynomial enclosure of $f$ over $Z$.
\end{proof}

\section{The AutoBound1D Algorithm} \label{sec:ab_algorithm}

We now present an new algorithm, \ref{alg:autobound}, for deriving a Taylor polynomial enclosure of an arbitrary univariate function $f: \reals \to \reals$ that can be written in terms of commonly-used atomic functions.  In Chapter~\ref{chap:bprop}, we generalize this algorithm to handle multivariate functions.

At a high level, \ref{alg:autobound} is similar to the algorithm defined in previous work on Taylor models \cite{makino1996remainder,makino1999efficient,makino2001higher,makino2003taylor},  in that both algorithms compose interval polynomials in a manner analogous to Taylor mode automatic differentiation.  However, there are several critical differences:

\begin{enumerate}
  \item \ref{alg:autobound} produces bounds of the form $R_{k-1}(x; f, x_0) \in I (x - x_0)^{k}$, in contrast to Taylor model bounds of the form $R_{k-1}(x; f, x_0) \in I$.  This makes it possible for the bounds derived by \ref{alg:autobound} to be used as the basis of a majorization-minimization algorithm (see \cite{streeter2023universal}), and has potential advantages in other applications as well (see Chapter~\ref{chap:applications}). %
  \item Where available, we use \emph{sharp} Taylor polynomial enclosures for the atomic functions (obtained using the theory developed in \cite{streeter2023sharp}), whereas previous work on Taylor models made use of looser enclosures based on derivative bounds.
  \item When composing two interval polynomials, we use a dedicated rule that produces tighter bounds than the simpler approach of repeatedly applying an interval polynomial multiplication rule.
\end{enumerate}

The remainder of this section is organized as follows.  In \S\ref{sec:interval_polynomial_extensions} we define interval polynomial extensions of various commonly used functions.  In \S\ref{sec:ab_pseudocode} we present pseudocode for \ref{alg:autobound}, prove its correctness, and give a worked example.

\subsection{Interval Polynomial Extensions of Atomic Functions} \label{sec:interval_polynomial_extensions}

The \ref{alg:autobound} algorithm will derive a Taylor polynomial enclosure of a function $f: \reals \to \reals$ by evaluating $f$ using \emph{interval polynomial extensions} of the atomic functions of which $f$ is composed.  As formalized in Definition~\ref{def:ip_extension}, these extended functions input and output interval polynomials of some fixed degree $k$, with the output interval polynomial satisfying a natural inclusion property that guarantees correctness of the algorithm.

In the following subsections, we define interval polynomial extensions of scalar arithmetic operations (such as additional and multiplication) and of arbitrary non-linear scalar functions (such as $\exp$ or $\log$).

To do so, we will use the following basic properties of interval arithmetic \cite{moore1966interval}, which we state without proof.  
For intervals $X, Y, Z \in \intervals$, and a scalar $\alpha \in \reals$, we have
\begin{equation} \label{eq:interval_sums_commutative}
  X + Y = Y + X
\end{equation}
\begin{equation} \label{eq:scalar_interval_product_distributive}
  \alpha(X + Y) = \alpha X + \alpha Y
\end{equation}
\begin{equation} \label{eq:interval_product_subdistributive}
  X (Y + Z) \subseteq X Y + X Z.
\end{equation}
Note that \eqref{eq:interval_product_subdistributive} does not hold with equality.  As a counterexample, if $X = [-3, 3]$, $Y = [1, 1]$, and $Z = [-1, -1]$, then $X(Y + Z) = [0, 0]$, while $X Y + X Z = [-3, 3] + [-3, 3] = [-6, 6]$.

In the subsections that follow, we state our extended functions in terms of two generic interval polynomials:
\begin{equation}
  A(z) = \sum_{i=1}^k A_{[i]} z^i, \quad \quad B(z) = \sum_{i=1}^k B_{[i]} z^i.
\end{equation}

\newcommand{\rangebound}{\ensuremath{\mathrm{RangeBound}}}

\subsubsection{Background: Bounding the Range of an Interval Polynomial}

Our interval polynomial extensions will be defined in terms of a function $\rangebound$, which bounds the range of an interval polynomial over a given trust region.  That is, for any interval polynomial $P$ and interval $Z \in \intervals$,
\begin{equation} \label{eq:range_bound}
	z \in Z \implies P(z) \subseteq \rangebound(P, Z).
\end{equation}
A simple way to bound the range of an interval polynomial is to evaluate the interval polynomial at $Z$ using interval arithmetic, defining
\begin{equation} \label{eq:default_range_bound}
	\rangebound(P, Z) \eqdef \sum_{i=0}^{\mathrm{degree}(P)} P_{[i]} Z^i.
\end{equation}
However, there are a variety of other possibilities which make different tradeoffs between tightness and computation.  See \cite{rokne1977bounds} for a numerical comparison of several different approaches, and see \cite{makino2005verified} for discussion of two approaches that have proven useful in global optimization algorithms.

\subsubsection{Addition, Subtraction, and Multiplication}

For interval polynomials $A(z) = \sum_{i=1}^k A_{[i]} z^i$ and $B(z) = \sum_{i=1}^k B_{[i]} z^i $, we have the following equalities:
\begin{equation} \label {eq:ipoly_sum}
  A(z) + B(z) = \sum_{i=0}^k (A_{[i]} + B_{[i]}) z^k %
\end{equation}
\begin{equation} \label {eq:ipoly_difference}
  A(z) - B(z) = \sum_{i=0}^k (A_{[i]} - B_{[i]}) z^k. %
\end{equation}
These equalities follow from the commutativity property \eqref{eq:interval_sums_commutative} and the distributivity property \eqref{eq:scalar_interval_product_distributive}.  Because the right hand sides are degree $k$ interval polynomials, these equalities immediately define interval polynomial extensions of the scalar addition and subtraction operations (over any trust regions).

For scalar multiplication, we first use the sub-distributivity property \eqref{eq:interval_product_subdistributive} to write
\begin{equation} \label{eq:ipoly_product}
  A(z) B(z) \subseteq \sum_{i=0}^k \sum_{j=0}^k A_{[i]} B_{[j]} z^{i+j}.  %
\end{equation}
To define an interval polynomial extension of scalar multiplication, we must enclose the degree $2 k$ interval polynomial on the right hand side of \eqref{eq:ipoly_product} by a degree $k$ polynomial.  To do so, we collect the terms of degree $\ge k$ in a single polynomial, then use a $\rangebound$ function satisfying \eqref{eq:range_bound} to reduce the degree of this polynomial to $k$.  Letting $S \eqdef \set{(i, j): i, j \in \set{0, 1, \ldots , k}}$, for $z \in Z$ we have:
\begin{align}
  A(z) B(z)
  & \subseteq \paren { \sum_{(i, j) \in S: i+j < k} A_{[i]} B_{[j]} z^{i+j} } + z^k \paren { \sum_{(i, j) \in S: i+j \ge k} A_{[i]} B_{[j]} z^{i+j-k}  } \nonumber \\
  & \subseteq \paren { \sum_{(i, j) \in S: i+j < k} A_{[i]} B_{[j]} z^{i+j} } + z^k \cdot \rangebound\paren { \sum_{(i, j) \in S: i+j \ge k} A_{[i]} B_{[j]} z^{i+j-k} , Z }. \label{eq:tipoly_product}
\end{align}
Equation \eqref{eq:tipoly_product} defines an interval polynomial extension of scalar multiplication for any trust regions $Y_1, Y_2$ and $Z$, with the result depending only on $Z$.

We note that equations \eqref{eq:ipoly_sum}, \eqref{eq:ipoly_difference}, and \eqref{eq:ipoly_product} appeared previously in a paper by Rokne \cite{rokne1975reducing}, along with a rule for interval polynomial division.

\subsubsection{Exponentiation with a Non-Negative Integer Exponent} \label{sec:ab_exponentiation}

We now consider computing $A(z)^p$, for an integer $p \ge 0$.  For $p \in \set{0, 1}$ this is trivial.  For $p \ge 2$, a natural approach is to simply apply the multiplication rule $p - 1$ times.  However, this approach produces an unnecessarily loose bound.  For example, consider squaring the degree 0 interval polynomial $A(z) = [-3, 3]$.  Using the multiplication rule gives
\begin{equation}
  A(z) A(z) = [-3, 3] \cdot [-3, 3] = [-9, 9].
\end{equation}
In contrast, using the exponentiation rule for interval arithmetic gives $A(z)^2 = [-3, 3]^2 = [0, 9]$.

To obtain a tighter bound than the one given by repeated multiplication, we use the exponentiation rule (rather that the product rule) for interval arithmetic whenever possible.  For example, in the case where $p = 2$ and $k = 2$, for $z \in Z$ we have
\begin{align}
  A(z)^2
  & = \paren{A_{[0]} + A_{[1]} z + A_{[2]} z^2}^2 \nonumber \\
  & \subseteq A_{[0]}^2 + 2 A_{[0]} A_{[1]} z + z^2 \paren{2 A_{[0]} A_{[2]} + A_{[1]}^2  + 2 A_{[1]} A_{[2]} z + A_{[2]}^2 z^2}  \nonumber \\
  & \subseteq A_{[0]}^2 + 2 A_{[0]} A_{[1]} z + z^2 \cdot \rangebound\paren{2 A_{[0]} A_{[2]} + A_{[1]}^2  + 2 A_{[1]} A_{[2]} z + A_{[2]}^2 z^2, Z}.
\end{align}
By computing the intervals $A_{[0]}^2$, $A_{[1]}^2$, and $A_{[2]}^2$ using the exponentiation rule (rather than the product) rule, we obtain a tighter result than would be obtained by computing $A(z)^2$ using \eqref{eq:tipoly_product}.  This approach generalizes naturally to other $k$ and $p$.

\subsubsection{Arbitrary Univariate Functions} \label{sec:nonlinear_ops}

We next define an interval polynomial extension of an arbitrary function $\sigma: \reals \to \reals$ for which a Taylor polynomial enclosure is known.  This immediately provides interval polynomial extensions of the all functions considered in \cite{streeter2023sharp}, including $\exp$, $\log$, $x \mapsto x^p$, and various neural network activation functions.

To define the extended function, we first consider the problem of composing two interval polynomials.  We immediately have
\begin{equation}
  A(B(z)) = \sum_{i=0}^k A_{[i]} ( B(z) )^i.
\end{equation}
Thus, for $z \in Z$, we could obtain an enclosure of $A(B(z))$ by applying the rule for non-negative integer exponentiation $k$ times, and summing the results. However, we can obtain a tighter enclosure by first rewriting $A(B(z))$ as a degree $k^2$ polynomial, then enclosing it by a degree $k$ polynomial using a single call to $\rangebound$, as in \eqref{eq:tipoly_product}.  As in Taylor-mode automatic differentiation, this can be implemented efficiently using Fa\`{a} di Bruno's formula \citep{bettencourt2019taylor}.  We denote the resulting enclosure by $A \circ_Z B$.

Applied to an input polynomial $P$, the extended function will compute a degree $k$ Taylor polynomial enclosure of $\sigma$, and then compose it with the polynomial $P - P_{[0]}$.  To define it formally, for any $y_0 \in \reals$ and $Y \in \intervals$, let $\pg_k(\sigma, y_0, Y)$ denote the known degree $k$ Taylor polynomial enclosure of $\sigma$ at $y_0$ over $Y$.  The degree $k$ interval polynomial extension of $\sigma$ over trust regions $Y_1$ and $Z$ is defined by:
\begin{equation} \label{eq:ipe_nonlinear}
  \Sigma(P) = \pg_k(\sigma, P_{[0]}, Y_1) \circ_Z (P - P_{[0]})
\end{equation}
where $P_{[0]}$ denotes the degree 0 term of $P$ (here assumed to be a scalar), and we use $A \circ_Z B$ to denote interval polynomial composition that is valid for $z \in Z$, as described in the previous paragraph.

\subsubsection{Exponentiation with a fractional or negative exponent}

For $p \in \reals$, where $p$ is not an integer or $p < 0$, we define $A(z)^p$ as the application of the nonlinear function $x \mapsto x^p$ to the interval polynomial $A(z)$ using \eqref{eq:ipe_nonlinear}.

\subsubsection{Division}

Having defined interval polynomial extensions of multiplication and exponentiation, we simply treat division as multiplication by the reciprocal:
\begin{equation} \label{eq:tipoly_division}
  \frac{A(z)}{B(z)} \eqdef A(z) B(z)^{-1}
\end{equation}
where $B(z)^{-1}$ is defined as the application of the nonlinear function $x \mapsto x^{-1}$ to $B(z)$ using \eqref{eq:ipe_nonlinear}.

Note that if it is possible for $B(z)$ to be 0 (i.e., $0 \in Y_2$) then \eqref{eq:ipe_nonlinear} will produce an interval polynomial with infinite coefficients, and the right hand side of \eqref{eq:tipoly_division} will have coefficients that are either infinite or indeterminate (i.e, NaN), as appropriate, depending on the coefficients of $A(z)$.

\subsubsection{Summary}

Table~\ref{tab:interval_polynomial_extensions} summarizes the interval polynomial extensions we have derived for addition, multiplication, and arbitrary univariate functions.  

\begin{table*}[h]
        \caption{Interval polynomial extensions of atomic functions.}
        \label{tab:interval_polynomial_extensions}
  \centering
        \begin{small}
        \begin{sc}
                \begin{tabular}{llll}
    \toprule
          Function & Degree $k$ interval polynomial extension over $Y_1, Y_2, \ldots$ and $Z$ \\
    \midrule
      $\sigma(a, b) = a + b$ & $\Sigma(A, B) \eqdef (A_{[0]} + B_{[0]}, A_{[1]} + B_{[1]}, \ldots, A_{[k]} + B_{[k]})$ \vspace{.4cm} \\ 
      $\sigma(a, b) = a b$ & \makecell[l]{$\Sigma(A, B) = Q$, where\\$\quad Q_{[j]} = \begin{cases}
      \sum_{l, m: l + m = j} A_{[l]} B_{[m]} & j < k\\
      \rangebound\paren{\sum_{l, m: l+m \ge k}  A_{[l]} B_{[m]} z^{l+m-k}, Z} & j = k.
      \end{cases}$} \vspace{.4cm} \\

      Univariate $\sigma$ & \makecell[l]{$\Sigma(A) = P \circ_Z (A - A_{[0]})$ (see \S\ref{sec:nonlinear_ops}), where $P$ is a degree $k$ Taylor\\ enclosure of $\sigma$ at $A_{[0]}$ over $Y_1$.} \\

\bottomrule
  \end{tabular}
  \end{sc}
  \end{small}
\end{table*}

For simplicity, we have omitted the extensions for subtraction, division, and exponentiation with a constant exponent, which can be defined in terms of the extensions given in the table.

\subsection{Pseudocode and Example} \label{sec:ab_pseudocode}

We now formally present the \ref{alg:autobound} algorithm.

The algorithm takes as input a function $f: \reals \to \reals$, represented as a \emph{symbolic expression}.  A symbolic expression is a sequence of equations, each of which gives the value of some intermediate variable as an atomic function of other intermediate variables.

We use $\primitives$ to denote the set of atomic univariate functions that may appear in a symbolic expression.
To keep our pseudocode as simple as possible, we will assume that the only binary operations that appear in a symbolic expression are $+$ and $\times$.  Functions containing subtraction or division can be converted to this form as a preprocessing step, using the relations $x - y = x + (-y)$ and $\frac x y = x y^{-1}$.  Functions containing exponentiation can also be converted to this form so long as each exponent is a constant, by letting $\primitives$ contain a univariate function $x \mapsto x^p$ for each distinct exponent $p$.

\begin{definition} [Symbolic expression]
A \emph{symbolic expression} is a pair $(\varset, \eqlist)$, where $\varset = (v_0, v_1, \allowbreak \ldots, v_n)$ is a tuple of intermediate variables, and $\eqlist = (e_1, e_2, \ldots, e_n)$ is a tuple of equations, where $e_i = (\sigma_i, L_i)$, $\sigma_i \in \primitives$ is an atomic function, and $L_i$ is a tuple of intermediate variables, of the same length as the number of arguments that $\sigma_i$ takes.
\end{definition}
In a symbolic expression, $v_0$ represents the input to the function, and $v_n$ represents the output. As an example, the function $f(x) = \sqrt{\exp(x)}$ can be represented by the symbolic expression $(\varset, \eqlist)$ where $\varset = (v_0, v_1, v_2)$ and $\eqlist = ( (\exp, (v_0) ), (\mathrm{sqrt}, (v_1)) )$, representing the equations:
\begin{align}
  v_0 & = x \nonumber \\
  v_1 & =  \exp(v_0) \nonumber \\
  v_2 & = \sqrt{v_1}.
\end{align}

We now describe \ref{alg:autobound}.
Given as input an integer $k$, a trust region $[a, b] \in \intervals$, a reference point $x_0 \in [a, b]$, and a symbolic expression $(\varset, \eqlist)$ with $\varset = (v_0, v_1, v_2, \ldots, v_n)$, \ref{alg:autobound} computes intervals $Y_0, Y_1, \ldots, Y_n$ such that
\begin{equation} \label{eq:Y_invariant}
  v_i (x) \in Y_i \quad \forall x \in [a, b]
\end{equation}
where $v_i(x)$ denotes the value of the intermediate variable $v_i$ as a function of the independent variable $x$.  The intervals $Y_i$ are obtained from the equations in $\eqlist$ using the rules of interval arithmetic.

In parallel, the algorithm computes degree $k$ interval polynomials $P_0, P_1, \ldots, P_n$ (represented as tuples of coefficients) such that
\begin{equation} \label{eq:P_invariant}
  v_i(x) \in P_i(x - x_0) \quad \forall x \in [a, b].
\end{equation}
The interval polynomials $P_i$ are obtained from the equations in $\eqlist$, using the interval polynomial extensions of the atomic functions (summarized in Table~\ref{tab:interval_polynomial_extensions}).  The algorithm returns $P_n$, which is a Taylor polynomial enclosure of the function $f: \reals \to \reals$ defined by the symbolic expression $(\varset, \eqlist)$, at $x_0$ over $[a, b]$.

Theorem~\ref{thm:autobound_1d} shows that \ref{alg:autobound} returns the coefficients of a Taylor polynomial enclosure of the function $f$ represented by the symbolic expression provided as input to the algorithm.

\newcommand{\thmautoboundoned}{
Assume that the table $\extable$ provided as a hyperparameter to \ref{alg:autobound} contains interval polynomial extensions of each primitive function in $\primitives$.   (That is, for any primitive function $\sigma: \reals^m \to \reals \in \primitives$, integer $k > 0$, and intervals $Y_1, Y_2, \ldots, Y_m, Z \in \intervals$, the function $F(z) = \extable[\sigma, k](z; Y_1, Y_2, \ldots, Y_m, Z)$ is an interval polynomial extension of $\sigma$ over $Y_1, Y_2, \ldots, Y_m$ and $Z$.)

Then, when given as input a symbolic expression $(\varset, \eqlist)$, an interval $[a, b] \in \intervals$, a scalar $x_0 \in [a, b]$, and a target degree $k > 0$, \ref{alg:autobound} returns a tuple containing the coefficients of a degree $k$ interval polynomial $P_n$ that satisfies:
\[
  f(x) \in P_n(x - x_0) \quad \forall x \in [a, b]
\]
where $f: \reals \to \reals$ is the function represented by the symbolic expression $(\varset, \eqlist)$.
Furthermore, if the first $k - 1$ coefficients of $P_n$ are scalars (which is the case when $\Sigma$ is set to its default value), then $P_n$ is a Taylor polynomial enclosure of $f$ at $x_0$ over $[a, b]$.
}
\newcommand{\thmautoboundonedproof}{

Let $Y_i$ and $P_i$, and be defined as in the pseudocode for \ref{alg:autobound}, and let $v_i(x)$ denote the value of intermediate variable $v_i$ as a function of $x$.
We will show by induction that, for $i = 0, 1, \ldots, n$, the following two invariants hold:
\begin{enumerate}
  \item $v_i(x) \in Y_i \quad \forall x \in [a, b]$.
  \item $P_i$ is an interval polynomial enclosure of $a_i(z) \eqdef v_i(x_0 + z)$ over $Z \eqdef [a - x_0, b - x_0]$.%
\end{enumerate}
For the base case $i = 0$, we have $v_0(x) = x$ and $Y_0 = (a, b)$, so invariant (1) holds trivially.  Because $P_0(z) = x_0 + z = v_0(x_0 + z) = a_0(z)$,  invariant (2) holds trivially as well.

Now consider some arbitrary $i > 0$, and assume that invariants (1) and (2) hold for all smaller $i$.  Letting $L_i = (v_{j_1}, v_{j_2}, \ldots, v_{j_m})$ (as in the pseudocode), we have:
\begin{align}
  Y_i
  & \subseteq Y_i^{(0)} \nonumber \\
  & = \sigma_i(Y_{j_1}, Y_{j_2}, \ldots, Y_{j_m}) \nonumber \\
  & \eqdef \set { \sigma_i(y_{j_1}, y_{j_2}, \ldots, y_{j_m}): y_{j_l} \in Y_{j_l}\ \forall l \in \set{1, 2, , \ldots, m} } \nonumber \\
  & \ni \sigma_i(v_{j_1}(x), v_{j_2}(x), \ldots, v_{j_m}(x)) \nonumber \\
  & = v_i(x)
\end{align}
where the fourth line uses the induction hypothesis.  Thus, invariant (1) holds.  To show that invariant (2) also holds, first observe that
\begin{equation}
  a_i(z) = \sigma_i(a_{j_1}(z), a_{j_2}(z), \ldots, a_{j_m}(z)).
\end{equation}
By the induction hypothesis, $P_{j_l}$ is an interval polynomial enclosure of $a_{j_l}$ over $Z$ for all $l  \in \set{1, 2, \ldots, m}$, and $a_{j_l}(z) \in Y_{j_l}$ for all $z \in Z$.  Furthermore, by the assumptions about $\Sigma$ made in the theorem statement, the code sets $P_i = F(P_{j_1}, P_{j_2}, \ldots, P_{j_m})$, where $F$ is an interval polynomial extension of $\sigma_i$ over $Y_{j_1}, Y_{j_2}, \ldots, Y_{j_m}$ and $Z$.  It follows by Proposition~\ref{prop:compose_enclosures} that invariant (2) holds.

We have thus shown that invariants (1) and (2) hold for all $i$.  Taking $i = n$, invariant (2) says that $P_n$ is an interval polynomial enclosure of $a_n(z) = v_n(x_0 + z) = f(x_0 + z)$ over $Z = [a - x_0, b - x_0]$.  Making the change of variable $x = z + x_0$ completes the first part of the proof.

The second claim (that $P_n$ is a Taylor polynomial enclosure if its first $k-1$ coefficients are scalars) follows from Lemma~\ref{lem:when_ip_is_tif}.
}
\begin{theorem} \label{thm:autobound_1d}
\thmautoboundoned
\end{theorem}

The proof of Theorem~\ref{thm:autobound_1d} uses induction to show that invariants \eqref{eq:Y_invariant} and \eqref{eq:P_invariant} are maintained at each step of the algorithm.  Applying invariant \eqref{eq:P_invariant} with $i = n$ then proves the first claim in the theorem.  The second claim (that $P_n$ is a Taylor polynomial enclosure if its first $k-1$ coefficients are scalars) follows from an analysis of the behavior of $P_n$ as $x \rightarrow x_0$.  A formal proof is given in Appendix A.

\begin{varalgorithm}{AutoBound1D}
  \begin{algorithmic}
  \caption{(a simplified algorithm for one-dimensional functions).}
  \label{alg:autobound}
  \STATE {\bf Hyperparameters}:
  \begin{enumerate}
  \item A table $\Sigma$, such that for any function $\sigma: \reals^m \to \reals \in \primitives$ and integer $k > 0$, the function
  \[
    F(z) = \Sigma[\sigma, k](z; Y_1, Y_2, \ldots, Y_m, Z)
  \]
  is a degree $k$ interval polynomial extension of $\sigma$ over trust regions $Y_1, Y_2, \ldots, Y_m$ and $Z$.
  \item A function $\rangebound$, such that for any interval polynomial $P$ and interval $Z$, $z \in Z \implies P(z) \subseteq \rangebound(P, Z)$. 
  \end{enumerate}
  \STATE The default value of $\Sigma$ is given in Table~\ref{tab:interval_polynomial_extensions}, and the default value of $\rangebound$ is given by \eqref{eq:default_range_bound}.
  \algrule
  \STATE {\bf Input}: a symbolic expression $(\varset, \eqlist)$, scalar $x_0 \in \reals$, interval $[a, b] \in \intervals$, and target degree $k \in \integers_{> 0}$.
  \STATE {\bf Output}: a tuple of coefficients that define a degree $k$ Taylor polynomial enclosure of $f: \reals \to \reals$ at $x_0$ over $[a, b]$, where $f$ is the function represented by the symbolic expression $(\varset, \eqlist)$.
  \algrule
  \STATE Let $\varset = (v_0, v_1, \ldots, v_n)$, and let $\eqlist = ((\sigma_1, L_1), (\sigma_2, L_2), \ldots, (\sigma_n, L_n))$.
  \STATE Set $P_0 \leftarrow (x_0, 1)$, $Y_0 \leftarrow [a, b]$, and $Z \leftarrow [a - x_0, b - x_0]$.
  \FOR {$i$ from $1$ to $n$}
    \STATE Let $j_q$ be the index of the $q$th variable in $L_i$, and let $m$ be the length of $L_i$ (so  $L_i = (v_{j_1}, v_{j_2}, \ldots, v_{j_m})$).
    \STATE Set $P_i \leftarrow \Sigma[\sigma_i, k](P_{j_1}, P_{j_2}, \ldots, P_{j_{m_i}}; Y_{j_1}, Y_{j_2}, \ldots, Y_{j_{m_i}}, Z)$.
    \STATE Set $Y_i^{(0)} \leftarrow \sigma_i(Y_{j_1}, Y_{j_2}, \ldots, Y_{j_{m_i}})$.  \COMMENT{Apply $\sigma_i$ to interval arguments using interval arithmetic.}
    \STATE Set $Y_i^{(1)} \leftarrow \rangebound(P_i, Z)$.
    \STATE Set $Y_i \leftarrow Y_i^{(0)} \cap Y_i^{(1)}$.
  \ENDFOR
  \STATE Return $P_n$.
\end{algorithmic}
\end{varalgorithm}

\subsubsection{Example}

We now trace through a run of \ref{alg:autobound}, using it to compute a quadratic Taylor polynomial enclosure of the function $f(x) = \frac{\exp(x)}{x+2}$ at $x_0 = 0$ over $[0, 2]$.  Evidently, the value of $f$ can be computed using the sequence of equations:
\begin{align}
  v_0 &= x \nonumber \\
  v_1 & = 2 + v_0 \nonumber \\
  v_2 & = \exp(v_0) \nonumber \\
  v_3 & = v_1^{-1} \nonumber \\
  v_4 & = v_2 v_3
\end{align}
which corresponds to the symbolic expression with intermediate variables $(v_0, v_1, v_2, v_3, v_4)$ and equations $( (\mathrm{plus\_two}, (v_0)), (\exp, (v_0)), (\mathrm{reciprocal}, (v_1)), (\times, (v_1, v_3)) )$.

Given $x_0 =0$ and the trust region $[-1, 1]$ as input, \ref{alg:autobound} initializes
\begin{align}
  P_0 & = (0,1) \nonumber \\
  Y_0 &= [-1, 1] \nonumber \\
  Z &= [-1, 1].
\end{align}
The value of $P_0$ represents the coefficients of the trivial Taylor polynomial enclosure $x = x_0 + 1 \cdot (x - x_0)$, while the value of $Y_0$ reflects the assumption $x \in [-1, 1]$, and the value of $Z$ reflects the assumption $x - x_0 \in [-1, 1]$.

On iteration $i = 1$, the algorithm processes the equation $(\mathrm{plus\_two}, (v_0))$.  The interval polynomial extension of the $\mathrm{plus\_two}$ function simply adds 2 to the $0$th coefficient of the polynomial, yielding
\begin{align}
  P_1 & = (2 + (P_0)_{[0]}, (P_0)_{[1]}) = (2, 1) \nonumber \\
  Y_1^{(0)} &= 2 + Y_0 = [1, 3] \nonumber \\
  Y_1^{(1)} &= \rangebound(P_1, Z) = 2 + 1 \cdot Z = [1, 3] \nonumber \\
  Y_1 & = [1, 3] \cap [1, 3] = [1, 3].
\end{align}
The value of $Y_1$ reflects that fact that $x \in [-1, 1]$ implies $2 + x \in [1, 3]$, while the value of $P_1$ reflects the fact that $x_0 = 0$ implies $2 + x = 2 + 1 \cdot (x - x_0)$.

On iteration $i = 2$, the algorithm processes the equation $(\exp, (v_0))$.  Applied to the polynomial $P_0$, the interval polynomial extension of $\exp$ over trust regions $Y_0$ and $Z$ returns $Q \circ_Z P_0$, where $Q$ is the sharp quadratic Taylor polynomial enclosure of $\exp$ at $(P_0)_{[0]} = 0$ over $Y_0 = [-1, 1]$.  Using the theory developed in \cite{streeter2023sharp}, we have $Q = (1, 1, [\frac 1 e, e - 2])$.  Because $P_0 = (0, 1)$, $Q \circ_Z P_0 = Q$, and therefore,
\begin{align}
  P_2 & = \paren { 1, 1, \left [\frac 1 e, e - 2 \right ] } \nonumber \\
  Y_2^{(0)} &= \exp(Y_0) = \left [ \frac 1 e, e \right ] \nonumber \\
  Y_2^{(1)} &= \rangebound(P_2, Z)  %
  = [0, e] \nonumber \\
  Y_2 &= \left [ \frac 1 e, e \right ]. 
\end{align}

On iteration $i = 3$, the algorithm processes the equation $(\mathrm{reciprocal}, (v_1))$.  Proceeding as in the previous iteration, we obtain
\begin{align}
  P_3 & = \paren { \frac 1 2, - \frac 1 4, \left [\frac {1} {12}, \frac{1}{4} \right ] } \nonumber \\
  Y_3 &%
  = \paren { \frac 1 3, 1 }.
\end{align}

Finally, on iteration $i = 4$, the algorithm applies the interval polynomial extension of the product function to the polynomials $P_2$ and $P_3$ to obtain
\begin{equation}
  P_4 = \paren { \frac 1 2, \frac 1 4, \left [\frac {3} {4 e} - \frac{5} {12}, \frac {3 e} {4} - \frac {1} {4 e} - \frac 5 4 \right ] }.
\end{equation}

We conclude
\begin{align}
  f(x) & \in \frac 1 2 + \frac 1 4 (x - x_0) + \left [\frac {3} {4 e} - \frac{5} {12}, \frac {3 e} {4} - \frac {1} {4 e} - \frac 5 4 \right ] (x - x_0)^2 \nonumber \\
    & \approx \frac 1 2 + \frac 1 4 (x - x_0) + \left [-0.14076, 0.69674 \right ] (x - x_0)^2. \label{eq:autobound_enclosure}
\end{align}

In contrast, as shown in \cite{walters2009}, applying the baseline interval arithmetic approach to the same function $f$ at $x_0$ over $[-1, 1]$ produces the substantially looser Taylor polynomial enclosure:
\begin{equation}
  f(x) \in \frac 1 2 + \frac 1 4 (x - x_0) + [-2.64, 4.04] (x - x_0)^2. \label{eq:baseline_enclosure}
\end{equation}

Figure~\ref{fig:exp_x_over_x_plus_2_enclosure} plots the quadratic Taylor polynomial enclosures given by \eqref{eq:autobound_enclosure} and \eqref{eq:baseline_enclosure}.  Note that, although the enclosure returned by \ref{alg:autobound} improves significantly over the baseline, it is not sharp, as is generally the case when the input is a non-trivial composite function.

\begin{figure}[h]
\begin{center}
\includegraphics[width=0.4\linewidth]{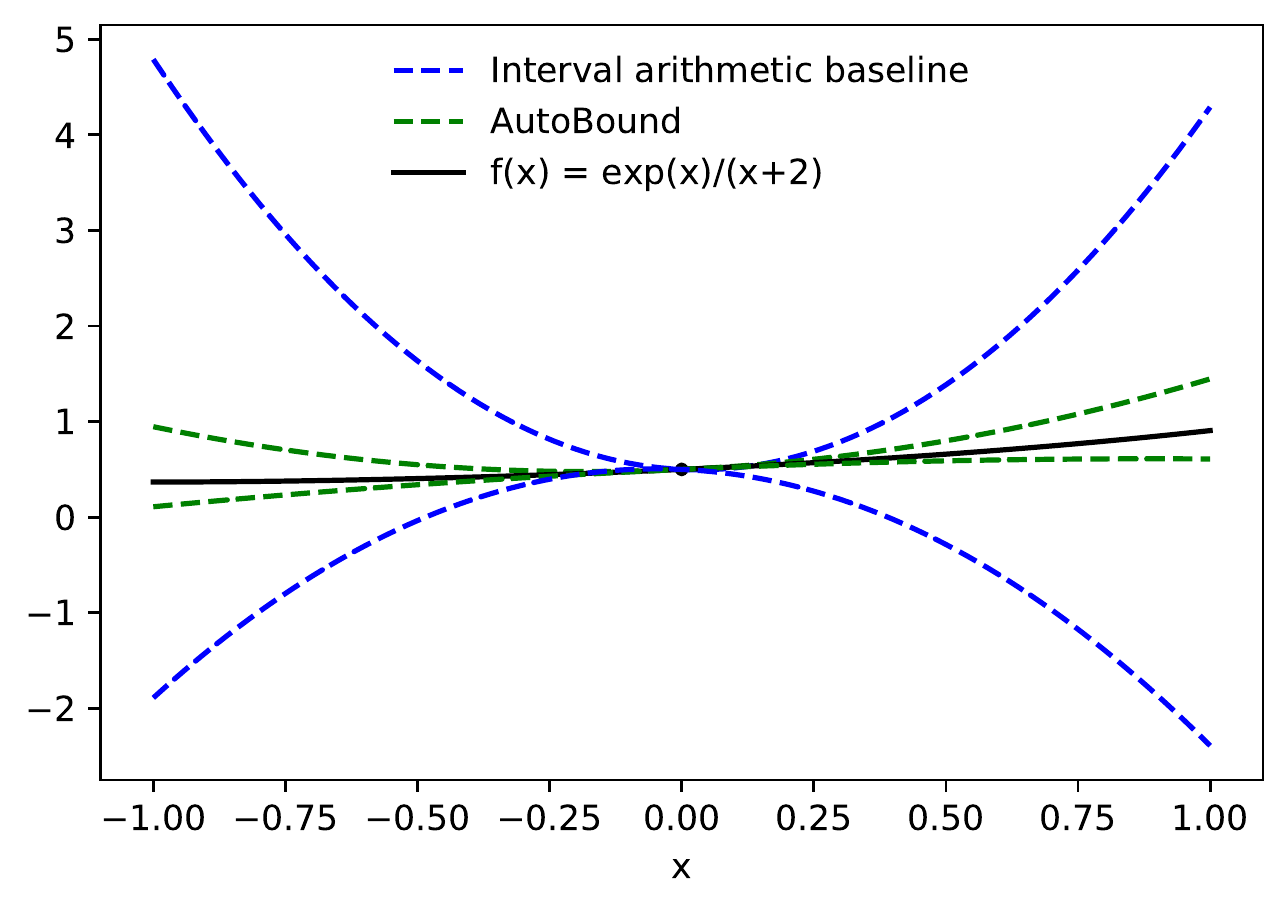}
\caption{Quadratic Taylor polynomial enclosures for $f(x) = \frac{\exp(x)}{2 + x}$ at $x_0 = 0$ over $[-1, 1]$, as derived by AutoBound and by the baseline interval arithmetic method \cite{walters2009}.}
\label{fig:exp_x_over_x_plus_2_enclosure}
\end{center}
\end{figure}

\subsection{Tightness of Automatically-Derived Bounds}

How tight are the Taylor polynomial enclosures returned by \ref{alg:autobound}?
We have not yet developed theory that answers this question in full generality.
However, in the special case where $f$ is a polynomial, the following theorem shows that \ref{alg:autobound} improves on a classical interval arithmetic baseline \cite{jaulin2001interval,hansen1979global,hansen2003global,moore1966interval} by a factor of at least $k+1$.

\newcommand{\thmabadvantage}{
Let $f(x) = \sum_{i=0}^n \alpha_i x^i$, and let $k$ be a positive integer such that for some $l > 0$, $\alpha_{k+l} \neq 0$.
Let $\ell$ be the smallest positive integer such that $\alpha^{k+\ell} \neq 0$.  Then,
\[
  \lim_{\epsilon \to 0} = \frac {\width(\iia(f, k, [-\epsilon, \epsilon]))} {\width(\iab(f, k, [-\epsilon, \epsilon], 0))} = {k + \ell \choose \ell}
\]
where $\iia(f, k, [a, b])$ is the interval obtained by evaluating $\frac {1} {k!} f^{(k)}$ over the interval $[a, b]$ according to the rules of interval arithmetic, and $\iab(f, k, [a, b], 0)$ is the degree $k$ coefficient of the polynomial returned by $\ref{alg:autobound}$ given the same arguments.
}
\newcommand{\thmabadvantageproof}{
By the rules of interval arithmetic,
\begin{align}
  & \iia(f, k, [-\epsilon, \epsilon]) \nonumber \\
  & \quad \quad = \sum_{i=0}^n \alpha_i \iia(x \mapsto x^i, k, [-\epsilon, \epsilon]) \nonumber \\
  & \quad \quad = \alpha_k + \sum_{i=k+1}^n \alpha_i \iia(x \mapsto x^i, k, [-\epsilon, \epsilon]) \nonumber \\
  & \quad \quad = \alpha_k + \sum_{j=1}^{n-k} \alpha_{k+j} \iia(x \mapsto x^{k+j}, k, [-\epsilon, \epsilon]) \nonumber \\
  & \quad \quad = \alpha_k + \sum_{j=1}^{n-k} \alpha_{k+j} {k+j \choose j} [-\epsilon, \epsilon]^j  & \mbox{by Lemma~\ref{lem:power_intervals}.} \label{eq:ia_poly_interval}
\end{align}
For any intervals $I, J \in \intervals$,
\begin{equation}
  \width(I+J) = \width(I) + \width(J).
\end{equation}
Thus, taking the width of both sides of \eqref{eq:ia_poly_interval}, we have
\begin{equation} \label{eq:ia_poly_width}
  \width(\iia(f, k, [-\epsilon, \epsilon])) = \sum_{j=1}^{n-k} \alpha_{k+j} {k+j \choose j} \width([-\epsilon, \epsilon]^j).
\end{equation}
By a similar argument,
\begin{equation} \label{eq:ab_poly_width}
  \width(\iab(f, k, [-\epsilon, \epsilon], 0)) = \sum_{j=1}^{n-k} \alpha_{k+j} \width([-\epsilon, \epsilon]^j).
\end{equation}
Dividing \eqref{eq:ia_poly_width} by \eqref{eq:ab_poly_width}, and taking the limit as $\epsilon \rightarrow 0$ completes the proof.
}
\begin{theorem} \label{thm:ab_advantage}
\thmabadvantage
\end{theorem}

\subsection{Limitations}

The most obvious limitation of \ref{alg:autobound} is that it applies only to univariate functions.  We devote Chapter~\ref{chap:bprop} to generalizing the algorithm to multivariate functions.

A more subtle limitation of the algorithm is that, although it can be applied to univariate functions that make use of bilinear operations (such as matrix multiplications or convolutions), each scalar product used to compute the output of the bilinear operation must have a dedicated equation in the symbolic expression, which is impractical  when using software frameworks such as TensorFlow \cite{abadi2016tensorflow}, PyTorch \cite{paszke2019pytorch}, or JAX \cite{bradbury2018jax}.  This limitation is removed as part of the generalization to multivariate functions in Chapter~\ref{chap:bprop}.

Finally, \ref{alg:autobound} cannot compute a Taylor polynomial enclosure of functions such as $f(x) = x^x$, where the independent variable appears in both a base and an exponent.  Fortunately, such functions do not typically arise in the applications of interest.  (In contrast, functions such as $f(x) = x^p$ for fixed $p \in \reals$ are supported, because $x \mapsto x^p$ can be treated as an atomic function.)

\chapter{Generalization to Multivariate Functions} \label{chap:bprop}

In this chapter, we generalize the \ref{alg:autobound} algorithm to handle multivariate functions.  To accomplish this, there are two distinct problems we must solve:
\begin{enumerate}
  \item We must make the algorithm work with multivariate interval polynomials, rather than scalar interval polynomials.
  \item In order to make efficient use of modern machine learning frameworks, we must enure that the output of our algorithm can be computed using a small number of widely-supported operations on tensors.
\end{enumerate}

Toward these ends, we introduce the notion of a \emph{tensor interval}: an interval whose end points are tensors rather than scalars.  When then define \emph{tensor interval polynomials}, which generalize interval polynomials by allowing the coefficients of the polynomial to be tensor intervals rather than scalar intervals.

With these concepts in place, elementwise operations on tensor interval polynomials, such as addition, subtraction, and (elementwise) multiplication, use simple rules analogous to the rules given in Chapter~\ref{chap:arithmetic} for scalar interval polynomials.
We also derive a rule for composing two tensor interval polynomials, which generalizes the rule for scalar interval polynomials.

For bilinear operations, such as matrix multiplications and convolutions, the situation is more subtle.  Because any bilinear operation can be written as a large number of scalar additions and multiplications, we could in principle handle bilinear operations using the existing addition and multiplication rules.  However, doing so would produce an impractically large (and inefficient) computation graph, with a vertex for each scalar multiplication (and potentially millions of vertices overall), which would prevent us from being able to make effective use of machine learning compilers (such as XLA) and automatic differentiation frameworks (e.g., TensorFlow, PyTorch, JAX).  To work around this, we derive rules for enclosing the result of this hypothetical computation in a (potentially wider) tensor interval that can be computed using a small number of standard operations on tensors.  These rules are analogous to, but different from, the rules for Interval Bound Propagation \cite{gowal2018effectiveness}, which provides an efficient way to enclose the results of interval arithmetic on tensors.

\section{Definitions and Notation} \label{sec:bprop_definitions}

\newcommand{\indices}{\mathrm{indices}}

As is standard in the machine learning literature, we use the term \emph{tensor} to denote a multidimensional array of real numbers (without requiring that the elements of the array define a physical tensor).  We use bold capital symbols such as $\mA$ to denote tensors.

We use $\shape(\mA)$ to denote the shape of tensor $\mA$, a tuple of non-negative integers.  Tuples can be multiplied by integers as in python, so for example if $\mA$ has shape (3, 5), then $2*\shape(\mA) = (3, 5, 3, 5)$.  We use $\indices(\mA)$ to denote the set of tuples of integers that index the elements of $\mA$.  For example, if $\shape(\mA) = (3, 5)$, then $\indices(\mA) = \set { (i, j) \in \integers^2: 0 \le i < 3, 0 \le j < 5 }$.

To define tensor polynomials, we must first define appropriate inner and outer products.  If $\mA$ is a tensor of rank $r$, and $\mB$ is a tensor of rank $q \le r$, the inner product $\inner{\mA}{\mB}$ is a tensor of rank $r - q$, whose elements are defined by
\begin{equation} \label{eq:inner}
  \inner{\mA}{\mB}_{i_1, i_2, \ldots, i_{r-q}} \eqdef \sum_{(j_1, j_2, \ldots, j_q) \in \mathrm{indices}(\mB)} \mA_{i_1, i_2, \ldots, i_{r-q}, j_1, j_2, \ldots, j_q}  \mB_{j_1, j_2, \ldots, j_q}.
\end{equation}
Observe that if $\a$ and $\b$ are vectors, $\inner{\a}{\b}$ is the usual dot product, while if $\mA$ is a matrix and $\b$ is a vector, then $\inner{\mA}{\b}$ is the usual matrix-vector product.

If $\mA$ is a tensor of rank $r$, and $\mB$ is a tensor of rank $s$, the outer product $\mA \otimes \mB$ is a tensor of rank $r + s$, whose elements are defined by
\begin{equation} \label{eq:outer}
  (\mA \otimes \mB)_{i_1, i_2, \ldots, i_r, j_1, j_2, \ldots j_s} = \mA_{i_1, i_2, \ldots, i_r} \mB_{j_1, j_2, \ldots, j_s}.
\end{equation}
For an integer $k \ge 0$, we use $\mA^{\otimes k}$ to denote a repeated outer product: for $k > 0$, $\mA^{\otimes k} \eqdef \mA \otimes \mA^{\otimes k-1}$, while $\mA^{\otimes 0} \eqdef 1$.  Observe that if $\mA$ has rank $r$, then $\mA^{\otimes k}$ has rank $k r$.

With this notation in hand, we can now define a \emph{tensor polynomial}.
\begin{definition} [Tensor polynomial]
A degree $k$ tensor polynomial is a function $\mA$ from tensors of some input shape $I$ to tensors of some output shape $O$, defined by coefficients $\mA_{[0]}, \mA_{[1]}, \ldots, \mA_{[k]}$, where $\mA_{[j]}$ has shape $O + j * I$.  For a tensor $\mZ$ of shape $I$, the value of $\mA(\mZ)$ is given by
\begin{equation} \label{eq:tensorpoly}
  \mA(\mZ) = \sum_{j=0}^k \inner {\mA_{[j]}} {\mZ^{\otimes j}}.
\end{equation}
\end{definition}

To define tensor \emph{interval} polynomials, we must generalize our inner and outer products to work with \emph{tensor intervals}, which we now define.

\begin{definition} [Tensor interval]
For tensors $\mA, \mB \in \reals^{n_1 \times n_2 \times \ldots \times n_k}$, the \emph{tensor interval} $[\mA, \mB]$ is the set of tensors $\set{\mX \in \reals^{n_1 \times n_2 \times \ldots \times n_k}: \mA \le \mX \le \mB}$, where the inequality is elementwise.
\end{definition}

We use calligraphic symbols such as $\sA$ to denote tensor intervals.  Following the conventions for scalar intervals, we use $\lep{\sA}$ and $\rep{\sA}$ to denote the left and right endpoints of a tensor interval $\sA$.  We use $r(\sA) \eqdef \frac {\rep{\sA} - \lep{\sA}} {2}$ to denote the radius of $\sA$, and $m(\sA) \eqdef \frac {\rep{\sA} + \lep{\sA}} {2}$ to denote the midpoint.  

Indexing a tensor interval gives a scalar  interval, for example if $\sA$ has rank $r$, then $\sA_{i_1, i_2, \ldots, i_r}$ denotes the scalar interval $[\lep{\sA}_{i_1, i_2, \ldots, i_r}, \rep{\sA}_{i_1, i_2, \ldots, i_r}]$.
With this indexing convention in place, the inner product definition \eqref{eq:inner} immediately generalizes to the case where one or both arguments are tensor intervals (rather than tensors).  The outer product definition \eqref{eq:outer} generalizes similarly.  We can now define a tensor interval polynomial in terms of these generalized inner and outer products.

\begin{definition} [Tensor interval polynomial]
A \emph{tensor interval polynomial} is a function $\sA$ from tensors of some input shape $I$ to tensors of some output shape $O$, defined by coefficients $\sA_{[0]}, \sA_{[1]}, \ldots, \sA_{[k]}$, where $\sA_{[j]}$ is a tensor interval of shape $O + j * I$.  For a tensor $\mZ$ of shape $I$, the value of $\sA(\mZ)$ is given by
\begin{equation} \label{eq:tensorpoly}
  \sA(\mZ) = \sum_{j=0}^k \inner {\sA_{[j]}} {\mZ^{\otimes j}}.
\end{equation}
\end{definition}

In the derivations that follow, we will also need to consider a generalized outer product that treats the first $n$ indices of $\mA$ and $\mB$ as ``batch'' indices.  If $\mA$ is a tensor of rank $r$, and $\mB$ is a tensor of rank $s$, where the first $n$ elements of $\shape(\mA)$ match the first $n$ elements of $\shape(\mB)$, then $\mA \otimes_n \mB$ denotes a tensor of rank $r + s - n$, whose elements are defined by
\begin{equation} \label{eq:generalized_outer}
  (\mA \otimes_n \mB)_{ b_1, b_2, \ldots, b_n, i_1, i_2, \ldots, i_{r-n}, j_1, j_2, \ldots, j_{s-n} } = \mA_{ b_1, b_2, \ldots, b_n, i_1, i_2, \ldots, i_{r-n} } \mB_{ b_1, b_2, \ldots, b_n, j_1, j_2, \ldots, j_{s-n}}.
\end{equation}
Observe that setting $n = 0$ recovers \eqref{eq:outer}, so that $\mA \otimes_0 \mB = \mA \otimes \mB$.
Like equation \eqref{eq:outer}, equation \eqref{eq:generalized_outer} extends naturally to the case where one or both arguments are tensor intervals.

\subsection{Operational Semantics for Tensor Intervals} \label{sec:semantics}

Tensor intervals have the same operational semantics as scalar intervals.  Formally, for an $n$-argument operation $\op$, we define
\begin{equation} \label{eq:tensor_interval_semantics}
  \op( \sA_1, \sA_2, \ldots, \sA_n) \eqdef \set { \op(\mA_1, \mA_2, \ldots, \mA_n): \mA_i \in \sA_i \ \forall i \in \set{1, 2, \ldots, n} }.
\end{equation}
The right hand side of \eqref{eq:tensor_interval_semantics} need not be expressible as a tensor interval, though in many cases it will be.  Under this definition,
the rules for adding, subtracting, multiplying, and exponentiating tensor intervals mirror the corresponding rules for interval arithmetic.  In particular,
\begin{equation}
  \sA + \sB \eqdef \set { \mA + \mB: \mA \in \sA, \mB \in \sB } = [ \lep{\sA} + \lep{\sB}, \rep{\sA} + \rep{\sB} ].
\end{equation}
Similarly,
\begin{equation}
  \sA \odot \sB = [ \min \set {  \lep{\sA} \odot \lep{\sB}, \lep{\sA} \odot \rep{\sB}, \rep{\sA} \odot \lep{\sB}, \rep{\sA} \odot \rep{\sB}  }, \max \set {  \lep{\sA} \odot \lep{\sB}, \lep{\sA} \odot \rep{\sB}, \rep{\sA} \odot \lep{\sB}, \rep{\sA} \odot \rep{\sB}  } ]
\end{equation}
where the $\min$ and $\max$ are elementwise.  The rule for computing the elementwise power $\sA^k$ is similarly an elementwise application of the usual interval arithmetic rule.

Note that the inner product of two tensor intervals was already defined by \eqref{eq:inner} in terms of scalar interval additional and multiplication, and thus does not have these semantics by definition.  However, we will see that the tensor interval inner product operation \emph{extends} the tensor inner product operation, in a sense to be made precise in the next section.

\subsection{Extended Functions}

The key to our algorithm will be to define extended versions of various atomic functions, which operate on tensor interval polynomials rather than tensors.  The first step is to define \emph{tensor interval extensions} of these functions.

\begin{definition} [Tensor interval extension]
Let $\op$ be a function that takes $n$ tensor arguments.  A \emph{tensor interval extension} of $\op$ is a function $\xop$ that takes $n$ tensor interval arguments of the same shapes as the corresponding arguments of $\op$, such that for any tensor intervals $\sA_1, \sA_2, \ldots, \sA_n$,
\[
   \op(\sA_1, \sA_2, \ldots, \sA_n) \subseteq \xop( \sA_1, \sA_2, \ldots, \sA_n)
\]
where the set on the left hand side is defined by \eqref{eq:tensor_interval_semantics}.  If this condition holds with equality, $\xop$ is said to be an \emph{exact} tensor interval extension of $\op$.
\end{definition}

The tensor interval inner product defined by \eqref{eq:inner} is a tensor interval extension of the corresponding tensor inner product.  Furthermore, when the second argument is a tensor (rather than a tensor interval), this extension is exact, as shown in the following proposition.  A formal proof is given in Appendix A.

\newcommand{\propinnerexact}{
The tensor interval inner product defined by \eqref{eq:inner} extends the tensor inner product defined by \eqref{eq:inner}.  That is, for any tensor intervals $\sA$ and $\sB$ of appropriate shapes,
\[
  \set{\inner{\mA}{\mB}: \mA \in \sA, \mB \in \sB} \subseteq \inner{\sA}{\sB}.
\]
Furthermore, when the second argument is a tensor (or a singleton tensor interval) this extension is exact: for any tensor interval $\sA$ and tensor $\mB$ of appropriate shapes,
\[
  \set{\inner{\mA}{\mB}: \mA \in \sA} = \inner{\sA}{\mB}.
\]
}
\newcommand{\propinnerexactproof}{
The first equation follows immediately from the properties of interval addition and multiplication.

We now prove the second equation.
By definition, if $\sA$ has rank $r$ and $\mB$ has rank $q \le r$,
\begin{equation}
  \inner{\sA}{\mB}_{i_1, i_2, \ldots, i_{r-q}}
  \eqdef \sum_{(j_1, j_2, \ldots, j_q) \in \mathrm{indices}(\mB)} \sA_{i_1, i_2, \ldots, i_{r-q}, j_1, j_2, \ldots, j_q}  \mB_{j_1, j_2, \ldots, j_q}
\end{equation}
Viewing the right hand side as a dot product of appropriately-defined vectors, it follows from Lemma 1 that
\[
  \inner{\sA}{\mB}_{i_1, i_2, \ldots, i_{r-q}} = \set { \inner{\mX}{\mB}: \mX \in \sA_{i_1, i_2, \ldots, i_{r-q}}}.
\]
Because the choice of $\mX \in \sA_{i_1, i_2, \ldots, i_{r-q}}$ can be made independently for each element of $\inner{\sA}{\mB}$, it follows that $\inner{\sA}{\mB} = \set{\inner{\mA}{\mB}: \mA \in \sA}$.
}
\begin{proposition} \label{prop:inner_exact}
\propinnerexact
\end{proposition}

We are now ready to define tensor interval \emph{polynomial} extensions.  Our definition will directly generalize Definition~\ref{def:ip_extension}, which applied to scalar interval polynomials.  See the discussion surrounding Definition~\ref{def:ip_extension} for an example that illustrates the definition, and the role that the different trust regions play.

\begin{definition} [Tensor interval polynomial extension] \label{def:tip_extension}
Let $\op$ be a function that takes $n$ tensor arguments.  A function $\xop$, which takes $n$ degree $k$ tensor interval polynomial arguments, is said to be a \emph{tensor interval polynomial extension} of $\op$ over trust regions $\sY_1, \sY_2, \ldots, \sY_n$ and $\sZ$ (each of which is a tensor interval) iff.\ the following condition holds: for any tensor $\mZ \in \sZ$, any tensor interval polynomials $\sP_1, \sP_2, \ldots, \sP_n$, and any tensors $\mY_1, \mY_2, \ldots, \mY_n$, with $\mY_i \in \sY_i$ and $\mY_i \in \sP_i(\mZ)$ for all $i$, we have
\[
  \op(\mY_1, \mY_2, \ldots, \mY_n) \in \xop(\sP_1, \sP_2, \ldots, \sP_n)(\mZ).
\]
\end{definition}

\section{The AutoBound Algorithm}

We are now ready to present the \ref{alg:autoboundprop} algorithm, which derives a Taylor polynomial enclosure of a multivariate function $f$ composed of known atomic functions.  This section is organized as follows:
\begin{itemize}
  \item In \S\ref{sec:tensor_interval_extensions}, we define tensor interval extensions of various atomic functions.
  \item In \S\ref{sec:tip_extensions}, we define tensor interval \emph{polynomial} extensions of these same functions.
  \item In \S\ref{sec:autoboundprop_code} we give pseudocode for the \ref{alg:autoboundprop} algorithm and provide a proof of correctness.
\end{itemize}

\subsection{Tensor Interval Extensions of Atomic Functions} \label{sec:tensor_interval_extensions}

We begin by deriving tensor interval extensions of various atomic functions.  For arithmetic operations (such as addition and multiplication) and elementwise unary functions (such as $\exp$ and $\log$) these rules mirror the usual rules for interval arithmetic.  For bilinear operations, we extend and generalize previous results on Interval Bound Propagation \cite{gowal2018effectiveness}.

\subsubsection{Binary Arithmetic Operations}

The rules for elementwise addition, multiplication, and exponentiation given in \S\ref{sec:semantics} immediately define exact tensor interval extensions of the corresponding operations.  For example, the function $F(\sA, \sB) = \sA + \sB = [\lep{\sA} + \lep{\sB}, \rep{\sA} + \rep{\sB} ]$ is an exact tensor interval extension of the function $f(\mA, \mB) = \mA + \mB$.

\subsubsection{Elementwise Functions}

A function $f$ is said to be \emph{elementwise} if it applies some univariate function $\sigma: \reals \to \reals$ to each element of a tensor independently, and returns a tensor of the same shape.  If $f$ returns a tensor of rank $r$, a tensor interval extension of $f$ is then given by
\begin{equation}
  F(\sX)_{i_1, i_2, \ldots, i_r} \eqdef \left [ \inf_{x \in \sX_{i_1, i_2, \ldots, i_r}} \set {\sigma(x)}, \sup_{x \in \sX_{i_1, i_2, \ldots, i_r}} \set {\sigma(x)} \right ].
\end{equation}
If $\sigma$ is monotonically increasing, this definition simplifies to $F(\sX) \eqdef [f(\lep{\sX}), f(\rep{\sX})]$.  Similarly, if $f$ is monotonically decreasing, it simplifies to $F(\sX) \eqdef [f(\rep{\sX}), f(\lep{\sX})]$.

\subsubsection{Bilinear Operations} \label{sec:tensor_interval_bilinear}

\newcommand{\bilinear}{\mathsf{bilinear}}
\newcommand{\bi}{\mathsf{b}}
\newcommand{\xbi}{\mathsf{B}}

We now define tensor interval extensions for arbitrary bilinear operations (e.g., matrix multiplication, convolution).  These tensor interval extensions will not be exact, and instead will make various tradeoffs between tightness and computation.

Because any bilinear operation can be written as a sequence of scalar multiplications and additions, we can in principle define a tensor interval extension using the rules for scalar interval arithmetic.  To define this extension formally, let $\bi$ be an arbitrary bilinear function.  It can be shown that, because $\bi$ is bilinear, there exists a tensor $\mW$ such that
\begin{equation}
  \bi(\mX, \mY) = \inner{\inner{\mW}{\mX}}{\mY}.
\end{equation}
Using Proposition~\ref{prop:inner_exact}, it follows that for any tensor intervals $\sX$ and $\sY$ with $\mX \in \sX$ and $\mY \in \sY$,
\begin{equation} \label{eq:naive_extension}
  \bi(\mX, \mY) \in \inner{\inner{\mW}{\sX}}{\sY} \eqdef \xbi(\sX, \sY).
\end{equation}
Thus, $\xbi$ is a tensor interval extension of $\bi$.%

Although $\xbi$ is a perfectly reasonable tensor extension of $\bi$, and $\xbi(\sX, \sY)$ can in principle be computed efficiently, $\xbi(\sX, \sY)$ is not efficiently computable in machine learning frameworks such as TensorFlow, which do not have first-class support for intervals.  To make efficient use of such frameworks, we would like to define a tensor interval extension of $\bi$ that can be computed in terms of a small number of calls to $\bi$ itself.

In the special case where either $\sA$ or $\sB$ is a singleton tensor interval (i.e., a tensor interval whose left and right endpoints are identical) previous work on Interval Bound Propagation \cite{gowal2018effectiveness} provides a tensor interval extension that requires just two calls to $\bi$.  The following proposition generalizes this result to define a tensor interval extension that requires four calls to $\bi$ in general, but only requires two calls in the special case where either $\sA$ or $\sB$ is a singleton.

\newcommand{\propibp}{
Let $\bi(\mX, \mY) = \inner{\inner{\mW}{\mX}}{\mY}$ be a bilinear operation, where $\mW \ge \zeros$ (elementwise).  Then, for any tensor intervals $\sA$ and $\sB$, and any tensors $\mA \in \sA$, $\mB \in \sB$,
\begin{align*}
  \bi(\mA, \mB) \in \bi \paren { m(\sA), m(\sB) } + [-1, 1] \paren { \bi \paren { r(\sA), |m(\sB)| } + \bi \paren { |m(\sA)|, r(\sB) } + \bi \paren { r(\sA), r(\sB) } } 
\end{align*}
where the functions $m$ and $r$ were defined in \S\ref{sec:bprop_definitions}, and return the midpoint and radius of a tensor interval, respectively.

If $\lep{\sA} = \rep{\sA}$, this can be simplified to:
\begin{align*}
  \bi(\lep{\sA}, \mB) \in \bi \paren { m(\sA), m(\sB) } + [-1, 1] \bi \paren { |m(\sA)|, r(\sB) } \quad \forall \mB \in \sB
\end{align*}
while if $\lep{\sB} = \rep{\sB}$, it can be simplified to
\begin{align*}
  \bi(\mA, \lep{\sB}) \in \bi \paren { m(\sA), m(\sB) } + [-1, 1] \bi \paren { r(\sA), |m(\sB)| } \quad \forall \mA \in \sA.
\end{align*}
}
\newcommand{\propibpproof}{
Consider two arbitrary tensors $\mA \in \sA$ and $\mB \in \sB$.  Because $\bi$ is bilinear,
\begin{align}
  \bi(\mA, \mB)
  & = \bi(\mA, m(\sB) + \mB - m(\sB)) \nonumber \\
  & = \bi(\mA, m(\sB)) + \bi(\mA, \mB - m(\sB)) \nonumber \\
  & = \bi(m(\sA), m(\sB)) + \bi(\mA - m(\sA), m(\sB)) + \bi(m(\sA), \mB - m(\sB)) + \bi(m(\sA) - m(\sB), \mB - m(\sB)) \label{eq:bilinear_dist}.
\end{align}
Because $\mW \ge 0$ (elementwise),
\begin{equation} \label {eq:ibp_upper_a}
  \bi(\mA - m(\sA), m(\sB)) \le \bi(|\mA - m(\sA)|, |m(\sB)|) \le \bi(r(\sA), |m(\sB)|)
\end{equation}
and similarly,
\begin{equation} \label {eq:ibp_upper_b}
  \bi(m(\sA), \mB - m(\sB)) \le \bi( |m(\sA)|, r(\sB) ).
\end{equation}
Combining \eqref{eq:bilinear_dist}, \eqref{eq:ibp_upper_a} and \eqref{eq:ibp_upper_b}, we have
\begin{equation}
    \bi(\mA, \mB) \le\bi(m(\sA), m(\sB)) + \bi(r(\sA), |m(\sB)|) +  \bi(|m(\sA)|, r(\sB)) + \bi(r(\sA), r(\sB)). \label{eq:ibp_upper}
\end{equation}
We can similarly prove the lower bound
\begin{equation}
  \bi(\mA, \mB) \ge \bi(m(\sA), m(\sB)) - \bi(r(\sA), |m(\sB)|) - \bi(|m(\sA)|, r(\sB)) - \bi(r(\sA), r(\sB). \label{eq:ibp_lower}
\end{equation}
Combining \eqref{eq:ibp_upper} and \eqref{eq:ibp_lower} proves the first inequality stated in the proposition.  The simplification in the special case $\lep{\sA} = \rep{\sA}$ follows from noting that this implies $r(\sA) = \zeros$, and therefore $\bi(r(\sA), |m(\sB)|) = \zeros$.  The simplification in the special case $\lep{\sB} = \rep{\sB}$ follows similarly.
}
\begin{proposition} \label{prop:ibp}
\propibp
\end{proposition}

We also provide an alternative tensor interval extension that produces a tighter interval at the cost of additional computation.  This result is based on the following lemma.

\newcommand{\lemlinearmultrule}{
For intervals $[\lep{x}, \rep{x}]$ and $[\lep{y}, \rep{y}]$,
\[
  [\lep{x}, \rep{x}] \cdot [\lep{y}, \rep{y}] \subseteq [\lep{x}^+ \lep{y}^+ + \rep{x}^+ \lep{y}^- + \lep{x}^- \rep{y}^+ + \rep{x}^- \rep{y}^-  , \rep{x}^+ \rep{y}^+ + \lep{x}^+ \rep{y}^- + \rep{x}^- \lep{y}^+ + \lep{x}^- \lep{y}^-]
\]
where for $z \in \reals$, we define $z^+ \eqdef \max \set {z, 0}$ and $z^- \eqdef \min \set {z, 0}$.
}
\newcommand{\lemlinearmultruleproof}{
First note that for any $z \in \reals$, $z = z^+ + z^-$.  Therefore,
\begin{align*}
  [\lep{x}, \rep{x}] \cdot [\lep{y}, \rep{y}]
  & = ([\lep{x}^+, \rep{x}^+] + [\lep{x}^-, \rep{x}^-]) \cdot ([\lep{y}^+, \rep{y}^+] + [\lep{y}^-, \rep{y}^-]) \\
  & \subseteq [\lep{x}^+, \rep{x}^+] \cdot [\lep{y}^+, \rep{y}^+] + [\lep{x}^+, \rep{x}^+] \cdot [\lep{y}^-, \rep{y}^-] + [\lep{x}^-, \rep{x}^-] \cdot [\lep{y}^+, \rep{y}^+] + [\lep{x}^-, \rep{x}^-] \cdot [\lep{y}^-, \rep{y}^-] \\
  & = [\lep{x}^+ \lep{y}^+, \rep{x}^+ \rep{y}^+] + [\rep{x}^+ \lep{y}^-, \lep{x}^+ \rep{y}^-] + [\lep{x}^- \rep{y}^+, \rep{x}^- \lep{y}^+ ] + [\rep{x}^- \rep{y}^-, \lep{x}^- \lep{y}^-] \\
  & = [\lep{x}^+ \lep{y}^+ + \rep{x}^+ \lep{y}^- + \lep{x}^- \rep{y}^+ + \rep{x}^- \rep{y}^-  , \rep{x}^+ \rep{y}^+ + \lep{x}^+ \rep{y}^- + \rep{x}^- \lep{y}^+ + \lep{x}^- \lep{y}^-  ].
\end{align*}
}
\begin{lemma} \label{lem:linear_mult_rule}
\lemlinearmultrule
\end{lemma}

To understand the uses and limitations of Lemma~\ref{lem:linear_mult_rule}, it is useful to compare it to the product rule for interval arithmetic, namely
\begin{equation} \label{eq:product_rule}
  [\lep{x}, \rep{x}] \cdot [\lep{y}, \rep{y}] = [ \min \set { \lep{x} \lep{y}, \lep{x} \rep{y}, \rep{x} \lep{y}, \rep{x} \rep{y} }, \max \set { \lep{x} \lep{y}, \lep{x} \rep{y}, \rep{x} \lep{y}, \rep{x} \rep{y} } ].
\end{equation}
It can be shown that if $0 \notin [\lep{x}, \rep{x}]$, or if $0 \notin [\lep{y}, \rep{y}]$, then the interval given by Lemma~\ref{lem:linear_mult_rule} coincides with the interval given by the product rule.  But if $0 \in [\lep{x}, \rep{x}]$ and $0 \in [\lep{y}, \rep{y}]$, the interval given by Lemma~\ref{lem:linear_mult_rule} can be looser.  For example, if $[\lep{x}, \rep{x}] = [-2, 3]$ and $[\lep{y}, \rep{y}] = [-5, 7]$, then the product rule gives $[\lep{x}, \rep{x}] \cdot [\lep{y}, \rep{y}] = [-15, 21]$, but Lemma~\ref{lem:linear_mult_rule} gives $[\lep{x}, \rep{x}] \cdot [\lep{y}, \rep{y}] \subseteq [ 0 + 3 \cdot -5 + -2 \cdot 7 + 0, 3 \cdot 7 + 0 + 0 + -2 \cdot -5 ] = [-29, 31]$.  

The virtue of Lemma~\ref{lem:linear_mult_rule} is that, unlike the product rule \eqref{eq:product_rule}, it gives an interval that is \emph{linear} as a function of $\lep{x}^-$, $\rep{x}^+$, $\lep{y}^-$, and $\rep{y}^+$.  This linearity allows us to prove the following theorem, which defines an alternative tensor interval extension of an arbitrary bilinear operation.

\newcommand{\thmapplybilineartensorinterval}{
Let $\bi$ be a bilinear function.  For tensor intervals $\sA, \sB$, and tensors $\mA \in \sA$, $\mB \in \sB$,
\begin{align*}
  \bi(\mA, \mB)
  & \in \left [\bi(\lep{\sA}^+, \lep{\sB}^+) + \bi(\rep{\sA}^+, \lep{\sB}^-) + \bi(\lep{\sA}^-, \rep{\sB}^+) + \bi(\rep{\sA}^-, \rep{\sB}^-) , \right . \\
  & \quad \quad \left . \bi(\rep{\sA}^+, \rep{\sB}^+) + \bi(\lep{\sA}^+, \rep{\sB}^-) + \bi(\rep{\sA}^-, \lep{\sB}^+) + \bi(\lep{\sA}^-, \lep{\sB}^-) \right ]
\end{align*}
where for any tensor $\mZ$ we define $\mZ^+ \eqdef \max \set {\mZ, 0}$ and $\mZ^- \eqdef \min \set {\mZ, 0}$ (and the minimum and maximum are elementwise).
}
\newcommand{\thmapplybilineartensorintervalproof}{
Without loss of generality, consider a scalar-valued and vector-variate bilinear operation $\bi: \reals^n \times \reals^m \to \reals$.  Because $\bi$ is bilinear, there exists a matrix $\mW \in \reals^{n \times m}$ such that, for any $\va \in \reals^n$ and $\vb \in \reals^m$,
\[
  \bi(\va, \vb) = \sum_{i=1}^n \sum_{j=1}^m \va_i \mW_{ij} \vb_j.
\]
It then follows from the properties of interval arithmetic that, for tensor intervals $\sA \in (\reals^n)^2$ and $\sB \in (\reals^m)^2$,
\[
  \bi(\sA, \sB) \subseteq \sum_{i=1}^n \sum_{j=1}^m \sA_i \mW_{ij} \sB_j.
\]
Applying Lemma~\ref{lem:linear_mult_rule} to each term on the right hand side, then rewriting the result in terms of applications of $\bi$ to the end points of $\sA$ and $\sB$, completes the proof.
}
\begin{theorem} \label{thm:apply_bilinear_tensor_interval}
\thmapplybilineartensorinterval
\end{theorem}

\subsubsection{Summary}

The following table summarizes the tensor interval extensions of atomic functions that we have derived in this section.  The first three rows are simply element-wise versions of the corresponding rules from interval arithmetic.  The remaining rows give the tensor interval extensions of bilinear operations presented in equation \eqref{eq:naive_extension}, Proposition~\ref{prop:ibp}, and Theorem~\ref{thm:apply_bilinear_tensor_interval}.

\begin{table*}[h]
        \caption{Tensor interval extensions of atomic functions.}
        \label{tab:tensor_interval_extensions}
  \centering
        \begin{small}
        \begin{sc}
                \begin{tabular}{llll}
    \toprule
          Function & Tensor interval extension(s) \\
    \midrule
      $f(\mX, \mY) = \mX + \mY$ & $F(\sX, \sY) = [\lep{\sX} + \lep{\sY}, \rep{\sX} + \rep{\sY} ]$ \vspace{.2cm} \\ 
      $f(\mX, \mY) = \mX \odot \mY$ & $\makecell[l]{F(\sX, \sY) = [ \min \set {  \lep{\sA} \odot \lep{\sB}, \lep{\sA} \odot \rep{\sB}, \rep{\sA} \odot \lep{\sB}, \rep{\sA} \odot \rep{\sB}  },\\ \quad \quad \quad \quad \quad \max \set {  \lep{\sA} \odot \lep{\sB}, \lep{\sA} \odot \rep{\sB}, \rep{\sA} \odot \lep{\sB}, \rep{\sA} \odot \rep{\sB}  } ]}$ \vspace{.2cm} \\
      \makecell[l]{Any elementwise $f$\\($f(\mX)_{i_1, i_2, \ldots, i_r} = \sigma(\mX_{i_1, i_2, \ldots, i_r})$)} & $F(\sX)_{i_1, i_2, \ldots, i_r} \eqdef \left [ \inf_{x \in \sX_{i_1, i_2, \ldots, i_r}} \set {\sigma(x)}, \sup_{x \in \sX_{i_1, i_2, \ldots, i_r}} \set {\sigma(x)} \right ]$ \vspace{.2cm} \\

    \multirow{2}{*}{\makecell[l]{\\ Any bilinear $f$\\($f(\mX, \mY) = \inner{\inner{\mW}{\mX}}{\mY}$)} }& $F(\sX, \sY) = \inner{\sX}{\inner{\mW}{\sY}}$ \\

  & $\makecell[l]{F(\sX, \sY) = f \paren { m(\sX), m(\sY) } \\ \quad \quad \quad \quad \quad + [-1, 1] \paren { f \paren { r(\sX), |m(\sY)| } + f \paren { |m(\sX)|, r(\sY) } + f \paren { r(\sX), r(\sY) } }}$ 
  \\  %

      & $\makecell[l]{F(\sX, \sY) = \left [f(\lep{\sX}^+, \lep{\sY}^+) + f(\rep{\sX}^+, \lep{\sY}^-) + f(\lep{\sX}^-, \rep{\sY}^+) + f(\rep{\sX}^-, \rep{\sY}^-) , \right . \\
  \quad \quad \quad \quad \quad \left . f(\rep{\sX}^+, \rep{\sY}^+) + f(\lep{\sX}^+, \rep{\sY}^-) + f(\rep{\sX}^-, \lep{\sY}^+) + f(\lep{\sX}^-, \lep{\sY}^-) \right ]}$ \\
\bottomrule
  \end{tabular}
  \end{sc}
  \end{small}
\end{table*}

\subsection{Tensor Interval Polynomial Extensions of Atomic Functions} \label {sec:tip_extensions}

Having developed tensor interval extensions of various atomic functions in the previous section, we are now ready to develop tensor interval \emph{polynomial} extensions (see Definition~\ref{def:tip_extension}) of these same functions.  These extended functions will form the basis of the \ref{alg:autoboundprop} algorithm we present in the next section.

\subsubsection{Bounding the Range of a Tensor Interval Polynomial}

As in \S\ref{sec:interval_polynomial_extensions}, our extended functions will be defined in terms of a function $\rangebound$, which bounds the range of an arbitrary polynomial over a trust region.  Given a tensor interval polynomial $\sP$ and a tensor interval $\sZ$ as arguments, the $\rangebound$ function returns a tensor interval that satisfies
\begin{equation} \label{eq:range_bound_tip}
	\mZ \in \sZ \implies \sP(\mZ) \subseteq \rangebound(\sP, \sZ).
\end{equation}
As in \S\ref{sec:interval_polynomial_extensions}, one option for the $\rangebound$ function is to simply evaluate $\sP$ at $\sZ$ using interval arithmetic, defining
\begin{equation} \label{eq:tip_default_range_bound}
  \rangebound(\sP, \sZ) \eqdef \sum_{i=0}^{\mathrm{degree}(\sP)} \inner{\sP_{[i]}} {\sZ^{\otimes i} }.
\end{equation}
However, as in \S\ref{sec:interval_polynomial_extensions}, other approaches are possible that allow for different tradeoffs between tightness and computation.

\subsubsection{Addition} \label {sec:tip_addition}

We first consider the problem of adding two degree $k$ tensor interval polynomials, say $\sA(\mZ) = \sum_{i=0}^k \inner {\sA_{[i]}} { \mZ^{\otimes i} }$ and $\sB(\mZ) = \sum_{i=0}^k \inner {\sB_{[i]}} { \mZ^{\otimes i} }$, where $\sA(\mZ)$ and $\sB(\mZ)$ have the same shape.
As defined in \S\ref{sec:semantics}, tensor interval addition is associative and commutative, and therefore 
\begin{equation}
  \sA(\mZ) + \sB(\mZ) = \sum_{i=0}^k \inner {\sA_i} { \mZ^{\otimes i} } + \inner {\sB_i} { \mZ^{\otimes i} }.
\end{equation}
Furthermore, it can be shown that for any tensor intervals $\sU$, $\sV$, and tensor $\mW$ of appropriate shapes, the inner product defined by \eqref{eq:inner} satisfies\footnote{In contrast, for tensor intervals $\sU, \sV, \sW$, we have only sub-distributivity: $\inner{\sU}{\sW} + \inner{\sV}{\sW} \subseteq \inner{\sU + \sV}{\sW}$.}
\begin{align}
  \inner{\sU}{\mW} + \inner{\sV}{\mW} = \inner{\sU + \sV}{\mW}.
\end{align}
Therefore,
\begin{equation}
  \sA(\mZ) + \sB(\mZ) = \sum_{i=0}^k \inner {\sA_i + \sB_i} { \mZ^{\otimes i} }.
\end{equation}
This equation defines an exact tensor interval polynomial extension of addition (over any trust regions).

\subsubsection{Multiplication} \label {sec:tip_multiplication}

Recall from \eqref{eq:interval_product_subdistributive} that interval multiplication is sub-distributive.  Likewise, elementwise multiplication of tensor intervals is sub-distributive: for tensor intervals $\sU$, $\sV$, $\sW$,
\begin{equation}
  \sU \odot (\sV + \sW) \subseteq \sU \odot \sV + \sU \odot \sW.
\end{equation}
Therefore, for interval polynomials $\sA(\mZ) = \sum_{i=0}^k \inner {\sA_{[i]}} { \mZ^{\otimes i} }$ and $\sB(\mZ) = \sum_{i=0}^k \inner {\sB_{[i]}} { \mZ^{\otimes i} }$, which return tensors of the same shape,
\begin{align}
  \sA(\mZ) \odot \sB(\mZ)
  & = \paren { \sum_{i=0}^k \inner{\sA_{[i]}} {\mZ^{\otimes i}} } \odot \paren { \sum_{j=1}^k \inner{ \sB_{[j]} }{ \mZ^{\otimes j} } } \nonumber \\
  & \subseteq \sum_{i=0}^k \sum_{j=0}^k \inner{ \sA_{[i]} }{ \mZ^{\otimes i} } \odot \inner{ \sB_{[j]} }{ \mZ^{\otimes j} }. \label{eq:expand_product}
\end{align}

To rewrite \eqref{eq:expand_product} as an interval polynomial, we will use the following proposition, whose proof is given in Appendix A.

\newcommand{\prophadamard}{
For a tensor $\mZ$, non-negative integers $p$ and $q$, length $s$ non-negative integer tuple $S$, tensor $\mU$ of shape $S + p*\shape(\mZ)$, and tensor $\sV$ of shape $S + q*\shape(\mZ)$,
\[
  \inner{\mU}{\mZ^{\otimes p}} \odot \inner{\mV}{\mZ^{\otimes q}} = \inner{\mU \otimes_s \mV}{\mZ^{\otimes (p+q)}}.
\]
Furthermore, for tensor intervals $\sU$ and $\sV$, of the same shape as $\mU$ and $\mV$ respectively,
\[
  \inner{\sU}{\mZ^{\otimes p}} \odot \inner{\sV}{\mZ^{\otimes q}} \subseteq \inner{\sU \otimes_s \sV}{\mZ^{\otimes (p+q)}}.
\]
}
\newcommand{\prophadamardproof}{
Assume without loss of generality that $\mZ$ is a vector, and that $s = 1$, so that both $\inner{\mU}{\mZ^{\otimes p}}$ and $\inner{\mB}{\mZ^{\otimes q}}$ are vectors.  Then, for any valid index $i$,
\begin{align}
  \paren { \inner{\mU}{\mZ^{\otimes p}} \odot \inner{\mV}{\mZ^{\otimes q}} }_i
  & = \paren { \sum_{j_1, j_2, \ldots, j_p} \mU_{i, j_1,j_2, \ldots j_p} \prod_{l=1}^p \mZ_{j_l} }
      \paren { \sum_{k_1, k_2, \ldots, k_q} \mV_{i, k_1,k_2, \ldots k_q} \prod_{m=1}^q \mZ_{k_m} } \nonumber \\
  & = \sum_{j_1, j_2, \ldots, j_p, k_1, k_2, \ldots, k_q} \mU_{i, j_1,j_2, \ldots j_p} \mV_{i, k_1,k_2, \ldots k_q} \prod_{l=1}^p \prod_{m=1}^q \mZ_{j_l} \mZ_{k_m} \nonumber \\
  & = \inner{\mU \otimes_s \mV}{\mZ^{\otimes (p+q)}}_i \label{eq:first_part}
\end{align}
where the last equality follows from the definitions of inner and outer products (\eqref{eq:inner} and \eqref{eq:generalized_outer}).

For the second part of the theorem, we have:
\begin{align}
  \inner{\sU}{\mZ^{\otimes p}} \odot \inner{\sV}{\mZ^{\otimes q}}
  &= \set { \inner{\mU}{\mZ^{\otimes p}} \odot \inner{\mV}{\mZ^{\otimes q}}: \mU \in \sU, \mV \in \sV } & \mbox{by Proposition~\ref{prop:inner_exact}} \nonumber \\
  &= \set {\inner{\mU \otimes_s \mV}{\mZ^{\otimes (p+q)}} : \mU \in \sU, \mV \in \sV} & \mbox{by \eqref{eq:first_part}} \nonumber  \\
  & \subseteq \inner{\sU \otimes_s \sV}{\mZ^{\otimes (p+q)}}
\end{align}
where the last line follows because the tensor interval inner and outer product operations are tensor interval extensions of the corresponding operations on tensors.
}
\begin{proposition} \label{prop:hadamard}
\prophadamard
\end{proposition}

Applying Proposition~\ref{prop:hadamard} to each term of \eqref{eq:expand_product}, and letting $s$ denote the rank of $\sA(\mZ)$ (which by assumption equals the rank of $\sB(\mZ)$), we have
\begin{equation} \label{eq:degree_2k}
  \sA(\mZ) \odot \sB(\mZ)
  \subseteq \sum_{i=0}^k \sum_{j=0}^k \inner{ \sA_{[i]} \otimes_s \sB_{[j]} }{ \mZ^{\otimes (i + j)} }.
\end{equation}

The right hand side of \eqref{eq:degree_2k} is a degree $2 k$ tensor interval polynomial, and therefore does not define a degree $k$ tensor interval extension of elementwise multiplication.  To obtain the desired degree $k$ polynomial, we will use the following theorem.

\newcommand{\thmdegreereduction}{
For a tensor $\mZ$, non-negative integers $p$ and $q$, length $s$ non-negative integer tuple $S$, tensor $\mA$ of shape $S + p*\shape(\mZ)$, tensor $\mB$ of shape $S + q*\shape(\mZ)$, and non-negative integer $k \le p + q$,
\[
  \inner { \mA } { \mZ^{\otimes p} } \odot \inner { \mB } { \mZ^{\otimes q} } = \inner { \inner { \mA \otimes_s \mB } { \mZ^{\otimes (p + q - k)} } } { \mZ^{\otimes k} }.
\]
}
\newcommand{\thmdegreereductionproof}{
First observe that the outer product defined in \eqref{eq:outer} is associative, and therefore
\begin{equation} \label{eq:outer_associative}
\mZ^{\otimes(p + q)} = \mZ^{\otimes k} \otimes \mZ^{\otimes (p + q - k)}.
\end{equation}
The theorem then follows from Propositions \ref{prop:hadamard} and \ref{prop:inner_outer}:
\begin{align}
  \inner { \mA } { \mZ^{\otimes p} } \odot \inner { \mB } { \mZ^{\otimes q} }
  & = \inner { \mA \otimes_s \mB } { \mZ^{\otimes(p + q)} } & \mbox{by Proposition~\ref{prop:hadamard}} \nonumber \\
  & = \inner { \mA \otimes_s \mB } { \mZ^{\otimes k} \otimes \mZ^{\otimes (p + q - k)} }  & \mbox{by \eqref{eq:outer_associative}} \nonumber \\  
  & = \inner { \inner { \mA \otimes_s \mB } { \mZ^{\otimes (p + q - k)} } } { \mZ^{\otimes k} } & \mbox{by Proposition~\ref{prop:inner_outer}.}  \label {eq:thm4_first_part}
\end{align}
}
\begin{theorem} \label{thm:degree_reduction}
\thmdegreereduction
\end{theorem}

Using Theorem~\ref{thm:degree_reduction} to rewrite \eqref{eq:degree_2k}, and letting $\integers_{0:k} \eqdef \set{0, 1, \ldots, k}$, we have for $\mZ \in \sZ$:
\begin{align}
  \sA(\mZ) \odot \sB(\mZ)
  & \subseteq \paren { \sum_{\substack{i, j \in \integers_{0:k}:\\i+j < k}} \inner { \sA_{[i]} \otimes_s \sB_{[j]} }  { \mZ^{\otimes i+j } }  }  + \sum_{\substack{i, j \in \integers_{0:k}:\\i+j \ge k}} \inner{ \inner { \sA_{[i]} \otimes_s \sB_{[j]} } { \mZ^{\otimes(i + j - k)}}   } { \mZ^{\otimes (i+j) } } \nonumber \\
  & \subseteq \paren { \sum_{\substack{i, j \in \integers_{0:k}:\\i+j < k}} \inner { \sA_{[i]} \otimes_s \sB_{[j]} }  { \mZ^{\otimes i+j } }  }  + \nonumber \\
  & \quad \quad \inner{ \rangebound\paren{\sum_{\substack{i, j \in \integers_{0:k}:\\i+j \ge k}} \inner { \sA_{[i]} \otimes_s \sB_{[j]} }  { \mZ^{\otimes(i + j - k)}}, \sZ }   } { \mZ^{\otimes (i+j) } }.
\end{align}
This equation defines an (inexact) tensor interval polynomial extension of elementwise multiplication over trust regions $(\sY_1, \sY_2)$ and $\sZ$ (where the trust regions $\sY_1$ and $\sY_2$ play no role).

\subsubsection{Exponentiation with a Non-Negative Integer Exponent} \label {sec:tip_exponentiation}

To compute $\sA(\mZ)^p$, for integer $p \ge 0$, we could simply apply the multiplication rule repeatedly.  However, we can obtain a tighter bound by expanding the polynomial, collecting terms, and using the power rule from interval arithmetic, as in \S\ref{sec:ab_exponentiation}.

\subsubsection{Elementwise Functions} \label {sec:tip_elementwise}

The rule for applying an elementwise function $\sigma$ to a tensor interval polynomial $\sA(\mZ) = \sum_{i=0}^k \inner { \sA_{[i]} } { \mZ^{\otimes i} }$ is conceptually simple: we compute a Taylor polynomial enclosure (itself a tensor interval polynomial) for $\sigma$, then compose it with $\sA(\mZ)$.

To formalize the rule, suppose we wish to define a degree $k$ tensor interval polynomial extension of $\sigma$ over trust regions $\sY$ and $\sZ$.  The first step is to compute a degree $k$ \emph{elementwise} tensor interval polynomial $\sS^\sY$, satisfying
\begin{equation} \label{eq:sigma_tif}
  \sigma(\mY) \in \sum_{i=0}^k \sS^\sY_{[i]} \odot \mY^i \quad \forall \mY \in \sY.
\end{equation}
For many elementwise functions $\sigma$ of interest (including $\exp$, $\log$, and various neural network activation functions), the theory we developed in \cite{streeter2023sharp} can be used to compute the tightest possible choice for the coefficients of $\sS^\sY$.%

It follows immediately from \eqref{eq:sigma_tif} (and the operational semantics defined in \S\ref{sec:semantics}) that
\begin{equation}
  \sigma(\sA(\mZ)) \subseteq \sum_{i=0}^k \sS^\sY_{[i]} \odot \sA(\mZ)^i.
\end{equation}
We could thus define a tensor interval polynomial extension of $\sigma$ over $\sY$ and $\sZ$ using the tensor interval polynomial extensions of elementwise multiplication and exponentiation defined in the previous section.  However, as in \S\ref{sec:interval_polynomial_extensions}, we can define a tighter extension which uses a single call to the $\rangebound$ function.  We denote this extension by $\sS^\sY \circ_{\sZ} \sA$.

\subsubsection{Bilinear Operations} \label {sec:tip_bilinear}

\newcommand{\pbi}{\mathsf{batched}}
\newcommand{\xpbi}{\mathsf{Batched}}

In \S\ref{sec:tensor_interval_bilinear}, we gave rules for applying an arbitrary bilinear operation to two tensor intervals.  Building on these results, the following theorem provides a rule for applying an arbitrary bilinear operation to a tensor interval polynomial.  For simplicity, we state the theorem for the case of a scalar-valued, vector-variate bilinear operation, however it extends readily to multi-valued operations with arguments of arbitrary shape.

\newcommand{\thmbilineartip}{
Let $\bi: \reals^n \times \reals^m \to \reals$ be a scalar-valued, vector-variate bilinear operation, and let $\mW \in \reals^{n \times m}$ be its transformation matrix, so that $\bi(\vx, \vy) = \vx^\tee \mW \vy$.  For any tensor $\mU$ whose first dimension has length $n$, and any tensor $\mV$ whose first dimension has length $m$, define
\[
  \pbi(\mU, \mV) \eqdef \sum_{p=1}^n \sum_{q=1}^m \mW_{pq} (\mU_{p} \otimes \mV_{q}).
\]
Let $\xpbi$ be a tensor interval extension of $\pbi$.

For any degree $k$ tensor interval polynomials $\sA$ and $\sB$, let the degree $k$ interval polynomial $\sQ(\sA,\sB)$ be defined by:
\[
  \sQ(\sA,\sB)_{[i]} \eqdef \begin{cases}
        \sum_{\substack{l, m \in \integers_{0:k}:l+m = i}} \xpbi\paren{\sA_{[l]}, \sB_{[m]}} & i < k  \\
        \rangebound\paren{\sum_{\substack{l, m \in \integers_{0:k}:\\l+m \ge k}} \inner { \xpbi(\sA_{[l]}, \sB_{[m]}) }  { \mZ^{\otimes(l + m - k)}}, \sZ } & i = k
      \end{cases}
\]
where $\integers_{0:k} \eqdef \set{0, 1, \ldots, k}$.

Then, for any degree $k$ tensor interval polynomials $\sA$ and $\sB$, and any tensor interval $\sZ$,
\[
  \set { \bi(\mX, \mY): \mX \in \sA(\mZ), \mY \in \sB(\mZ) } \subseteq \sum_{i=0}^k \inner {\sQ(\sA,\sB)_{[i]} } {\mZ^{\otimes i}} \quad \forall \mZ \in \sZ.
\]

Accordingly, for any tensor intervals $\sY_1$, $\sY_2$, and $\sZ$, $\sQ$ is a degree $k$ tensor interval polynomial extension of $\bi$ over trust regions $(\sY_1, \sY_2)$ and $\sZ$ (where $\sY_1$ and $\sY_2$ play no role).
}
\newcommand{\thmbilineartipproof}{
Let $\sA$ and $\sB$ be tensor interval polynomials of degree $k$.
Consider some fixed $\mZ \in \sZ$, and fixed tensors $\mX \in \sA(\mZ)$ and $\mY \in \sB(\mZ)$.  By Lemma~\ref{lem:tip_choice}, there exist tensors $\mA_0, \mA_1, \ldots, \mA_k$, with $\mA_l \in \sA_{[l]}$ for all $l$, such that
\begin{equation}
  \mX = \sum_{l=0}^k \inner { \mA_l } { \mZ^{\otimes l} }.
\end{equation}
 Similarly, there exist tensors $\mB_0, \mB_1, \ldots, \mB_k$, with $\mB_m \in \sB_{[m]}$ for all $m$, such that
 \begin{equation}
   \mY = \sum_{m=0}^k \inner { \mB_m } { \mZ^{\otimes m} }.
\end{equation}
Thus, because $\bi$ is bilinear,
\begin{equation} \label{eq:expand_bilinear}
  \bi( \mX, \mY) = \sum_{l=0}^k \sum_{m=0}^k \bi( \inner { \mA_l } { \mZ^{\otimes l} } , \inner { \mB_m } { \mZ^{\otimes m} } ).
\end{equation}
To complete the proof, we will show that for all $l$ and $m$,
\begin{equation} \label{eq:sts}
  \bi( \inner { \mA_l } { \mZ^{\otimes l} } , \inner { \mB_m } { \mZ^{\otimes m} } ) \in \inner { \inner { \xpbi(\sA_l, \sB_m) } { \mZ^{\otimes \max \set { 0, l + m - k } } } } {\mZ^{\otimes \min \set {k, l+m }} }.
\end{equation}
This suffices to prove the theorem because, combining \eqref{eq:expand_bilinear} and \eqref{eq:sts}, we have:
\begin{align}
  \bi( \mX, \mY) & \subseteq \sum_{l=0}^k \sum_{m=0}^k \inner { \inner { \xpbi(\sA_l, \sB_m) } { \mZ^{\otimes \max \set { 0, l + m - k } } } } {\mZ^{\otimes \min \set {k, l+m }} } \nonumber \\
  & \subseteq \sQ(\mZ)
\end{align}
where the second line follows from the definition of $\sQ$ (and the assumed behavior of the $\rangebound$ function).

To see that \eqref{eq:sts} holds, observe that for any tensor intervals $\sU$ and $\sV$, and tensors $\mU \in \sU$, $\mV \in \sV$,
\begin{align}
  \bi(\inner{\mU}{\mZ^{\otimes l}}, \inner{\mV}{ \mZ^{\otimes m}})
  & = \sum_{p,q} \mW_{pq} \inner{\mU}{\mZ^{\otimes l}}_p \inner{\mV}{ \mZ^{\otimes m}}_q \nonumber \\
  & = \sum_{p,q} \mW_{pq} \inner{\mU_{p}}{\mZ^{\otimes l}} \inner{\mV_{q}}{ \mZ^{\otimes m}}  & \mbox{by inner product def.} \nonumber \\
  & = \sum_{p,q} \mW_{pq} \inner{\mU_{p} \otimes \mV_{q} }{\mZ^{\otimes (l+m)}} & \mbox{by Proposition~\ref{prop:hadamard} } \nonumber \\
  & = \inner{\sum_{p,q} \mW_{pq} (\mU_{p} \otimes \mV_{q}) }{\mZ^{\otimes (l+m)}}  & \mbox{by inner product def.} \nonumber \\
  & = \inner { \pbi (\mU, \mV) } {\mZ^{\otimes (l+m)} } & \mbox {by def.\ of $\pbi$} \nonumber \\
  & \in \inner { \xpbi (\sU, \sV) } {\mZ^{\otimes (l+m)} } & \mbox {by def.\ of $\xpbi$} \nonumber \\
  & \subseteq \inner { \inner { \xpbi(\sU, \sV) } { \mZ^{\otimes \max \set { 0, l + m - k } } } } {\mZ^{\otimes \min \set {k, l+m }} } & \mbox{by Theorem~\ref{thm:degree_reduction}. }  \label{eq:rewrite_term}
\end{align}
Plugging $\mU = \mA_l$, $\sU = \sA_{[l]}$, $\mV = \mB_m$, and $\sV = \sB_{[m]}$ into \eqref{eq:rewrite_term} proves \eqref{eq:sts}, which proves the theorem.
}
\begin{theorem} \label{thm:bilinear_tip}
\thmbilineartip
\end{theorem}

Note that for given bilinear operation $\bi$, Theorem~\ref{thm:bilinear_tip} can be used to define various interval polynomial extensions of $\bi$, depending on which tensor interval extension we use for the function $\pbi$ (as defined in the theorem statement).  In particular, the three tensor interval extensions presented in \S\ref{sec:tensor_interval_bilinear} provide three possible interval polynomial extensions of $\bi$, which make different tradeoffs between computation and tightness.

\begin{table*}[h]
        \caption{Degree $k$ tensor interval polynomial extensions of atomic functions over trust regions $\sY_1, \sY_2, \ldots, \sY_n$ and $\sZ$ (where $n$ is the number of arguments).}
        \label{tab:tip_extensions}
  \centering
        \begin{small}
        \begin{sc}
                \begin{tabular}{llll}
    \toprule
          Function & Tensor interval polynomial extension \\
    \midrule
      $f(\mX, \mY) = \mX + \mY$ & $F(\sA, \sB)_{[i]} = \sA_{[i]} + \sB_{[i]}$ \vspace{.3cm} \\ 
      $f(\mX, \mY) = \mX \odot \mY$
    & \makecell[l]{$F(\sA, \sB)_{[i]} = \begin{cases}
        \sum_{\substack{l, m \in \integers_{0:k}:l+m = i}} \sA_{[i]} \otimes_s \sB_{[j]} & i < k  \\
        \rangebound\paren{\sum_{\substack{l, m \in \integers_{0:k}:\\l+m \ge k}} \inner { \sA_{[l]} \otimes_s \sB_{[m]} }  { \mZ^{\otimes(l + m - k)}}, \sZ } & i = k
      \end{cases}$ \\
      where $s$ is the common rank of $\sA(\mZ)$ and $\sB(\mZ)$.
       } \vspace{.3cm} \\

      \makecell[l]{Any elementwise $f$\\($f(\mX)_{i_1, \ldots} = \sigma(\mX_{i_1, \ldots})$)} & \makecell[l]{$F(\sA) = \sS^{\sY_1} \circ_{\sZ} \sA$, 
      where $\sS^{\sY_1}$ is an elementwise tensor interval polynomial\\enclosure of $\sigma$ over $\sY_1$ (see \S\ref{sec:tip_elementwise}).
      } \vspace{.3cm} \\

    \makecell[l]{Any bilinear $f$\\($f(\mX, \mY) = \inner{\inner{\mW}{\mX}}{\mY}$)}
    & \makecell[l]{$F(\sA, \sB)_{[i]} = \begin{cases}
        \sum_{\substack{l, m \in \integers_{0:k}:l+m = i}} \xpbi\paren{\sA_{[l]}, \sB_{[m]}} & i < k  \\
        \rangebound\paren{\sum_{\substack{l, m \in \integers_{0:k}:\\l+m \ge k}} \inner { \xpbi(\sA_{[l]}, \sB_{[m]}) }  { \mZ^{\otimes(l + m - k)}}, \sZ } & i = k
      \end{cases}$ \\
      where $\xpbi$ is a tensor interval extension of the batched version \\of $f$ (see Theorem~\ref{thm:bilinear_tip}).
      } \\
\bottomrule
  \end{tabular}
  \end{sc}
  \end{small}
\end{table*}

\subsubsection{Summary}

Table~\ref{tab:tip_extensions} summarizes the tensor interval polynomial extensions we have just derived.

\subsection{Pseudocode and Analysis} \label{sec:autoboundprop_code}

Having defined rules for applying various functions to tensor interval polynomials, we are now in a position to define the \ref{alg:autoboundprop} algorithm, which computes a Taylor polynomial enclosure of an arbitrary vector-variate and vector-valued function $f$ composed of these functions.\footnote{The algorithm can be applied to tensor-variate and tensor-valued functions by appropriate reshaping.}

At a high level, the algorithm is simple: given a symbolic expression for a function $f: \reals^d \to \reals^s$, a center point $\vx_0 \in \reals^d$, and a target degree $k$, we compute the coefficients of a degree $k$ Taylor polynomial enclosure of $f$ by \emph{evaluating} f on the identity polynomial, using the tensor interval polynomial extensions of each atomic function.  These tensor interval polynomial extensions are defined over trust regions, which are computed by evaluating $f$ using the \emph{tensor interval} extensions of the atomic functions.  Pseudocode for the \ref{alg:autoboundprop} algorithm follows.

\begin{varalgorithm}{AutoBound}
  \begin{algorithmic}
  \caption{}
  \label{alg:autoboundprop}
  \STATE {\bf Hyperparameters}:
  \begin{enumerate}
     \item A table $\extable$, such that for any function $\sigma \in \primitives$:
  \begin{itemize}
      \item $\extable[\sigma, 0]$ is a tensor interval extension of $\sigma$, and
      \item for any integer $k$, $\extable[\sigma, k]$ is a function $F$ that defines a degree $k$ interval polynomial extension of $\sigma$, over trust regions provided as parameters to $F$.
  \end{itemize} 
  \item A function $\rangebound$, such that for any tensor interval polynomial $\sP$ and interval $\sZ$, $\mZ \in \sZ \implies \sP(\mZ) \subseteq \rangebound(\sP, \sZ)$.
  \end{enumerate}
  \STATE The default value of $\Sigma$ is given in Table~\ref{tab:tip_extensions}, and the default value of $\rangebound$ is given by \eqref{eq:tip_default_range_bound}.
  \algrule
  \STATE {\bf Input}: a symbolic expression $(\varset, \eqlist)$ for a function $f: \reals^d \to \reals^s$, a center point $\vx_0 \in \reals^d$, a vector interval $[\va, \vb]$, and a target degree $k \in \integers_{> 0}$.
  \STATE {\bf Output}: the coefficients of a tensor interval polynomial $\sP_n$, such that $f(\vx) \in \sP_n(\vx - \vx_0)$ for all $\vx \in [\va, \vb]$.
  \algrule
  \STATE Let $\varset = \set{v_0, v_1, \ldots, v_n}$, and let $\eqlist = \set{(\sigma_i, L_i)}_{i=1}^n$.
  \STATE Initialize $\sP_0 \leftarrow (\vx_0, \eye_{d \times d})$, $\sY_0 \leftarrow [\va, \vb]$, and $\sZ \leftarrow [\va - \vx_0, \vb - \vx_0]$.
  \FOR {$i$ from $1$ to $n$}
    \STATE Let $j_q$ be the index of the $q$th variable in $L_i$, and let $m$ be the length of $L_i$ (so  $L_i = (v_{j_1}, v_{j_2}, \ldots, v_{j_m})$).
    \STATE Set $\sP_i \leftarrow \extable[\sigma_i, k](\sP_{j_1}, \sP_{j_2}, \ldots, \sP_{j_m}; (\sY_{j_1}, \sY_{j_2}, \ldots, \sY_{j_m}), \sZ)$.
    \STATE Set $\sY_i^{(0)} \leftarrow \extable[\sigma_i, 0](\sY_{j_1}, \sY_{j_2}, \ldots, \sY_{j_m})$.
    \STATE Set $\sY_i^{(1)} \leftarrow \rangebound(\sP_i, \sZ)$.
    \STATE Set $\sY_i \leftarrow \sY_i^{(0)} \cap \sY_i^{(1)}$.
  \ENDFOR
  \STATE Return $\sP_n$.
\end{algorithmic}
\end{varalgorithm}

By adjusting the hyperparameter $\extable$, different tradeoffs can be made between the tightness of the tensor interval polynomial enclosures that \ref{alg:autoboundprop} returns and the computation required to compute their coefficients.  For example, for bilinear operations one can use any of the tensor interval extensions defined in Table~\ref{tab:tensor_interval_extensions}.

\newcommand{\thmautoboundprop}{
Assume that the \ref{alg:autoboundprop} hyperparameter $\extable$ has the properties stated in the pseudocode.  Then, given as input a symbolic expression $(\varset, \eqlist)$ defining a function $f: \reals^d \to \reals^s$, a vector $\vx_0 \in \reals^d$, and a target degree $k$, \ref{alg:autoboundprop} returns the coefficients of a tensor interval polynomial $\sP_n$ such that
\[
  f(\vx) \in \sP_n(\vx - \vx_0) \quad \forall \vx \in [\va, \vb].
\]
}
\newcommand{\thmautoboundpropproof}{
We will show inductively that, for $i = 0, 1, \ldots, n$, the algorithm maintains the following invariants:
\begin{enumerate}
  \item $v_i(\vx) \in \sY_i$ for all $\vx \in [\va, \vb]$.
  \item $v_i(\vx) \in \sP_i(\vx - \vx_0)$ for all $\vx \in [\va, \vb]$.
\end{enumerate}
For the base case $i = 0$, we have $\sY_0 = [\va, \vb]$, and invariant (1) holds trivially.  Similarly, because $\sP_0(\vx) = \vx_0 + \inner{\eye_{d \times d}}{\vx - \vx_0} = \vx$, invariant (2)  holds trivially.

Now consider some arbitrary $i > 0$, and assume the two invariants hold for all smaller $i$.  Consider iteration $i$ of the for loop, let $\Sigma_0$ and $\Sigma_k$ be defined as in the pseudocode, and let $L_i = (v_{j_1}, v_{j_2}, \ldots, v_{j_m})$ (as in the pseudocode).  By assumption, $\Sigma_0$ is a tensor interval extension of $\sigma_i$.  By the induction hypothesis, $v_{j_l}(\vx) \in \sY_{j_l}$ for $l \in \set{1, 2, \ldots, m}$ for all $\vx \in [\va, \vb]$.  Thus, for any $\vx \in [\va, \vb]$,
\begin{equation}
  v_i(\vx) = \sigma_i(v_{j_1}(\vx), v_{j_2}(\vx), \ldots, v_{j_l}(\vx)) \in \Sigma_0(\sY_{j_1}, \sY_{j_2}, \ldots, \sY_{j_l}) \subseteq \sY_j
\end{equation}
and invariant (1) is satisfied.

Similarly, using the assumption that $\Sigma_k$ is a degree $k$ tensor interval polynomial extension of $\sigma_i$ over $(\sY_{j_1}, \sY_{j_2}, \ldots, \sY_{j_m})$ and $\sZ$, for any $\vx \in \vx_0 + \sZ = [\va, \vb]$, we have:
\begin{equation}
  v_j(\vx) = \sigma_i(v_{j_1}(\vx), v_{j_2}(\vx), \ldots, v_{j_m}(\vx)) \in \Sigma_k( \sP_{j_1}, \sP_{j_2}, \ldots, \sP_{j_m})(\vx - \vx_0) = \sP_j(\vx - \vx_0) 
\end{equation}
and invariant (2) is satisfied.

Both invariants therefore hold for all $i$.  The theorem then follows from invariant (2), taking $i = n$.

}

The following theorem shows that \ref{alg:autoboundprop} always returns a valid tensor interval polynomial enclosure, provided that the hyperparameter $\extable$ satisfies the assumptions stated in the pseudocode.  The proof is a straightforward induction, and is given in Appendix A.

\begin{theorem} \label{thm:autoboundprop}
\thmautoboundprop
\end{theorem}

\section{An Implementation in JAX}

We now briefly describe our implementation of the \ref{alg:autoboundprop} algorithm in JAX \cite{bradbury2018jax}.

Our implementation is built upon a few libraries which are not specific to JAX, but instead are designed to work with any Numpy-compatible API:
\begin{enumerate}
  \item a \emph{tensor interval arithmetic} library contains code for performing various operations on tensor intervals, according to the rules defined in \S\ref{sec:tensor_interval_extensions}, using any Numpy-like API as a back end.  The library contains support for elementary arithmetic operations (such as $+$ and $*$), as well as arbitrary bilinear operations (such as matrix multiplications or convolutions).
  \item a \emph{tensor interval polynomial arithmetic library} offers similar functionality for tensor interval polynomials, using the rules defined in \S\ref{sec:tip_extensions}.
  \item a \emph{sharp Taylor polynomial enclosure library} contains code for computing sharp Taylor polynomial enclosures for various elementwise functions (e.g., $\exp$, $\log$, $\mathrm{relu}$, $\mathrm{softplus}$), using the theory developed in \cite{streeter2023sharp}.
\end{enumerate}

Building on these libraries, we provide code that allows one to compute a Taylor polynomial enclosure for an arbitrary JAX-traceable python function.  To do so, the code first converts the python function to a JAX expression (Jaxpr), looks up the sharp Taylor polynomial enclosure of elementwise functions used in the Jaxpr, and then combines them according to the \ref{alg:autoboundprop} algorithm, making use of the libraries just described.

One important subtlety is that JAX does not have primitives for certain functions such as $\mathrm{softplus}$.  Instead, the $\mathrm{softplus}$ function is represented in terms of elementary functions such as $\exp$ and $\log$.  Thus, running our algorithm directly on the Jaxpr would not take advantage of the known sharp Taylor polynomial enclosure of $\mathrm{softplus}$, and would instead compute weaker bounds obtained via the sharp Taylor polynomial enclosures for $\exp$ and $\log$.  To work around this, our implementation has pattern-matching functionality that allows us to identify uses of the $\mathrm{softplus}$ function within a Jaxpr, and to treat them as a single atomic function with a known sharp Taylor polynomial enclosure.

Our code is available on GitHub at \url{http://github.com/google/autobound}.

\section{Future Work} \label{sec:ab_future}

In the future, we plan to extend the results presented in this chapter in several ways:
\begin{enumerate}
  \item \emph{Developing a reverse-mode algorithm.}  The \ref{alg:autoboundprop} algorithm operates in a manner analogous to forward-mode automatic differentiation, and thus is only efficient when the number of inputs to the function is small.  To enable additional applications, it would be very useful to develop a reverse-mode algorithm, which would let us efficiently compute Taylor remainder series bounds for functions with millions of inputs.  We provide a brief sketch of the ideas that make a reverse-mode algorithm possible below.
  \item \emph{Computing diagonal bounds.}  Our bounds on $R_{k-1}(\vx; f, \vx_0)$ are in terms of $(\vx - \vx_0)^{\otimes k}$.  Thus, the number of coefficients in the bound goes up exponentially with $k$.  In some applications, it would be useful to instead obtain a bound in terms of $(\vx - \vx_0)^{k}$, so that the number of coefficients in the bound is independent of $k$ (and linear in the dimension of $\vx$).
  \item \emph{Supporting other set arithmetics.}  Although we have stated our results in terms of interval arithmetic, the same approach could be used in conjunction with other set arithmetics, for example ellipsoid arithmetic.  This may yield tighter (and therefore more useful) bounds.
\end{enumerate}

To outline how a reverse-mode variant of \ref{alg:autoboundprop} would work, consider the function $f: \reals \to \reals$ defined by
  \[
   f(x) = \sigma_1(\sigma_2(\sigma_3(x))).
  \]
 \ref{alg:autoboundprop} would compute a Taylor polynomial enclosure of $f$ by first computing a Taylor polynomial enclosure (say, $P_2$) for $\sigma_2 \circ \sigma_3$, then composing a Taylor polynomial enclosure of $\sigma_1$ with $P_2$ to obtain a Taylor polynomial enclosure of $f$.  But it is also possible to work in the other direction, first computing a Taylor polynomial enclosure of $\sigma_1 \circ \sigma_2$ (as a function of $\sigma_3(x)$) and then composing it with a Taylor polynomial enclosure of $\sigma_3$ to obtain a Taylor polynomial enclosure of $f$.
 
This idea can be generalized to more complex functions, analogous to the work of Wang \cite{wang2017high}, who presented a generalization of reverse-mode automatic differentiation that computes higher-order derivatives.  Furthermore, by combining this approach with the diagonal bounds idea described in item (2), we expect to be able to compute quadratic enclosures of the form $f(\vx) \in \nabla f(\vx)^\tee (\vx - \vx_0) + I (\vx - \vx_0)^2$ using time and memory comparable to that required to compute the gradient.  Such enclosures could be very useful in the context of majorization-minimization optimization algorithms, as we discuss further in \cite{streeter2023universal}.

\chapter{Applications} \label{chap:applications}

In this chapter, we discuss four applications of the \ref{alg:autoboundprop} algorithm: automatically deriving majorization-minimization (MM) optimizers, verified global optimization, verified numerical integration, and automatically deriving sharper versions of Jensen's inequality.  The application to MM optimization is the subject of a companion paper \cite{streeter2023universal} which develops multiple MM optimizers, and uses them to perform hyperparameter-free training of deep neural networks.

\section{Automatically Deriving MM Optimization Algorithms}

The motivation for this work was to derive tighter majorizers for use in majorization-minimization (MM) optimization \cite{de2016block,lange2016mm}.  MM is a class of optimization methods that iteratively reduce a loss by minimizing a locally-tight upper bound on the loss, called a \emph{majorizer}.  On iteration $t$, an MM optimizer minimizes a majorizer that is tight at some point $x_t$ in order to obtain a point $x_{t+1}$ with lower (or in trivial cases, equal) loss, as illustrated in Figure~\ref{fig:mm}.  This process can be repeated until a local minimum is reached.

\begin{figure}[h]
\begin{center}
\includegraphics[width=0.5\linewidth]{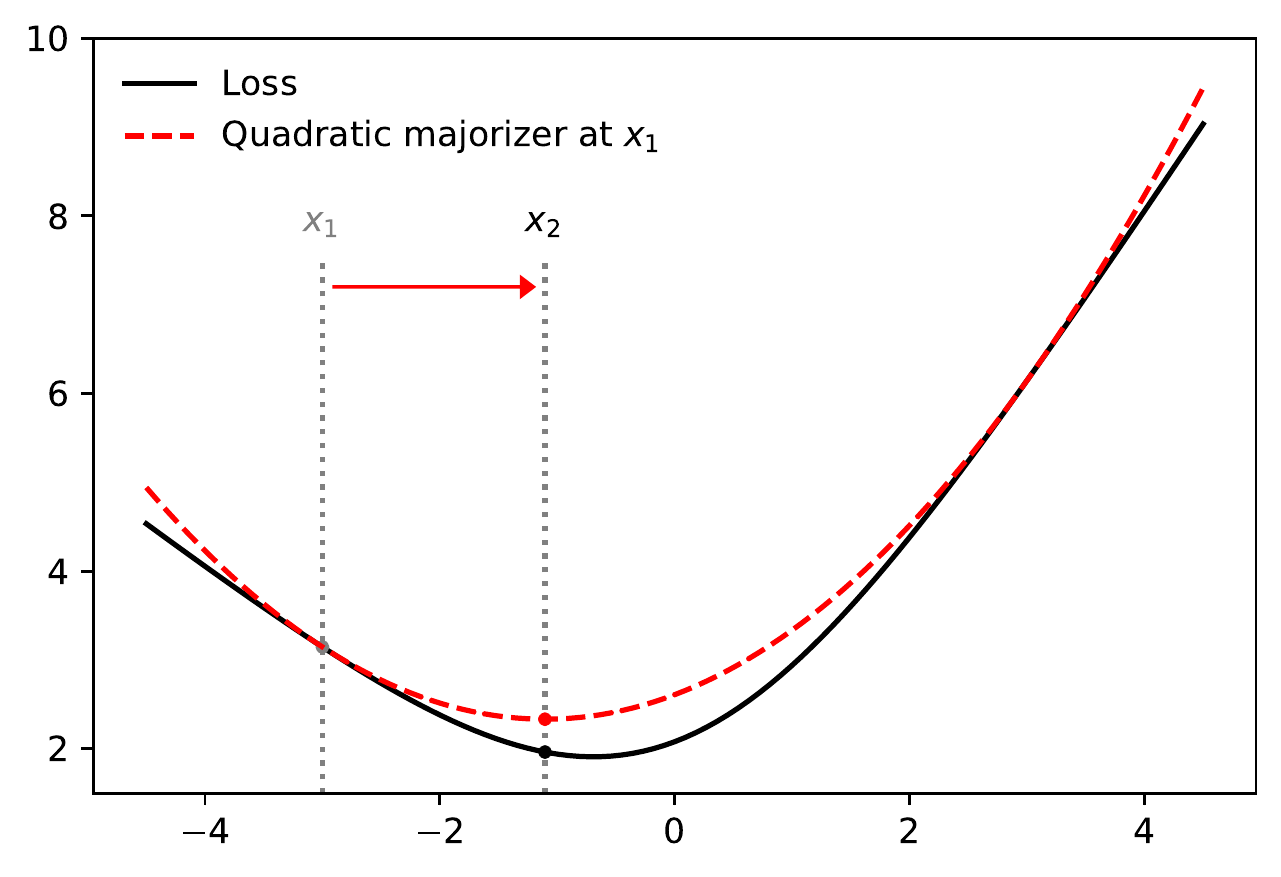}
\caption{Majorization-minimization (MM) optimization.  Starting at a point $x_1$, an MM optimizer computes an upper bound (majorizer) that is tight at $x_1$, then minimizes the upper bound to obtain a point $x_2$ with lower loss.  The process can be repeated until a local minimum is reached.}
\label{fig:mm}
\end{center}
\end{figure}

By design, the upper bound provided by a Taylor polynomial enclosure is a majorizer.  \ref{alg:autoboundprop} can therefore immediately be used to automatically derive MM optimization algorithms for arbitrary losses.  Due to the time and memory requirements of \ref{alg:autoboundprop}, a direct application of this idea is only practical for losses that depend on a small number of scalar variables.  However, the approach can still be applied to high-dimensional problems by using \ref{alg:autoboundprop} to solve a lower-dimensional subproblem, such as computing the learning rate for each step of gradient descent.
In a companion paper \cite{streeter2023universal}, we explore this idea in detail and show that it can lead to globally convergent, hyperparameter-free optimizers for training deep neural networks.

\section{Verified Global Optimization}

Global optimization is the subject of a vast literature.  A large subset of this literature assumes the loss function being minimized is deterministic and available in symbolic form (as opposed to being accessible only via queries, for example if the loss is the result of a physical experiment).

Verified global optimization (e.g., \cite{berner1996new,caprani1993use,hansen1979global,hansen2003global,kearfott1996review,ichida1979interval,makino2005verified,ratschek1988new}) seeks to find a \emph{provable} global minimum of an arbitrary (possibly non-convex) loss, provided in symbolic form.  This goal differs from  that of MM optimization, which only seeks to converge to a local minimum.

Verified global optimization algorithms are typically instances of \emph{branch and bound}, a procedure that finds a global minimum by iteratively partitioning the parameter space into disjoint subsets, and proving lower bounds on the minimum value of the loss within each subset.  A global lower bound can be obtained by taking the minimum lower bound across all subsets in the partition.  As the partition is refined to contain smaller and smaller subsets, the difference between the best known upper bound on the minimum loss (obtained by evaluating the loss at some point) and the global lower bound can be made arbitrarily small, allowing the algorithm to guarantee that it has found a point whose loss is within $\epsilon$ of the global minimum, for arbitrarily small $\epsilon > 0$.

Bounds on the Taylor remainder series of the form given by \emph{Taylor models} have previously been used for verified global optimization \cite{makino2005verified}.  The Taylor remainder series bounds provided in this work have the potential to improve these results by providing both tighter upper bounds and tighter lower bounds.

To explain how, suppose we wish to globally minimize a function $f: \reals^n \to \reals$. For any point $\vx_t$ in parameter space, and any trust region (specified by a tensor interval), we can compute an upper bound $\rep{f}(\vx, \vx_t)$ that is tight at $\vx_t$ and valid over the trust region.  We can similarly derive a function $\lep{f}(\vx_t, \vx)$ that lower bounds $f$, where the bound is tight at $\vx_t$.  The minimum values of $\rep{f}$ and $\lep{f}$ over the trust region, which can be computed in closed form in the case where they are quadratic, can then be used as the basis of a branch-and-bound algorithm.  These bounds can be tighter than the ones that would be obtained using Taylor models, and hence can let us find an approximate global minimum more efficiently.

\subsection{Experiment}

Figure~\ref{fig:branch_and_bound} depicts the results of globally minimizing the non-convex function $f(x) = 2 (x-1)^2 + (x-1)^3$ over the interval $[-2, 2]$, using a branch and bound algorithm based on quadratic upper and lower bounds returned by \ref{alg:autoboundprop}.  The plot on the left shows the loss function, as well as the points at which $f$ was evaluated by the branch and bound algorithm during the search.  The plot on the right shows the upper and lower bounds on the global minimum as a function of the number of steps (each step involves computing upper and lower bounds that are valid over some trust region that is a subset of $[-2, 2]$).  After 17 steps, the algorithm is able to find the global minimum (which in this case is simply the left end point of the interval, -2) up to machine precision.

\begin{figure}[h]
\begin{center}
\includegraphics[width=0.4\linewidth]{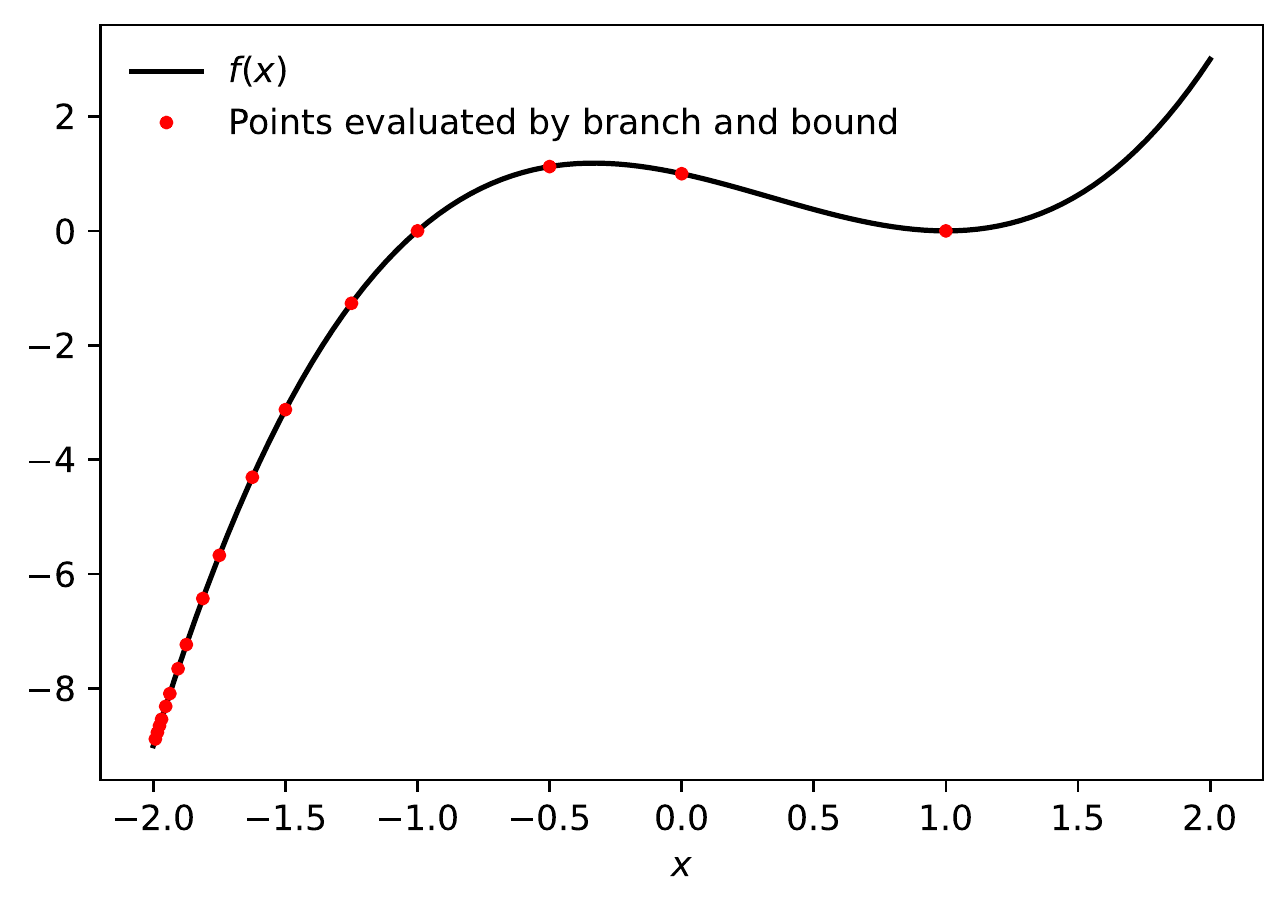}
\includegraphics[width=0.4\linewidth]{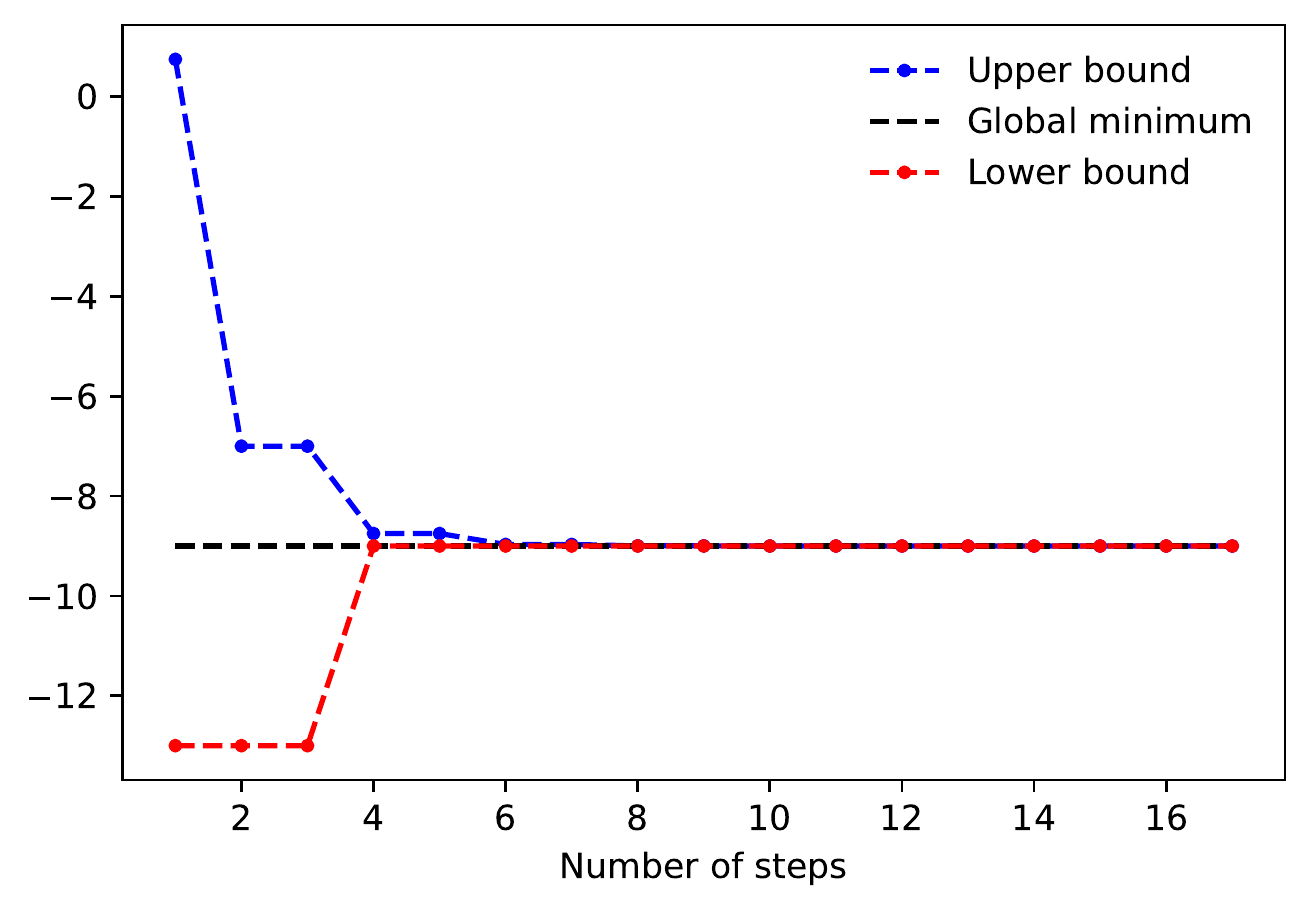}
\caption{Globally minimizing the non-convex function $f(x) = 2 (x-1)^2 + (x-1)^3$, using a branch and bound algorithm based on quadratic upper and lower bounds computed by \ref{alg:autoboundprop}.  The left plot shows the points evaluated on each step, and the right plot shows the upper and lower bounds on the minimum loss as a function of the number of steps.}
\label{fig:branch_and_bound}
\end{center}
\end{figure}

\section{Verified Numerical Integration}

We now consider the problem of verified numerical integration (e.g., \cite{berz1999new,corliss1987adaptive,eiermann1989automatic,holzmann1996newton,kelch1993numerical,lang2001derivative,petras2007principles}).  In this problem, we wish to compute upper and lower bounds on an integral,
\begin{equation}
  \int_{\sX} f(x) \dd x
\end{equation}
where $\sX$ is typically some set over which polynomials can be integrated in closed form, such as a hyperrectangle.  As was the case with verified global optimization, verified numerical integration is a problem to which Taylor models have been successfully applied \cite{berz1999new}, and which can potentially be solved more effectively using the bounds computed by \ref{alg:autoboundprop}.

Using the \ref{alg:autoboundprop} algorithm, we can obtain polynomial upper and lower bounds on $f$, which we can then integrate over $\sX$ in order to obtain upper and lower bounds on the integral.  These upper and lower bounds may be loose, but can be made tighter by partitioning $\sX$ into disjoint subsets, applying the same method to each subset, and summing the upper and lower bounds over all subsets.  The partitioning of $\sX$ into subsets can be done adaptively, by greedily subdividing the subsets where the difference between the upper and lower bounds is largest.

\subsection{Algorithm Description}

To illustrate this idea, suppose we wish to compute $\int_{x=a}^b f(x) \dd x$, for some function $f: \reals \to \reals$ provided in symbolic form.  Given a point $x_0 \in [a, b]$, we can obtain upper and lower bounds on the integral as follows.  First, we use \ref{alg:autoboundprop} to compute an interval $I$ such that
\begin{equation}
  f(x) \in T_{k-1}(x; f, x_0) + I (x - x_0)^{k}
\end{equation}
where $T_{k-1}(\cdot; f, x_0)$ is the degree $k-1$ Taylor polynomial of $f$ at $x_0$.  We then have
\begin{equation}
  \int_{x=a}^b f(x) \dd x \in \int_{x=a}^b T_{k-1}(x; f, x_0) + I (x - x_0)^{k} \dd x.
\end{equation}
The integral on the right hand side can be computed in closed form, yielding an interval whose end points provide lower and upper bounds on $\int_{x=a}^b f(x) \dd x$.

To obtain tighter upper and lower bounds, we may partition the interval $[a, b]$ into intervals $[x_i, x_{i+1}]$ for $i = 1, 2, \ldots, n$, and for each interval $[x_i, x_{i+1}]$, use \ref{alg:autoboundprop} to compute an interval $I_i$ such that $f(x) \in T_{k-1}(x; f, \bar x_i) + I_i (x - x_0)^{k}$ for $x \in [x_i, x_{i+1}]$, where $\bar x_i \eqdef \frac{x_i + x_{i+1}} {2}$ denotes the midpoint of the $i$th interval.  We then have
\begin{equation} \label{eq:integral_bounds}
  \int_{x=a}^b f(x) \dd x = \sum_{i=1}^n \int_{x=x_i}^{x_{i+1}} f(x) \dd x \in \sum_{i=1}^n \int_{x=x_i}^{x_{i+1}} T_{k-1}(x; f, \bar x_i) + I_i (x - \bar x_i)^{k} \dd x.
\end{equation}

This approach generalizes the approach of bounding the integral in each cell using interval arithmetic (e.g., see Chapter 9 of \cite{moore1966interval}), which can be recovered by computing Taylor polynomial enclosures of degree 0.
The approach can be easily generalized to integration over higher-dimensional sets.

There are many ways we might partition the interval into subintervals.  A natural approach is to adaptively subdivide the intervals where the gap between the upper and lower bounds on the integral is largest.  We leave investigation of such strategies to future work, adopting a simple non-adaptive strategy in the experiments that follow.

\subsection{Experiment}

Figure~\ref{fig:integral_enclosure} summarizes the performance of the algorithm just described, when used to bound the value of $\int_{x=0}^1 \exp(x) \dd x$, by partitioning the interval $[0, 1]$ into a uniform grid of $n$ cells.  In this simple case, the exact value of the integral can be computed in closed form, and is $e - 1$.  The figure plots the width of the computed interval (whose end points are the upper and lower bounds on the integral) as a function of $n$, as well as the distance from the midpoint of the interval to the exact value (which is $e - 1$).  Note that the latter quantity is orders of magnitude smaller than the former -- in this case, the midpoint of the lower and upper bounds provides an estimate of the integral that is much more accurate than the worst-case approximation error bound given by the interval width.

\begin{figure}[h]
\begin{center}
\includegraphics[width=0.4\linewidth]{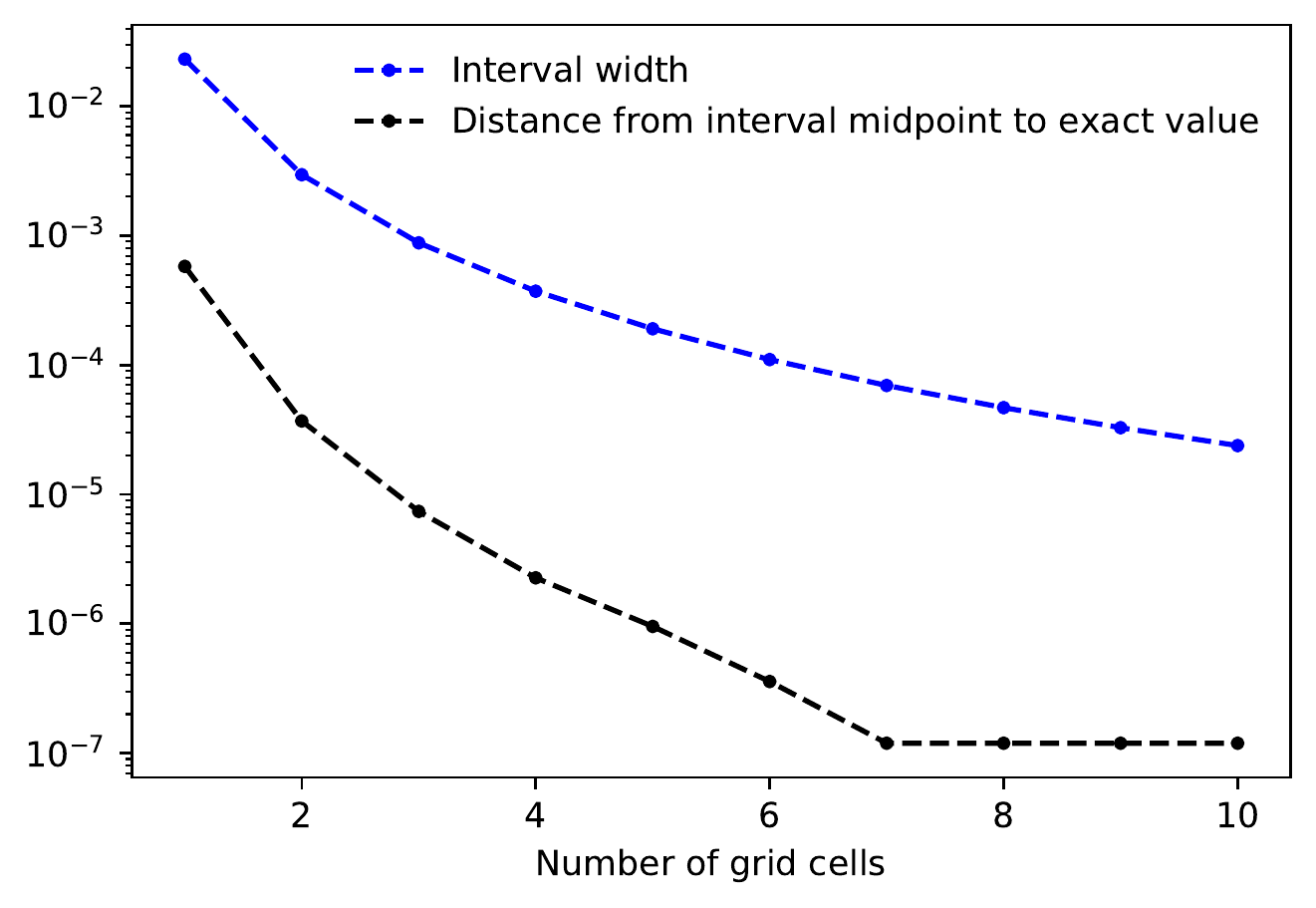}
\caption{Width of an interval that encloses the integral $\int_{x=0}^1 \exp(x) \dd x$, computed via equation \eqref{eq:integral_bounds} using \ref{alg:autoboundprop}, as a function of the number of grid cells $n$ into which the interval $[0, 1]$ is divided.}
\label{fig:integral_enclosure}
\end{center}
\end{figure}

As mentioned previously, these results can likely be improved by using an adaptively-refined grid instead of a uniform grid.

\section{Automatically Sharpening Jensen's Inequality}

\newcommand{\E}{\mathbb{E}}
\newcommand{\f}{\varphi}
\newcommand{\Var}{\mathrm{Var}}

Jensen's inequality is among the most widely used inequalities in applied mathematics.  Among other things, it is used to show non-negativity of the Kullback-Liebler divergence in information theory \cite{cover1999elements}, and in the derivation of MM optimization algorithms for various functions \cite{lange2016mm}.

Applied to a one-dimensional convex function $\f: \reals \to \reals$, Jensen's inequality states that for any scalar random variable $X$, 
\begin{equation} \label{eq:jensen}
  \E[\f(X)] - \f(\E[X]) \ge 0.
\end{equation}
The left hand side of \eqref{eq:jensen} is called the \emph{Jensen gap}.  Jensen's inequality says that if $\f$ is convex, the Jensen gap belongs to the interval $[0, \infty]$.  If all we know is that $\f$ is convex, it can be shown that this interval cannot be tightened.  However, a series of recent works \cite{gao2017bounds,lee2021further,liao2018sharpening,simic2009best,walker2014lower} have shown that by exploiting additional knowledge about $\f$, one can compute tighter bounds on the Jensen gap.

Most relevant to our work, \cite{lee2021further,liao2018sharpening} showed that tighter bounds on the Jensen gap can be obtained by computing a Taylor polynomial enclosure of $\f$ (though neither work uses this terminology).  Liao and Berg \cite{liao2018sharpening} showed how to take advantage of quadratic Taylor polynomial enclosures, and Lee \emph{et al.} \cite{lee2021further} extended their result to Taylor polynomial enclosures of arbitrary even degree.  We further extend this result to Taylor polynomial enclosures of arbitrary degree (odd or even), to obtain the following theorem.

\newcommand{\thmjensen}{
Let $\f: \reals \to \reals$ be a $(k-1)$-times differentiable function, and let $X$ be a scalar random variable that lies in the interval $[a, b]$ with probability 1.  Let the interval $I \in \intervals$ define a degree $k$ Taylor polynomial enclosure of $\f$, centered at $\mu \eqdef \E[X]$, and valid over $[a, b]$, so that we have
\[
  \f(x) \in \paren { \sum_{i=0}^{k-1} \frac {1} {i!} \f^{(k)}(\mu) (x - \mu)^i } + I \cdot (x - \mu)^{k} \quad \forall x \in [a, b].
\]
Then,
\[
  \E[\f(X)] \in \f(\mu) + \paren { \sum_{i=2}^{k-1} \frac{1}{i!} \f^{(i)}(\mu) \E [ (X - \mu)^i ] } + I \cdot \E[\min \set {0, (X - \mu)^{k} }] + I \cdot \E[\max \set {0, (X - \mu)^{k} }].
\]
If $k$ is even,\footnote{Note that the same simplification is not possible if $k$ is odd.  For example, if $k = 1$, $X$ is uniform over $[-2, 2]$, and $I = [-1, 1]$, then $I \cdot \E[\min \set {0, (X - \mu)^{k} }] + I \cdot \E[\max \set {0, (X - \mu)^{k} }] = [-1, 1] \cdot - \frac 1 2 + [-1, 1] \cdot \frac 1 2 = [-1, 1]$, whereas $I \cdot \E[ (X - \mu)^{k}] = [-1, 1] \cdot 0 = 0.$}  this simplifies to:
\[
  \E[\f(X)] \in \f(\mu) + \paren { \sum_{i=2}^{k-1} \frac {1} {i!} \f^{(i)}(\mu) \E [ (X - \mu)^i ] } + I \cdot \E[ (X - \mu)^{k} ].
\]
}
\newcommand{\thmjensennofootnote}{
Let $\f: \reals \to \reals$ be a $k$ times differentiable function, and let $X$ be a scalar random variable that lies in the interval $[a, b]$ with probability 1.  Let the interval $I \in \intervals$ define a degree $k + 1$ Taylor polynomial enclosure of $\f$, centered at $\mu \eqdef \E[X]$, and valid over $[a, b]$, so that we have
\[
  \f(x) \in \paren { \sum_{i=0}^k \frac{\f^{(k)}(\mu)}{i!} (x - \mu)^i } + I \cdot (x - \mu)^{k+1} \quad \forall x \in [a, b].
\]
Then,
\[
  \E[\f(X)] \in \f(\mu) + \paren { \sum_{i=2}^k \frac {\f^{(i)}(\mu)} {i!} \E [ (X - \mu)^i ] } + I \cdot \E[\min \set {0, (X - \mu)^{k+1} }] + I \cdot \E[\max \set {0, (X - \mu)^{k+1} }].
\]
If $k + 1$ is even, this simplifies to:
\[
  \E[\f(X)] \in \f(\mu) + \paren { \sum_{i=2}^k \frac {\f^{(i)}(\mu)} {i!} \E [ (X - \mu)^i ] } + I \cdot \E[ (X - \mu)^{k+1} ].
\]
}
\newcommand{\thmjensenproof}{
Letting $F$ be the cumulative distribution function of $X$,
\begin{align}
  \E[\f(X)]
  & = \int_a^b \f(x) \dd F(x) \nonumber \\
  & \in \int_a^b \paren { \sum_{i=0}^{k-1} \frac{\f^{(i)}(\mu)}{i!} (x - \mu)^i } + I (x - \mu)^{k} \dd F(x) \nonumber \\
  & = \paren { \sum_{i=0}^{k-1} \frac {\f^{(i)}(\mu)} {i!} \E [ (X - \mu)^i ] } + \int_a^b I (x - \mu)^{k} \dd F(x) \nonumber \\
  & = \f(\mu) + \paren { \sum_{i=2}^{k-1} \frac {\f^{(i)}(\mu)} {i!} \E [ (X - \mu)^i ] } + \int_a^b I (x - \mu)^{k} \dd F(x). \label{eq:ex_bound}
\end{align}
To complete the proof, we use the identity $z = \min \set{0, z} + \max \set{0, z}$ to write:
\begin{align}
  \int_a^b I (x - \mu)^{k} \dd F(x)
  & = \int_a^b I \cdot \paren { \min \set {0, (x - \mu)^{k} } + \max \set {0, (x - \mu)^{k} } } \dd F(x) \nonumber \\
  & = \int_a^b I \cdot \min \set {0, (x - \mu)^{k} }  \dd F(x) + \int_a^b I \cdot \max \set {0, (x - \mu)^{k} } \dd F(x) \nonumber \\
  & = I \cdot \int_a^b \min \set {0, (x - \mu)^{k} }  \dd F(x) + I \cdot \int_a^b \max \set {0, (x - \mu)^{k} } \dd F(x) \nonumber \\
  & = I \cdot \E[\min \set {0, (X - \mu)^{k} }] + I \cdot \E[\max \set {0, (X - \mu)^{k} }] \label{eq:integral_as_expectation}
\end{align}
where on the third line, the fact that the integrand has the same sign for all $x \in [a, b]$ allowed us to bring the interval $I$ outside the integral.
Plugging \eqref{eq:integral_as_expectation} into \eqref{eq:ex_bound} completes the proof.
}
\begin{theorem} [Extending Theorem 2.1 of \cite{lee2021further}] \label{thm:jensen}
\thmjensen
\end{theorem}

In the special case $k = 2$, Theorem~\ref{thm:jensen} says that if $I$ defines a quadratic Taylor polynomial enclosure of $\f$ at $\E[X]$ over $[a, b]$, then
\begin{equation}
  \E[\f(X)] - \f(\E[X]) \in I \cdot \Var[X].
\end{equation}
If $\f$ is convex, we can set $\lep{I} = 0$ to recover Jensen's inequality.  If $\lep{I} > 0$, we obtain an inequality that is tighter than Jensen's.
Furthermore, if $\f$ is an analytic function over $[a, b]$, and if $\E[ (X - \mu)^{k} ]$ is finite for all $k$, then Theorem~\ref{thm:jensen} gives upper and lower bounds on the Jensen gap that become arbitrarily tight as $k \rightarrow \infty$ \cite{lee2021further}.

One limitation of previous work is that it did not have a general-purpose mechanism for computing an interval $I$ that defines a Taylor polynomial enclosure suitable for use in Theorem~\ref{thm:jensen}, limiting the applicability of the theorem to cases where $\f$ can be analyzed by hand.  Our work provides such a mechanism, thus greatly extending the reach of Theorem~\ref{thm:jensen}.

\subsection{Experiment}

As an example, we now compute upper and lower bounds on the Jensen gap for the function $\f(x) = \exp(x)$, for $X$ drawn from a uniform distribution over $[-1, 1]$.  In this simple example, the Jensen gap can be computed exactly in closed form, and is
\begin{equation} \label{eq:jensen_gap_example}
  \E[\f(X)] - \f(\E[X]) = \frac{e - e^{-1}}{2} - 1 \approx 0.175201.
\end{equation}

Figure~\ref{fig:jensen_gap} plots the upper and lower bounds on the Jensen gap that we obtain using Theorem~\ref{thm:jensen}, computing the interval $I$ using \ref{alg:autoboundprop}, as a function of the degree of the Taylor polynomial enclosure (i.e., the value of $k+1$ in the theorem statement).  As can be seen in the figure, the upper and lower bounds rapidly converge to the true value, with the gap between the upper and lower bounds decreasing exponentially as a function of degree.  Note that even the degree 2 Taylor polynomial enclosure gives a much sharper bound than Jensen's inequality, which shows only that the Jensen gap is in $[0, \infty]$.

\begin{figure}[h]
\begin{center}
\includegraphics[width=0.4\linewidth]{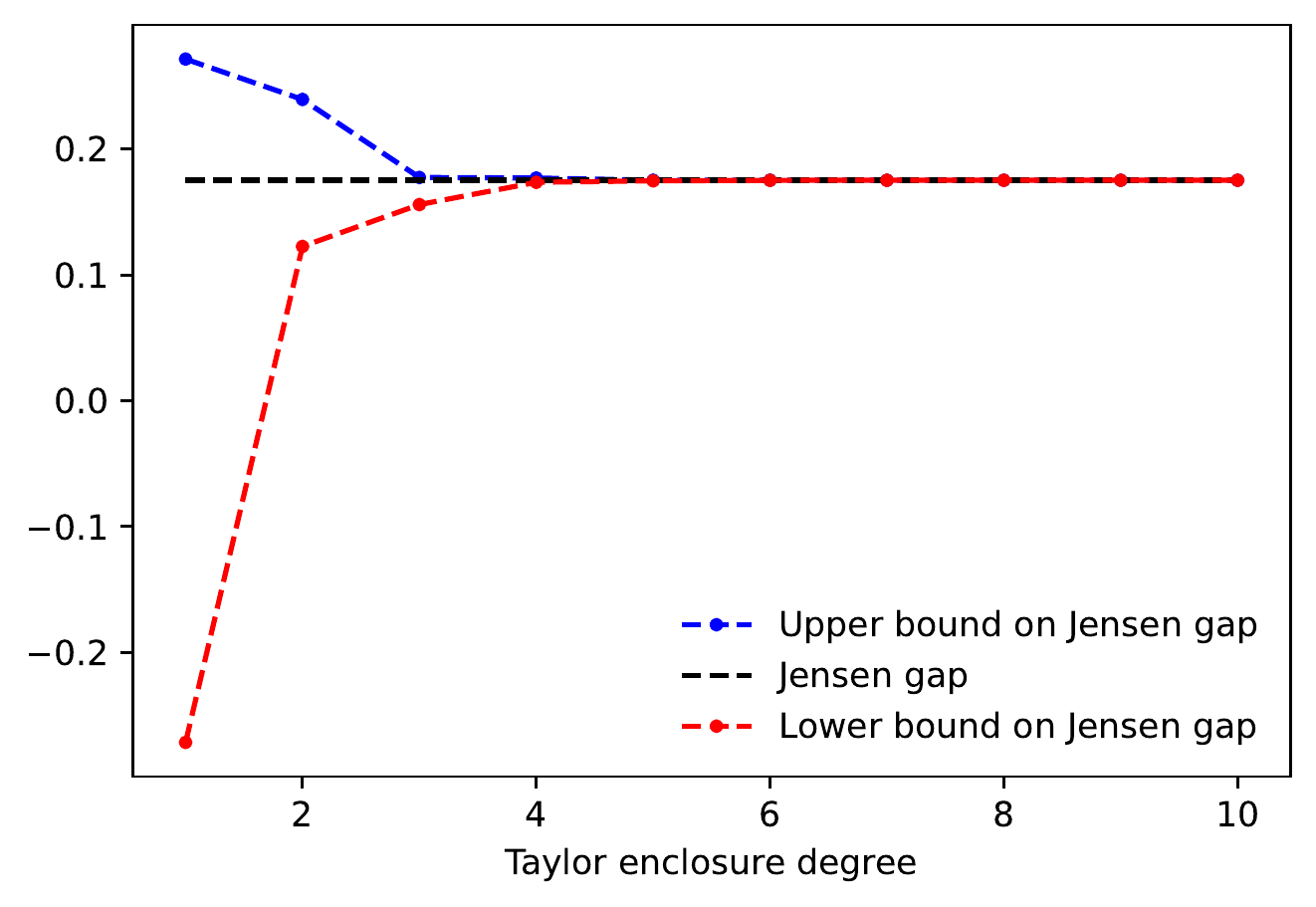}
\includegraphics[width=0.4\linewidth]{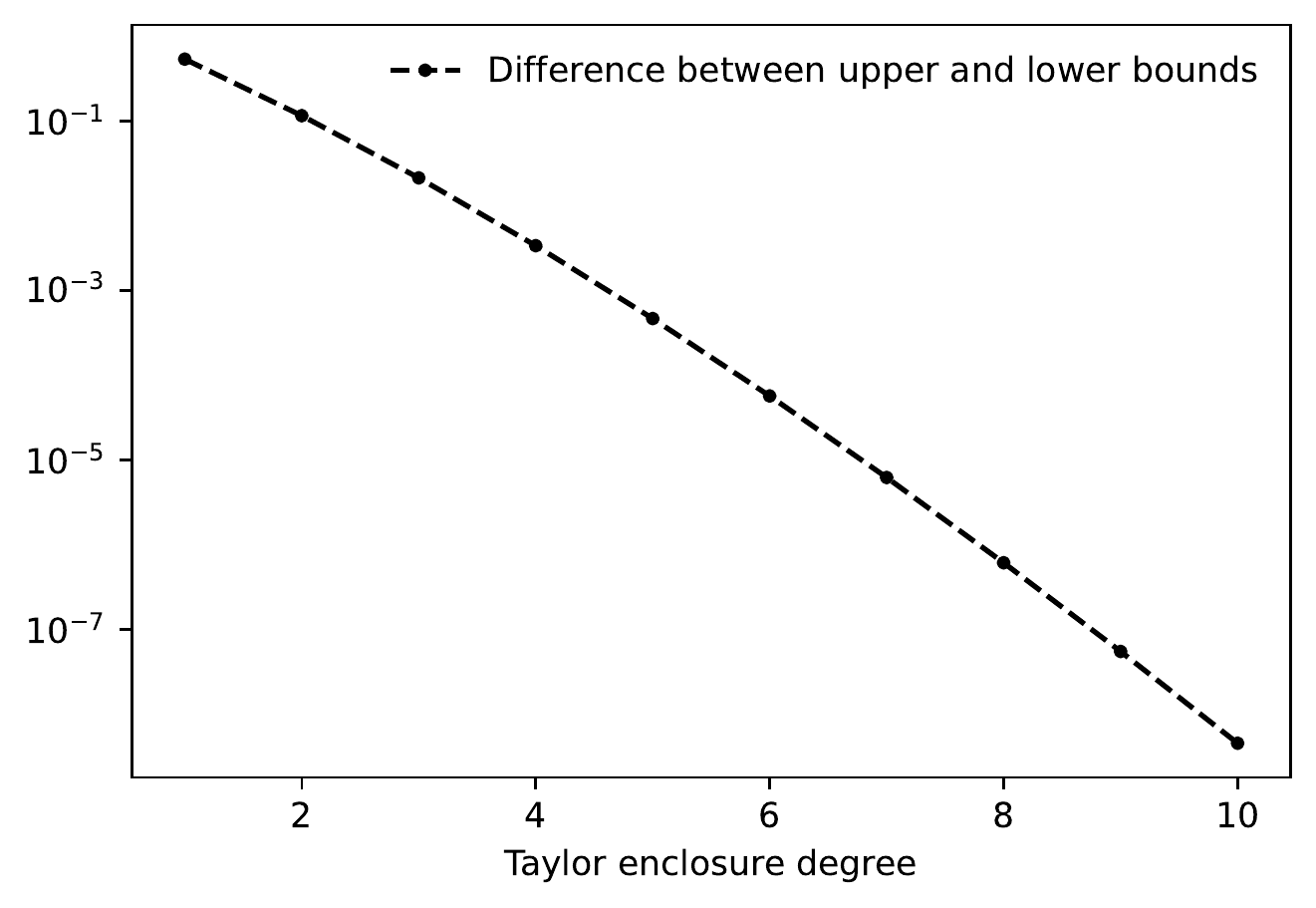}
\caption{Upper and lower bounds on the Jensen gap, $\E[\f(X)] - \f(\E[X])$, for $\f(x) = \exp(x)$ and $X \sim \mathrm{Uniform}(-1, 1)$, as a function of the degree of the Taylor polynomial enclosure used to compute the bound.  The left plot shows the upper and lower bounds, and the right plot shows the difference between them.  For degree 2 or larger, each bound provides a sharper version of Jensen's inequality, which only guarantees that the Jensen gap lies in the interval $[0, \infty]$.}
\label{fig:jensen_gap}
\end{center}
\end{figure}

\chapter{Related Work} \label{chap:related_work}

At a high level, our \ref{alg:autobound} algorithm can be seen as a variant of Taylor mode automatic differentiation that uses interval arithmetic.  A different variant of Taylor mode automatic differentiation that uses interval arithmetic has been presented in the literature on \emph{Taylor models} \cite{makino1996remainder,makino1999efficient,makino2001higher,makino2003taylor}.  A Taylor model is equivalent to a Taylor remainder series bound of the form $R_{k-1}(x; f, x_0) \in I$, where $I$ is an interval.  The algorithm for computing a Taylor model for a function $f$ can be thought of as evaluating $f$ using a set of extended operators that input and output $(P, I)$ pairs, where $P$ is a Taylor polynomial and $I$ is an interval that bounds the remainder term.  As discussed in \S\ref{sec:ab_algorithm}, \ref{alg:autobound} differs from previous work on Taylor models in three ways: (\emph{i}) by producing bounds of the form $R_{k-1}(x; f, x_0) \in I (x - x_0)^{k}$, which allows the bounds to be used as the basis of an MM optimizer (see \cite{streeter2023universal}), (\emph{ii}) by using sharp Taylor polynomial enclosures to obtain tighter bounds, and (\emph{iii}) by composing interval polynomials in a tighter way, further tightening the bounds.

A variant of Taylor models called \emph{Lagrange models} \cite{van2016certifying} produces bounds of a form similar to ours.  Applied to one-dimensional functions, Lagrange models give a bound of the form $R_{k-1}(x; f, x_0) \in [-\epsilon, \epsilon] (x - x_0)^{k}$, which is a Taylor polynomial enclosure defined by a symmetric interval.  In the one-dimensional case, the novelty of our work relative to \cite{van2016certifying} lies in the use of sharp Taylor polynomial enclosures and in the tighter composition of interval polynomials (discussed in the previous paragraph).  For multivariate functions, Taylor polynomial enclosures have many more free parameters than Lagrange models, which allows for tighter bounds.

Other noteworthy approaches that use interval polynomials to bound functions include Boundary Arithmetic \cite{lanford1982computer,eckmann1984computer,eckmann1987complete} and Ultra Arithmetic \cite{kaucher1983solving,kaucher1984residual,kaucher1988validating,kaucher1984self}.  See \cite{neumaier2003taylor} for a discussion of the history of these methods and their limitations, and see \cite{makino2003taylor} for additional discussion and a comparison to Taylor Models.

More recently, \cite{duracz2008polynomial} presented a prototype of a constraint solver that performs arithmetic operations on Chebyshev polynomials.  Like our approach, this work involves evaluating a function using polynomial extensions of various operations.  However, the use of Chebyshev polynomials makes the details of these operations different from ours, and the presentation in \cite{duracz2008polynomial} considers only a limited set of operations (addition, multiplication, exponentiation, and absolute value).

An alternative approach is to compute an interval that defines a degree $k$ Taylor polynomial enclosure by evaluating an expression for $f^{(k)}(x_0)$ using interval arithmetic \cite{jaulin2001interval,hansen1979global,hansen2003global,moore1966interval}.  Though elegant, this approach has some inherent limitations which can cause it to quickly produce enormous  intervals \cite{makino2003taylor}.  In the experiments of \cite{makino2003taylor}, these intervals were observed to be up to a factor of $10^{42}$ wider than the ones obtained using Taylor models.

Another approach to bounding the Taylor remainder series uses a variant of Cauchy's estimate to bound the Taylor coefficients, and thereby to bound the remainder term \cite{neher2003improved}.  However, this approach requires maximizing functions over the complex plane via global optimization methods (e.g., branch and bound), and is not a practical alternative to the approach we present here.

\chapter{Summary and Conclusions} \label{chap:conclusions}

In this work, we have presented a generalization of Taylor-mode automatic differentiation that, in addition to computing a Taylor polynomial approximation of a function, computes polynomial upper and lower bounds that are guaranteed to hold over a user-specified \emph{trust region}.

In Chapter~\ref{chap:arithmetic}, we considered the problem of deriving a Taylor polynomial enclosure of an arbitrary function composed of known atomic functions.  To this end, we developed an interval arithmetic variant of Taylor mode automatic differentiation which we call \ref{alg:autoboundprop}.  \ref{alg:autoboundprop} is similar in spirit to the algorithm used in previous work to compute \emph{Taylor models} \cite{makino1996remainder,makino1999efficient,makino2001higher,makino2003taylor}, but produces bounds of a different functional form that allows for additional applications, and contains several technical innovations that make the bounds tighter.

In Chapter~\ref{chap:bprop}, we extended the results of Chapter~\ref{chap:arithmetic} to tensor-valued and tensor-variate functions.  This required us to introduce additional notation, and to pay special attention to handling bilinear operations in a way that allows our algorithm to be implemented efficiently using modern automatic differentiation frameworks (such as TensorFlow, PyTorch, or JAX).  In the process of doing so, we generalized previous results for Interval Bound Propagation \cite{gowal2018effectiveness}, and developed an alternative that produces tighter bounds at the cost of additional computation.

We then turned our attention to applications of \ref{alg:autoboundprop}.  Most notably, we showed that the bounds computed by \ref{alg:autoboundprop} can be used to automatically derive majorization-minimization (MM) optimization algorithms, a possibility we explore further in a companion paper \cite{streeter2023universal}.
We also showed in Chapter~\ref{chap:applications}, via simple demonstrations on toy problems, that \ref{alg:autoboundprop} can be used for verified global optimization, verified numerical integration, and deriving sharper versions of Jensen's inequality. 

Our work suggests many possible extensions, some of which we have discussed in \S\ref{sec:ab_future}.  Perhaps the most important of these is the development of a reverse-mode algorithm, which would allow us to efficiently compute Taylor polynomial enclosures for functions of a large number of scalar variables.  In addition, we believe there are many applications of Taylor polynomial enclosures which remain to be explored.

\bibliographystyle{plainnat}
\bibliography{autobound}

\begin{thebibliography}{52}
\providecommand{\natexlab}[1]{#1}
\providecommand{\url}[1]{\texttt{#1}}
\expandafter\ifx\csname urlstyle\endcsname\relax
  \providecommand{\doi}[1]{doi: #1}\else
  \providecommand{\doi}{doi: \begingroup \urlstyle{rm}\Url}\fi

\bibitem[Abadi et~al.(2016)Abadi, Barham, Chen, Chen, Davis, Dean, Devin,
  Ghemawat, Irving, Isard, et~al.]{abadi2016tensorflow}
Mart{\'\i}n Abadi, Paul Barham, Jianmin Chen, Zhifeng Chen, Andy Davis, Jeffrey
  Dean, Matthieu Devin, Sanjay Ghemawat, Geoffrey Irving, Michael Isard, et~al.
\newblock {TensorFlow}: a system for large-scale machine learning.
\newblock In \emph{Proceedings of OSDI '16: 12th USENIX Symposium on Operating
  Systems Design and Implementation}, pages 265--283, 2016.

\bibitem[Berner(1996)]{berner1996new}
S~Berner.
\newblock New results on verified global optimization.
\newblock \emph{Computing}, 57\penalty0 (4):\penalty0 323--343, 1996.

\bibitem[Berz and Makino(1999)]{berz1999new}
Martin Berz and Kyoko Makino.
\newblock New methods for high-dimensional verified quadrature.
\newblock \emph{Reliable Computing}, 5\penalty0 (1):\penalty0 13--22, 1999.

\bibitem[Bettencourt et~al.(2019)Bettencourt, Johnson, and
  Duvenaud]{bettencourt2019taylor}
Jesse Bettencourt, Matthew~J Johnson, and David Duvenaud.
\newblock Taylor-mode automatic differentiation for higher-order derivatives in
  {JAX}.
\newblock 2019.
\newblock URL \url{https://openreview.net/pdf?id=SkxEF3FNPH}.

\bibitem[Bradbury et~al.(2018)Bradbury, Frostig, Hawkins, Johnson, Leary,
  Maclaurin, Necula, Paszke, Vander{P}las, Wanderman-{M}ilne, and
  Zhang]{bradbury2018jax}
James Bradbury, Roy Frostig, Peter Hawkins, Matthew~James Johnson, Chris Leary,
  Dougal Maclaurin, George Necula, Adam Paszke, Jake Vander{P}las, Skye
  Wanderman-{M}ilne, and Qiao Zhang.
\newblock {JAX}: composable transformations of {P}ython+{N}um{P}y programs,
  2018.
\newblock URL \url{http://github.com/google/jax}.

\bibitem[Caprani et~al.(1993)Caprani, Godthaab, and Madsen]{caprani1993use}
Ole Caprani, Brian Godthaab, and Kaj Madsen.
\newblock Use of a real-valued local minimum in parallel interval global
  optimization.
\newblock \emph{Interval Computations}, 2:\penalty0 71--82, 1993.

\bibitem[Corliss and Rall(1987)]{corliss1987adaptive}
George~F Corliss and Louis~B Rall.
\newblock Adaptive, self-validating numerical quadrature.
\newblock \emph{SIAM Journal on Scientific and Statistical Computing},
  8\penalty0 (5):\penalty0 831--847, 1987.

\bibitem[Cover(1999)]{cover1999elements}
Thomas~M Cover.
\newblock \emph{Elements of Information Theory}.
\newblock John Wiley \& Sons, 1999.

\bibitem[de~Leeuw(2016)]{de2016block}
Jan de~Leeuw.
\newblock \emph{Block Relaxation Methods in Statistics}.
\newblock 2016.
\newblock URL \url{https://bookdown.org/jandeleeuw6/bras/}.

\bibitem[Duracz and Konecn{\`y}(2008)]{duracz2008polynomial}
Jan~Andrzej Duracz and Michal Konecn{\`y}.
\newblock Polynomial function enclosures and floating point software
  verification.
\newblock \emph{Proceedings of CFV}, pages 56--67, 2008.

\bibitem[Eckmann and Wittwer(1987)]{eckmann1987complete}
Jean-Pierre Eckmann and Peter Wittwer.
\newblock A complete proof of the {F}eigenbaum conjectures.
\newblock \emph{Journal of Statistical Physics}, 46\penalty0 (3):\penalty0
  455--475, 1987.

\bibitem[Eckmann et~al.(1984)Eckmann, Koch, and Wittwer]{eckmann1984computer}
Jean~Pierre Eckmann, Hans Koch, and Peter Wittwer.
\newblock \emph{A computer-assisted proof of universality for area-preserving
  maps}.
\newblock American Mathematical Society, 1984.

\bibitem[Eiermann(1989)]{eiermann1989automatic}
Martin~C Eiermann.
\newblock Automatic, guaranteed integration of analytic functions.
\newblock \emph{BIT Numerical Mathematics}, 29\penalty0 (2):\penalty0 270--282,
  1989.

\bibitem[Gao et~al.(2017)Gao, Sitharam, and Roitberg]{gao2017bounds}
Xiang Gao, Meera Sitharam, and Adrian~E Roitberg.
\newblock Bounds on the {J}ensen gap, and implications for mean-concentrated
  distributions.
\newblock \emph{arXiv preprint arXiv:1712.05267}, 2017.

\bibitem[Gowal et~al.(2018)Gowal, Dvijotham, Stanforth, Bunel, Qin, Uesato,
  Arandjelovic, Mann, and Kohli]{gowal2018effectiveness}
Sven Gowal, Krishnamurthy Dvijotham, Robert Stanforth, Rudy Bunel, Chongli Qin,
  Jonathan Uesato, Relja Arandjelovic, Timothy Mann, and Pushmeet Kohli.
\newblock On the effectiveness of interval bound propagation for training
  verifiably robust models.
\newblock In \emph{International Conference on Computer Vision}, 2018.

\bibitem[Griewank and Walther(2008)]{griewank2008evaluating}
Andreas Griewank and Andrea Walther.
\newblock \emph{Evaluating Derivatives: Principles and Techniques of
  Algorithmic Differentiation}.
\newblock SIAM, 2008.

\bibitem[Hansen and Walster(2003)]{hansen2003global}
Eldon Hansen and G~William Walster.
\newblock \emph{Global Optimization Using Interval Analysis: Revised and
  Expanded}, volume 264.
\newblock CRC Press, 2003.

\bibitem[Hansen(1979)]{hansen1979global}
Eldon~R Hansen.
\newblock Global optimization using interval analysis: the one-dimensional
  case.
\newblock \emph{Journal of Optimization Theory and Applications}, 29\penalty0
  (3):\penalty0 331--344, 1979.

\bibitem[Holzmann et~al.(1996)Holzmann, Lang, and
  Sch{\"u}tt]{holzmann1996newton}
Oliver Holzmann, Bruno Lang, and Holger Sch{\"u}tt.
\newblock Newton's constant of gravitation and verified numerical quadrature.
\newblock \emph{Reliable Computing}, 2\penalty0 (3):\penalty0 229--239, 1996.

\bibitem[Ichida and Fujii(1979)]{ichida1979interval}
Kozo Ichida and Yasuo Fujii.
\newblock An interval arithmetic method for global optimization.
\newblock \emph{Computing}, 23\penalty0 (1):\penalty0 85--97, 1979.

\bibitem[Jaulin et~al.(2001)Jaulin, Kieffer, Didrit, and
  Walter]{jaulin2001interval}
Luc Jaulin, Michel Kieffer, Olivier Didrit, and {\'E}ric Walter.
\newblock \emph{Applied Interval Analysis}.
\newblock Springer, 2001.

\bibitem[Kaucher and Miranker(1984{\natexlab{a}})]{kaucher1984residual}
E~Kaucher and WL~Miranker.
\newblock Residual correction and validation in functoids.
\newblock In \emph{Defect Correction Methods}, pages 169--192. Springer,
  1984{\natexlab{a}}.

\bibitem[Kaucher(1983)]{kaucher1983solving}
Edgar Kaucher.
\newblock Solving function space problems with guaranteed close bounds.
\newblock In \emph{A New Approach to Scientific Computation}, pages 139--164.
  Elsevier, 1983.

\bibitem[Kaucher and Miranker(1988)]{kaucher1988validating}
Edgar Kaucher and Willard~L Miranker.
\newblock Validating computation in a function space.
\newblock In \emph{Reliability in Computing}, pages 403--425. Elsevier, 1988.

\bibitem[Kaucher and Miranker(1984{\natexlab{b}})]{kaucher1984self}
Edgar~W Kaucher and Willard~L Miranker.
\newblock \emph{Self-validating numerics for function space problems:
  Computation with guarantees for differential and integral equations}.
\newblock Elsevier, 1984{\natexlab{b}}.

\bibitem[Kearfott(1996)]{kearfott1996review}
R~Baker Kearfott.
\newblock A review of techniques in the verified solution of constrained global
  optimization problems.
\newblock \emph{Applications of Interval Computations}, pages 23--59, 1996.

\bibitem[Kelch(1993)]{kelch1993numerical}
Rainer Kelch.
\newblock Numerical quadrature by extrapolation with automatic result
  verification.
\newblock In \emph{Mathematics in science and engineering}, volume 189, pages
  143--185. Elsevier, 1993.

\bibitem[Lanford(1982)]{lanford1982computer}
Oscar~E Lanford.
\newblock A computer-assisted proof of the {F}eigenbaum conjectures.
\newblock \emph{American Mathematical Society}, 6\penalty0 (3), 1982.

\bibitem[Lang(2001)]{lang2001derivative}
Bruno Lang.
\newblock Derivative-based subdivision in multi-dimensional verified gaussian
  quadrature.
\newblock In \emph{Symbolic Algebraic Methods and Verification Methods}, pages
  145--152. Springer, 2001.

\bibitem[Lange(2016)]{lange2016mm}
Kenneth Lange.
\newblock \emph{MM optimization algorithms}.
\newblock SIAM, 2016.

\bibitem[Lee et~al.(2021)Lee, Chang, and Kim]{lee2021further}
Sang~Kyu Lee, Jae~Ho Chang, and Hyoung-Moon Kim.
\newblock Further sharpening of {J}ensen's inequality.
\newblock \emph{Statistics}, 55\penalty0 (5):\penalty0 1154--1168, 2021.

\bibitem[Liao and Berg(2018)]{liao2018sharpening}
JG~Liao and Arthur Berg.
\newblock Sharpening {J}ensen's inequality.
\newblock \emph{The American Statistician}, 2018.

\bibitem[Makino and Berz(2001)]{makino2001higher}
K~Makino and M~Berz.
\newblock Higher order verified inclusions of multidimensional systems by
  {T}aylor models.
\newblock \emph{Nonlinear Analysis: Theory, Methods \& Applications},
  47\penalty0 (5):\penalty0 3503--3514, 2001.

\bibitem[Makino and Berz(1996)]{makino1996remainder}
Kyoko Makino and Martin Berz.
\newblock Remainder differential algebras and their applications.
\newblock \emph{Computational Differentiation: Techniques, Applications, and
  Tools}, pages 63--74, 1996.

\bibitem[Makino and Berz(1999)]{makino1999efficient}
Kyoko Makino and Martin Berz.
\newblock Efficient control of the dependency problem based on {T}aylor model
  methods.
\newblock \emph{Reliable Computing}, 5\penalty0 (1):\penalty0 3--12, 1999.

\bibitem[Makino and Berz(2003)]{makino2003taylor}
Kyoko Makino and Martin Berz.
\newblock Taylor models and other validated functional inclusion methods.
\newblock \emph{International Journal of Pure and Applied Mathematics},
  6:\penalty0 239--316, 2003.

\bibitem[Makino and Berz(2005)]{makino2005verified}
Kyoko Makino and Martin Berz.
\newblock Verified global optimization with {T}aylor model-based range
  bounders.
\newblock \emph{Transactions on Computers}, 11\penalty0 (4):\penalty0
  1611--1618, 2005.

\bibitem[Moore(1966)]{moore1966interval}
Ramon~E Moore.
\newblock \emph{Interval Analysis}, volume~4.
\newblock Prentice-Hall Englewood Cliffs, 1966.

\bibitem[Neher(2003)]{neher2003improved}
Markus Neher.
\newblock Improved validated bounds for {T}aylor coefficients and for {T}aylor
  remainder series.
\newblock \emph{Journal of Computational and Applied Mathematics}, 152\penalty0
  (1-2):\penalty0 393--404, 2003.

\bibitem[Neumaier(2003)]{neumaier2003taylor}
Arnold Neumaier.
\newblock Taylor forms -- use and limits.
\newblock \emph{Reliable Computing}, 9\penalty0 (1):\penalty0 43--79, 2003.

\bibitem[Paszke et~al.(2019)Paszke, Gross, Massa, Lerer, Bradbury, Chanan,
  Killeen, Lin, Gimelshein, Antiga, et~al.]{paszke2019pytorch}
Adam Paszke, Sam Gross, Francisco Massa, Adam Lerer, James Bradbury, Gregory
  Chanan, Trevor Killeen, Zeming Lin, Natalia Gimelshein, Luca Antiga, et~al.
\newblock Pytorch: An imperative style, high-performance deep learning library.
\newblock \emph{Advances in neural information processing systems}, 32, 2019.

\bibitem[Petras(2007)]{petras2007principles}
Knut Petras.
\newblock Principles of verified numerical integration.
\newblock \emph{Journal of Computational and Applied Mathematics}, 199\penalty0
  (2):\penalty0 317--328, 2007.

\bibitem[Ratschek and Rokne(1988)]{ratschek1988new}
Helmut Ratschek and Jon Rokne.
\newblock \emph{New Computer Methods for Global Optimization}.
\newblock Wiley, 1988.

\bibitem[Rokne(1975)]{rokne1975reducing}
Jon Rokne.
\newblock Reducing the degree of an interval polynomial.
\newblock \emph{Computing}, 14\penalty0 (1):\penalty0 5--14, 1975.

\bibitem[Rokne(1977)]{rokne1977bounds}
Jon Rokne.
\newblock Bounds for an interval polynomial.
\newblock \emph{Computing}, 18\penalty0 (3):\penalty0 225--240, 1977.

\bibitem[Simic(2009)]{simic2009best}
Slavko Simic.
\newblock Best possible global bounds for {J}ensen's inequality.
\newblock \emph{Applied Mathematics and Computation}, 215\penalty0
  (6):\penalty0 2224--2228, 2009.

\bibitem[Streeter(2023)]{streeter2023universal}
Matthew Streeter.
\newblock Universal majorization-minimization algorithms.
\newblock \emph{arXiv preprint arXiv:2308.00190}, 2023.

\bibitem[Streeter and Dillon(2023)]{streeter2023sharp}
Matthew Streeter and Joshua~V Dillon.
\newblock Sharp {T}aylor polynomial enclosures in one dimension.
\newblock \emph{arXiv preprint arXiv:2308.00679}, 2023.

\bibitem[van Der~Hoeven(2016)]{van2016certifying}
Joris van Der~Hoeven.
\newblock Certifying trajectories of dynamical systems.
\newblock In \emph{Mathematical Aspects of Computer and Information Sciences:
  6th International Conference, MACIS 2015, Berlin, Germany, November 11-13,
  2015, Revised Selected Papers 6}, pages 520--532. Springer, 2016.

\bibitem[Walker(2014)]{walker2014lower}
Stephen~G Walker.
\newblock On a lower bound for the {J}ensen inequality.
\newblock \emph{SIAM Journal on Mathematical Analysis}, 46\penalty0
  (5):\penalty0 3151--3157, 2014.

\bibitem[Walters and Corliss(2009)]{walters2009}
James~B. Walters and George~F. Corliss.
\newblock \emph{Automatic Differentiation: Point and Interval Taylor
  Operators}, pages 170--176.
\newblock Springer, 2009.

\bibitem[Wang(2017)]{wang2017high}
Mu~Wang.
\newblock \emph{High Order Reverse Mode of Automatic Differentiation}.
\newblock PhD thesis, Purdue University, 2017.

\end{thebibliography}

\begin{appendices}

\chapter{Proofs} \label{sec:proofs}

\section{Taylor Polynomial Enclosures for One-Dimensional Functions}

We now prove Theorem~\ref{thm:autobound_1d}.  We will need the following lemma, which shows that if $F$ is an interval polynomial enclosure in which only the final coefficient is an interval (rather than a scalar), then $F$ is a Taylor polynomial enclosure.

\begin{lemma} \label{lem:when_ip_is_tif}
Let $F(z) = \sum_{i=0}^k F_{[i]} z^i$ be a degree $k$ interval polynomial enclosure of a $k$ times differentiable function $f: \reals \to \reals$ over $Z \in \intervals$.  Suppose that, for $j < k$, $F_{[j]}$ is a scalar (rather than an interval).  Then $F_{[j]} = \frac {1} {j!} f^{(j)}(0)$, and $F$ is therefore a Taylor polynomial enclosure of $f$ at $0$ over $Z$.  (Accordingly, if $g(x) = f(x - x_0)$, then $F$ is a Taylor polynomial enclosure of $g$ at $x_0$ over $x_0 + Z$.)
\end{lemma}
\begin{proof}
Suppose for contradiction that the lemma does not hold, and let $j$ be the smallest index such that $F_j \neq \frac {1} {j!} f^{(j)}(0)$.  Let $T_j$ and $R_j$ denote the Taylor polynomial and the corresponding remainder term, as defined in \S\ref{sec:spoly_definitions}.  Then,
\begin{equation}
  \lim_{z \rightarrow 0} \frac { R_{j-1}(z; f, 0) } { z^j } = \frac {1} {j!} f^{(j)}(0) \neq F_{[j]}.
\end{equation}
At the same time, letting $\lep{F}(z) \eqdef \lep{F(z)}$ and $\rep{F}(z) \eqdef \rep{F(z)}$, we have
\begin{equation}
  F_{[j]} = \lim_{z \rightarrow 0} \frac { R_{j-1}(z; \lep{F}, 0) } { z^j } = \lim_{z \rightarrow 0} \frac { R_{j-1}(z; \rep{F}, 0) } { z^j }.
\end{equation}
Thus, for some sufficiently small $z > 0$, we must have
\begin{equation} \label{eq:some_small_z}
  \frac { R_{j-1}(z; f, 0) } { z^j } \notin \left [\frac { R_{j-1}(z; \lep{F}, 0) } { z^j }, \frac { R_{j-1}(z; \rep{F}, 0) } { z^j } \right ].
\end{equation}
At the same time, because $F_{[i]} = \frac {1} {i!} f^{(i)}(0)$ for $i < j$ (by definition of $j$),
\begin{equation} \label{eq:taylor_match}
  T_{j-1}(z; f, 0) = T_{j-1}(z; \lep{F}, 0) = T_{j-1}(z; \rep{F}, 0).
\end{equation}
Thus, multiplying both sides by of \eqref{eq:some_small_z} by $z^j$, adding $T_{j-1}(z; f, 0)$ to both sides, and using \eqref{eq:taylor_match},
\begin{equation}
  f(z) \notin [ \lep{F}(z), \rep{F}(z) ] = F(z)
\end{equation}
contradicting the assumption that $F$ is an interval polynomial enclosure of $f$.
\end{proof}

\begin{customthm}{4}
\thmautoboundoned
\end{customthm}
\begin{proof}
\thmautoboundonedproof
\end{proof}

To prove Theorem~\ref{thm:ab_advantage}, we will use the following lemma.

\begin{lemma} \label{lem:power_intervals}
Let $k$ and $\ell$ be positive integers, and let $f(x) = x^{k + \ell}$.  For any interval $[a, b]$ with $0 \in [a, b]$,
\[
  \iia(f, k, [a, b]) = {k+\ell \choose \ell} [a, b]^\ell
\]
and
\[
  \iab(f, k, [a, b], 0)) = [a, b]^\ell
\]
where $\iia$ and $\iab$ are defined as in Theorem~\ref{thm:ab_advantage}.
\end{lemma}
\begin{proof}
The $k$th derivative is $f^{(k)}(x) = \frac {(k+\ell)!} {\ell!} x^\ell$.  Plugging in $x = [a, b]$, and evaluating the resulting expression according to the rules of interval arithmetic, gives
\begin{equation}
  f^{(k)}([a, b]) = \frac {(k+\ell)!} {\ell!} [a, b]^\ell.
\end{equation}
Therefore,
\begin{equation} \label{eq:ia_interval}
  \iia(f, k, [a, b]) = \frac {1} {k!} f^{(k)}([a, b]) = {k+\ell \choose \ell} [a, b]^\ell.
\end{equation}
Applied to $f$, \ref{alg:autobound} returns a polynomial $P_1$, obtained by applying the interval polynomial extension of non-negative integer exponentiation (defined in \S\ref{sec:ab_exponentiation}) to the  polynomial $P_0(z) = x_0 + z$.  The interval polynomial extension expands $P_0(z)^{k+\ell} = (x_0 + z)^{k+\ell}$ using the binomial theorem, factors terms of degree $> k$ into the product of $z^k$ and a single polynomial, and invokes the $\rangebound$ function to obtain
\begin{align}
  P_1(z)
  & = \paren { \sum_{i=0}^{k-1} {k + \ell \choose i} x_0^{k+\ell-i} z^i } + z^k \cdot \rangebound\paren{z \mapsto \sum_{i=k}^{k+\ell} {k+\ell \choose i} x_0^{k+\ell-i} z^{i-k}, Z} \nonumber \\
  & = z^k \cdot \rangebound\paren{z \mapsto z^\ell, Z}
\end{align}
where on the second line we have used the fact that $x_0 = 0$ (defining $0^0 \eqdef 1$).
Assuming the $\rangebound$ function uses interval arithmetic evaluation, we have $\rangebound(z \mapsto z^\ell, Z) = Z^\ell = [a, b]^\ell$ (for $x_0 = 0$).  Therefore,
\begin{equation} \label{eq:ab_interval}
  \iab(f, k, [a, b], 0)) = [a, b]^\ell.
\end{equation}
\end{proof}

\begin{customthm}{5}
\thmabadvantage
\end{customthm}
\begin{proof}
\thmabadvantageproof
\end{proof}

\section{Generalization to Multivariate Functions}

We first prove Proposition~\ref{prop:inner_exact}.  To do so, we will use the following proposition.

\begin{proposition} \label {prop:dot_product_exact}
For any interval vector $\sx \in \intervals^n$ and vector $\vy \in \reals^n$,
\[
  \sx \cdot \vy = \set {\vx \cdot \vy: \vx \in \sx }.
\]
\end{proposition}

Proposition~\ref{prop:dot_product_exact} is a special case of a more general result for interval expressions in which each interval is used exactly once; see \cite[Chapter~5]{moore1966interval}.

\begin{customproposition}{5}
\propinnerexact
\end{customproposition}
\begin{proof}
\propinnerexactproof
\end{proof}

\subsection{Bilinear Operations on Tensor Intervals}

We now provide proofs of the rules given in \S\ref{sec:tensor_interval_bilinear} for applying bilinear operations to tensor intervals.  We first prove Proposition~\ref{prop:ibp}, which provides a simple rule that generalizes previous work on Interval Bound Propagation \cite{gowal2018effectiveness}.

\begin{customproposition}{6}
\propibp
\end{customproposition} 
\begin{proof}
\propibpproof
\end{proof}	

We next prove Theorem~\ref{thm:apply_bilinear_tensor_interval}, which provides an alternative rule that yields a potentially narrower interval at the cost of additional computation.  To do so, we first prove Lemma~\ref{lem:linear_mult_rule}.

\begin{customlemma}{3}
\lemlinearmultrule
\end{customlemma}
\begin{proof}
\lemlinearmultruleproof
\end{proof}

We are now ready to prove Theorem~\ref{thm:apply_bilinear_tensor_interval}.
\begin{customthm}{6}
\thmapplybilineartensorinterval
\end{customthm}
\begin{proof}
\thmapplybilineartensorintervalproof
\end{proof}

\subsection{Proof of Theorem~\ref{thm:degree_reduction}}

In order to prove Theorem~\ref{thm:degree_reduction}, we first prove Proposition~\ref{prop:hadamard}.

\begin{customproposition}{7}
\prophadamard
\end{customproposition}
\begin{proof}
\prophadamardproof
\end{proof}

To prove Theorem~\ref{thm:degree_reduction}, we will also need the following proposition.

\newcommand{\propinnerouter}{
For any tensors $\mA$, $\mB$ and $\mC$, of shapes such that $\inner{\mA} {\mB \otimes \mC}$ is well-defined,
\[
  \inner{\mA} {\mB \otimes \mC} = \inner {  \inner{\mA} {\mC}  } {\mB}.
\]
}
\newcommand{\propinnerouterproof}{
To simplify indexing, assume without loss of generality that $\mA$ is a matrix, and that both $\mB$ and $\mC$ are vectors (which implies $\inner{\mA} {\mB \otimes \mC}$ is a scalar).  We have
\begin{align}
  \inner{\mA}{\mB \otimes \mC}
  & = \sum_{ij} \mA_{ij} (\mB \otimes \mC)_{ij} \nonumber \\
  & = \sum_{ij} \mA_{ij} \mB_i \mC_j \nonumber \\
  & = \sum_i \paren{\sum_j \mA_{ij} \mC_j } \mB_i \nonumber \\
  & = \inner{\inner{\mA}{\mC}}{\mB}.
\end{align}
}
\begin{proposition} \label{prop:inner_outer}
\propinnerouter
\end{proposition}
\begin{proof}
\propinnerouterproof
\end{proof}

We are now ready to prove Theorem~\ref{thm:degree_reduction}.
\begin{customthm}{7}
\thmdegreereduction
\end{customthm}
\begin{proof}
\thmdegreereductionproof
\end{proof}

\subsection{Proof of Theorem~\ref{thm:bilinear_tip}}

To prove Theorem~\ref{thm:bilinear_tip}, we will need the following lemma.

\begin{lemma} \label{lem:tip_choice}
Let $\sA(\mZ) = \sum_{i=0}^k \inner{\sA_{[i]}}{\mZ^{\otimes i}}$  be a degree $k$ interval polynomial.  Then, for any tensor $\mZ$,
\[
  \sum_{i=0}^k \inner{\sA_{[i]}}{\mZ^{\otimes i}} = \set { \sum_{i=0}^k \inner{\mA_i}{\mZ^{\otimes i}}: \mA_i \in \sA_{[i]}\ \forall i \in \set{0, 1, \ldots, k}}.
\]  
\end{lemma}

Lemma~\ref{lem:tip_choice} follows from Proposition~\ref{prop:inner_exact}, and the fact that $\sX + \sY = \set {\mX + \mY: \mX \in \sX, \mY \in \sY}$.

\begin{customthm}{8}
\thmbilineartip
\end{customthm}
\begin{proof}
\thmbilineartipproof
\end{proof}

\subsection{Proof of Theorem~\ref{thm:autoboundprop}}

\begin{customthm}{9}
\thmautoboundprop
\end{customthm}
\begin{proof}
\thmautoboundpropproof
\end{proof}

\section{Proof of Theorem~\ref{thm:jensen}}

\begin{customthm}{12}
\thmjensennofootnote
\end{customthm}
\begin{proof}
\thmjensenproof
\end{proof}

\end{appendices}

\end{document}